\def\year{2022}\relax
\documentclass[letterpaper]{article} % DO NOT CHANGE THIS
\usepackage{aaai22}  % DO NOT CHANGE THIS
\usepackage{times}  % DO NOT CHANGE THIS
\usepackage{helvet}  % DO NOT CHANGE THIS
\usepackage{courier}  % DO NOT CHANGE THIS
\usepackage[hyphens]{url}  % DO NOT CHANGE THIS
\usepackage{graphicx} % DO NOT CHANGE THIS
\usepackage{natbib}  % DO NOT CHANGE THIS AND DO NOT ADD ANY OPTIONS TO IT
\usepackage{caption} % DO NOT CHANGE THIS AND DO NOT ADD ANY OPTIONS TO IT
\DeclareCaptionStyle{ruled}{labelfont=normalfont,labelsep=colon,strut=off} % DO NOT CHANGE THIS
\usepackage{amsfonts}       % blackboard math symbols
\usepackage{nicefrac}       % compact symbols for 1/2, etc.
\usepackage{minitoc}

\usepackage{algorithm}
\usepackage{algorithmic}
\usepackage{subcaption}
\usepackage{microtype}
\usepackage{wrapfig,lipsum,booktabs}
\usepackage{colortbl}
\usepackage{minitoc}
\usepackage[dvipsnames]{xcolor}

\usepackage{subdef}
\usepackage{todo}

\newcommand{\figref}[1]{Fig.~\ref{#1}}
\newcommand{\tabref}[1]{Tab.~\ref{#1}}
\newcommand{\secref}[1]{Sec.~\ref{#1}}
\newcommand{\AlgRef}[1]{Algo.~\ref{#1}}

\newcommand{\model}{\mbox{\textsc{Prism}}}

\usepackage{newfloat}
\usepackage{listings}
\floatstyle{ruled}
\newfloat{listing}{tb}{lst}{}
\floatname{listing}{Listing}
\title{\model: A Rich Class of Parameterized Submodular Information Measures for Guided Subset Selection}
\author {
    Suraj Kothawade\textsuperscript{\rm 1},
    Vishal Kaushal \textsuperscript{\rm 2},
    Ganesh Ramakrishnan \textsuperscript{\rm 2},
    Jeff Bilmes \textsuperscript{\rm 3},
    Rishabh Iyer \textsuperscript{\rm 1,2}
}
\affiliations {
    \textsuperscript{\rm 1} University of Texas at Dallas\\
    \textsuperscript{\rm 2} Indian Institute of Technology, Bombay\\
    \textsuperscript{\rm 3} University of Washington, Seattle\\
    suraj.kothawade@utdallas.edu, vkaushal@cse.iitb.ac.in, ganesh@cse.iitb.ac.in, bilmes@uw.edu, rishabh.iyer@utdallas.edu
}
\usepackage{bibentry}
\begin{document}

\maketitle
\doparttoc
\faketableofcontents
\begin{abstract}
  With ever-increasing dataset sizes, subset selection techniques are becoming increasingly important for a plethora of tasks. It is often necessary to \emph{guide} the subset selection to achieve certain desiderata, which includes \emph{focusing or targeting} certain data points, while \emph{avoiding} others. Examples of such problems include: i) \emph{targeted learning}, where the goal is to find subsets with rare classes or rare attributes on which the model is underperforming, and ii) \emph{guided summarization}, where data ({\em e.g.}, image collection, text, document or video) is summarized for quicker human consumption with specific additional user intent.  
  %Often, the subsets need to possess specific characteristics to be effective for the tasks at hand. For example, a supervised classification model's performance on a particular class or a particular slice of data can be improved at a given additional labeling cost by adding a small subset of unlabeled points matching that class or slice of data (target) to the training set. We call such subsets \emph{targeted subsets}. A more nuanced example is that of summarization where data (e.g., image collections, text or videos) is summarized for quicker human consumption with specific additional user intent (query-focused summarization). The target here need not be restricted to a query but can assume different semantics. For example, update-summarization requires a summary to be different from an existing summary. The existing summary thus serves as the target and the characteristic desired of the targeted subset is that it should be \emph{different} from the target. 
  Motivated by such applications, we present \model{}, a rich class of \textbf{P}a\textbf{R}ameter\textbf{I}zed \textbf{S}ubmodular information \textbf{M}easures. %, that can be used to find such \emph{targeted} subsets. 
  %\model{} extends submodular mutual information and conditional gain functions to allow a target set to be different from the ground set from which subsets are sought. 
Through novel functions and their parameterizations, \model{} offers a variety of modeling capabilities that enable a trade-off between desired qualities of a subset like diversity or representation and similarity/dissimilarity with a set of data points. 
% such as, trading off between target relevance or the strictness of the dissimilarity from the set we want to avoid on one hand, and diversity or representation on the other. 
% This makes it possible to cater to different characteristics desired of the targeted subsets.
We demonstrate how \model{} can be applied to the two real-world problems mentioned above, which require guided subset selection. In doing so, we show that \model{} interestingly generalizes some past work, therein reinforcing its broad utility. %\model{} offers a unified approach to joint modelling of different flavors of the target and
Through extensive experiments on diverse datasets, we demonstrate the superiority of \model{} over the state-of-the-art in targeted learning and in guided image-collection summarization. \model\ is available as a part of the \textsc{Submodlib} (\url{https://github.com/decile-team/submodlib}) and \textsc{Trust} (\url{https://github.com/decile-team/trust}) toolkits.\looseness-1
\end{abstract}

\vspace{-3ex}
\section{Introduction}
\vspace{-0.5ex}

%\todo{Things that needed tending to are marked as this.} \response{Once they are tended to, respond like this and then once it is agreeably resolved, we can comment out these macros.}

%\todo{See the file comments.txt in the current directory for some comments from Jeff.}

Recent times have seen explosive growth in data across several modalities, including text, images, and videos. This has given rise to the need for finding techniques for selecting effective smaller data subsets with specific characteristics for a variety of down-stream tasks. Often, we would like to \emph{guide} the data selection to either \emph{target} or \emph{avoid} a certain set of data slices. One application is, what we call, \emph{targeted learning}, where the goal is to select data points similar to data slices on which the model is currently performing poorly. These slices are data points that either belong to rare classes or have common rare attributes ({\em e.g.}, color, background, {\em etc.}). An example of such a scenario is shown in Fig.~\ref{fig:motivating-applications}(a), where a self-driving car model struggles in detecting ``cars in a dark background`` because of a lack of such images in the training set. The targeted learning problem is to augment the training dataset with more of such rare images, with an aim to improve model performance. Another example is detecting cancers in biomedical imaging datasets, where the number of cancerous images are often a small fraction of the non-cancerous images. 
%The guided selection problem here is as follows: given a few examples of these rare classes or slices (called the target), we would like to select unlabeled data points similar to the target to augment the training set. 
%Fig.~\ref{fig:motivating-applications}(a) shows the detailed workflow. %One example is efficient cost-effective training of machine learning models, where we need to select samples that are most useful for model training. 
%In practice, there is often a distribution shift between training and testing data. A model's performance on a desired target can be improved (under given additional labeling costs) by augmenting the training data with samples best matching the target distribution from a large pool of unlabeled data. %Relative to training on all of the targeted data, training on such smaller subsets often entails significant speedups and labeling time/cost reductions without sacrificing much accuracy. 
%One way to achieve this is by granting access to a clean validation set matching the target set distribution and using it as a target. Alternatively, the target set could be a critical slice of the data ({\em e.g.}, indoor images of people in a dark background, or example images from specific classes that a user might particularly care about). In such a case, the goal can be to improve the model's performance on the target without sacrificing the overall accuracy and with minimum additional labeling costs ({\it c.f.}, Fig.~\ref{fig:motivating-applications}(a)).

%Increased from 0.43 to 0.45
\begin{figure}[h]
\centering
\begin{subfigure}[b]{0.42\textwidth}
   \includegraphics[width=1\linewidth]{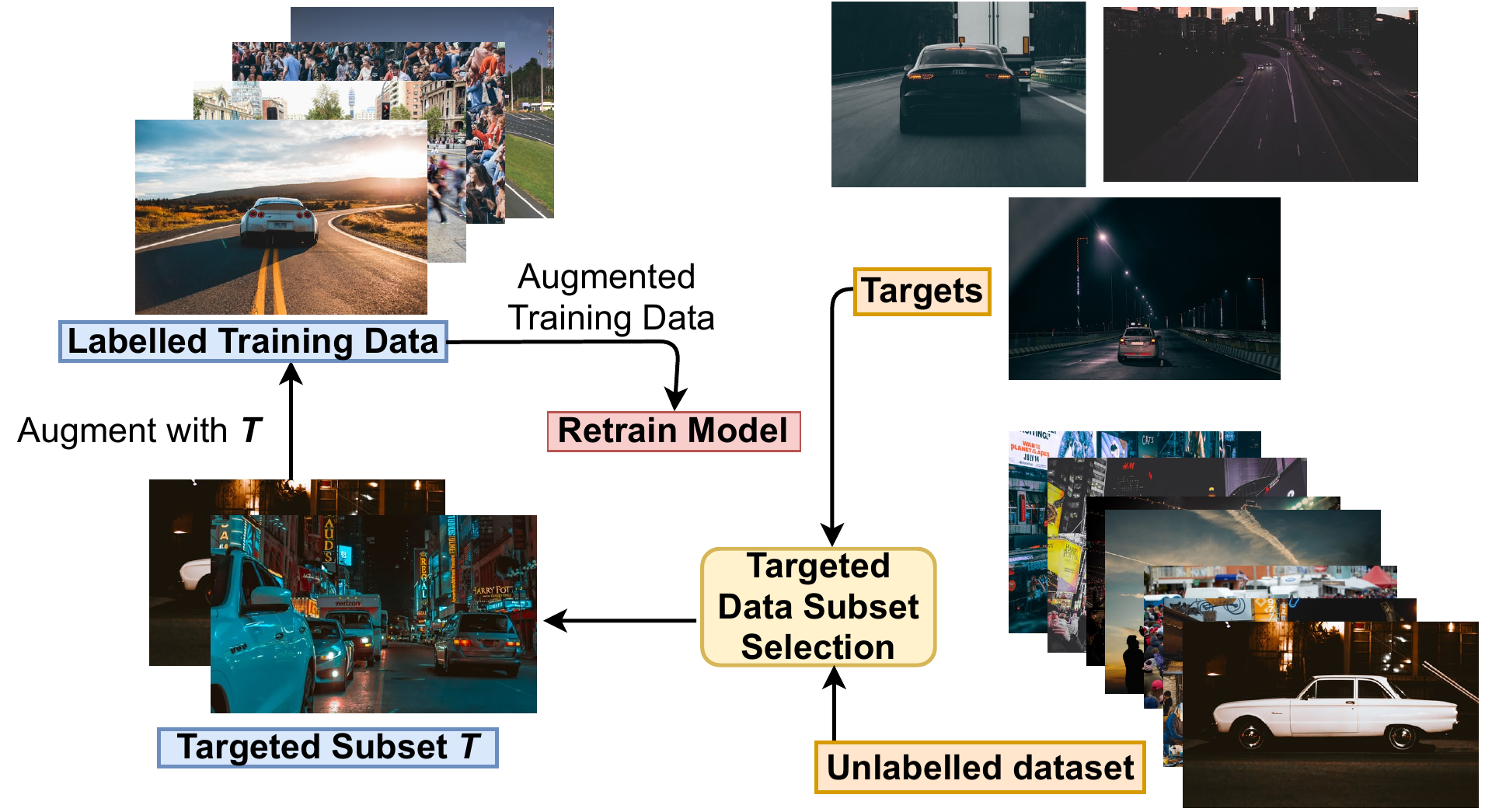}
   \caption{}
   \label{fig:mot-tss} 
\end{subfigure}
\begin{subfigure}[b]{0.42\textwidth}
   \includegraphics[width=1\linewidth]{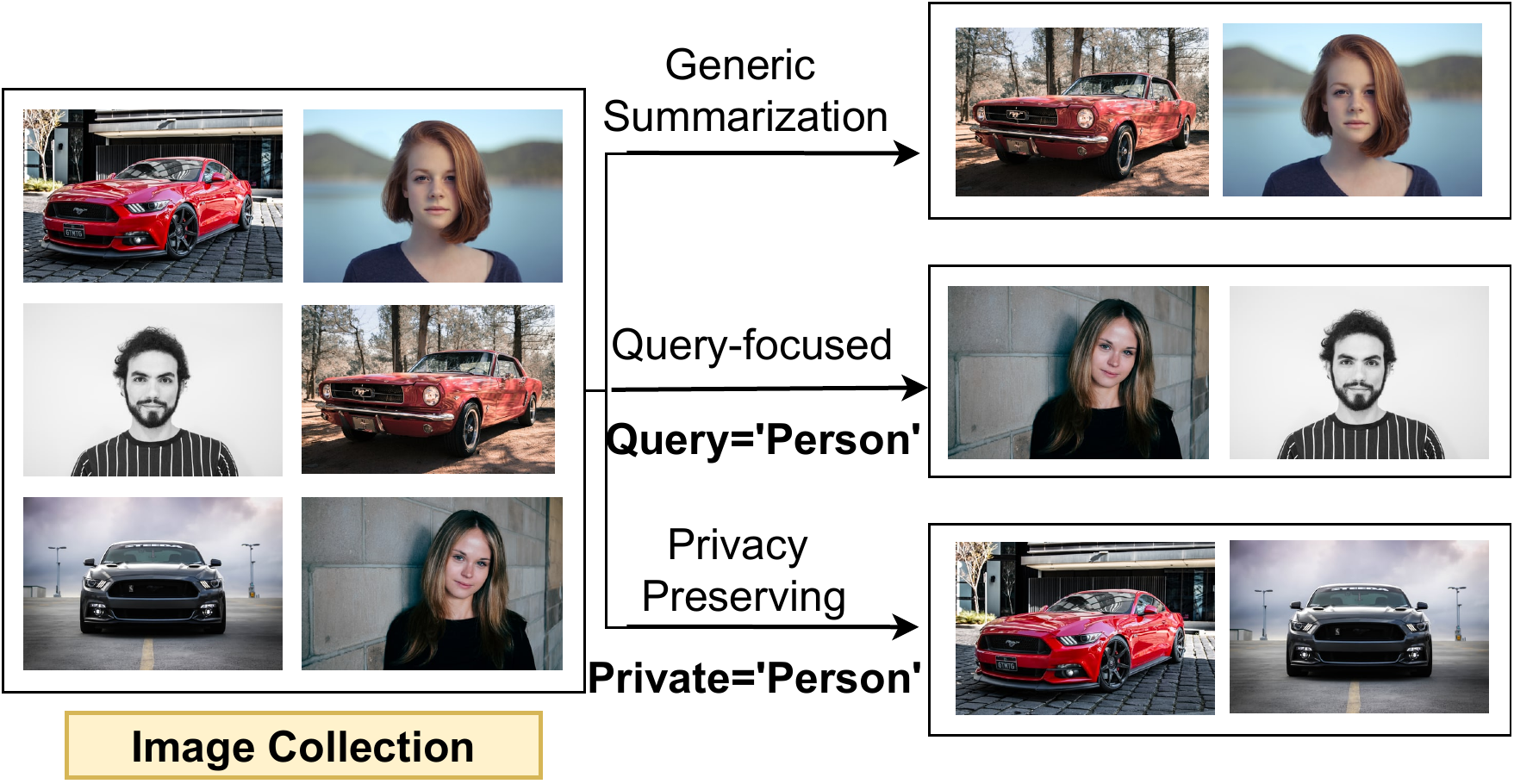}
   \caption{}
   \label{fig:mot-tsum}
\end{subfigure}
\caption{Applications of guided subset selection. (a) \emph{Targeted learning:} improving a model's performance on night images (target), which are under-represented in the training data. This is achieved by augmenting it with a subset matching the target. (b) \emph{Guided summarization:} finding a summary similar to a query set or a summary dissimilar to a private set.}
\vspace{-2ex}
\label{fig:motivating-applications}
\end{figure}

Another application comes from the summarization task, where an image collection, a video, or a text document is summarized for quicker human consumption by eliminating redundancy, while preserving the main content. While a number of applications require \emph{generic} summarization ({\em i.e.}, simply picking a representative and diverse subset of the massive dataset), it is often important to capture certain user intent in summarization. We call this \emph{guided summarization}. Examples of guided summarization include: (i) \emph{query-focused summarization}~\cite{sharghi2016query, xiao2020convolutional}, where a summary similar to a specific query is desired, %(subset matching the target)
and (ii) \emph{%query-irrelevant and 
privacy-preserving summarization}, where a summary dissimilar to %not close to a given query or completely different from 
a given private set of data points is desired (say, for privacy issues). See Fig.~\ref{fig:motivating-applications}(b) for a pictorial illustration.\looseness-1 

%In this work we propose \model{} - a novel framework based on a rich class of \textbf{P}a\textbf{R}ameter\textbf{I}zed \textbf{S}ubmodular information \textbf{M}easures that can be used for targeted subset selection tasks.\looseness-1

\subsection{Our Contributions}
\label{sec:contributions}
\textbf{\model{} Framework:} We define \model{} through different instantiations and parameterizations of various submodular information measures (\secref{sec:model}). These allow for modeling a spectrum of semantics required for guided subset selection, like relevance to a query set, irrelevance to a private set, and diversity among selected data points. We study the effect of parameter trade-off among these different semantics and present interesting insights.

\noindent \textbf{\model{} for Targeted Learning:} We present a novel algorithm (\secref{sec:tss}, \AlgRef{algo:tss}) to apply \model{} for targeted learning, which aims to improve a model's performance on
rare slices of data. Specifically, we show that submodular information measures are very effective in finding the examples from the rare classes in a large unlabeled set (akin to finding a needle in a haystack). On several image classification tasks, \model{} obtains $\approx$ 20-30\% gain in accuracy of rare classes ($\approx$ 12\% more than existing approaches) by just adding a few additional labeled points from the rare classes. Furthermore, we show that \model{} is $20\times$ to $50\times$ more label-efficient compared to random sampling, and $2\times$ to $4\times$ more label-efficient compared to existing approaches (see \secref{sec:exp-tss}). We also show that ~\AlgRef{algo:tss} generalizes some existing approaches for data subset selection, reinforcing its utility (\secref{sec:connections}). \looseness-1

%a desired class by adding a targeted subset to the training data (see \figref{fig:motivating-applications}(a)). Our method significantly outperforms other techniques for image classification tasks on CIFAR-10~\cite{krizhevsky2009learning}, MNIST~\cite{lecun1998gradient}, SVHN~\cite{netzer2011reading} and Pneumonia-MNIST~\cite{yang2021medmnist, kermany2018identifying} on classes of interest at a given additional labeling cost. Without compromising on the overall accuracy across classes, \model{} obtains $\approx$ 20-30\% gain over the model's performance before re-training with the additional targeted subset; $\approx$ 12\% more than other methods (see \secref{sec:exp-tss}). We show that ~\AlgRef{algo:tss} generalizes some existing approaches for data subset selection~\cite{killamsetty2020glister, mirzasoleiman2020coresets, killamsetty2021grad} and is thus widely applicable and useful (\secref{sec:tss-connections}).

\noindent \textbf{\model{} for Guided Summarization.} We propose a learning framework for guided summarization using \model{} (\secref{sec:tsum}). We show that \model{} offers a \emph{unified} treatment to the different flavors of guided summarization (query-focused and privacy-preserving) and generalizes some existing approaches to summarization,
%~\cite{sharghi2016query, sharghi2017query, vasudevan2017query, lin2012submodularity, li2012multi, lin2011class, lin2004rouge}, 
again reinforcing its utility. We show that %demonstrate its effectiveness therein. %(\secref{sec:modelsum}) and show that \model{} offers a \textbf{unified approach to joint modelling of different flavors of the target}. %and show that \model{} offers a unified approach to the different semantics of \emph{target} in query-focused summarization, topic-irrelevant or privacy-preserving summarization and update summarization and 
%We demonstrate in \secref{subsec:sum-exp} that  %, and has interesting connections to a number of existing approaches for targeted data selection (\secref{sec:gen}).
it outperforms other existing approaches on a real-world image collections dataset (\secref{sec:exp-tsum}). \looseness-1

\subsection{Related Work}
\textbf{Submodularity and Submodular Information Measures: } Submodularity~\cite{fujishige2005submodular} is a rich yet tractable subfield of non-linear combinatorial optimization~\cite{krause2014submodular}.% and nice connections to convexity and concavity~\cite{bach2011learning,lovasz1983submodular,iyer2015polyhedral}. %Proposed \model{} 
We provide novel formulations of the recently introduced class of submodular information measures~\cite{levin2020online,iyer2021submodular} for guided data subset selection.%, which are generalizations of information theoretic measures (like entropy, information gain, and mutual information) to submodular functions. 

\noindent \textbf{Data Subset Selection, Coresets, and Active Learning:} 
A number of papers have studied data subset selection in different applications and settings. Several recent papers have studied data subset selection for speeding up training. These include approaches involving submodularity~\cite{wei2015submodularity, kaushal2019learning}, gradient coresets~\cite{mirzasoleiman2020coresets,killamsetty2021grad} and bi-level based coresets~\cite{killamsetty2020glister}. Another application is active learning, where the goal is to select and label a subset of unlabeled data points to improve model performance~\cite{settles2009active}. Several recent approaches which combine notions of diversity and uncertainty have become popular~\cite{wei2015submodularity,sener2018active,ash2020deep}. One such state-of-the-art approach is \textsc{Badge}~\cite{ash2020deep}, which samples points that have diverse hypothesized gradients. Most of these paradigms have been studied in the setting of generic data subset selection, and are ineffective when it comes to guided subsets. Some recent works like \textsc{Grad-match}~\cite{killamsetty2021grad} and \textsc{Glister}~\cite{killamsetty2020glister} select subsets based on a held out validation set, which can be a rare slice of data. Similarly, \cite{kirchhoff2014-submodmt} compute a targeted subset of training data in the spirit of transductive learning for machine translation. %We apply \model{} to find targeted subsets of unlabeled data for improving a model's performance on the targeted slice, and in fact show that it generalizes many existing targeted data selection approaches.

\noindent \textbf{Summarization: } A number of instances of summarization have been studied in the past, including image collection summarization~\cite{celis2020implicit, ozkose2019diverse, singh2019image, tschiatschek2014learning}, text/document summarization~\cite{lin2012learning,chali2017towards,yao2017recent}, and video summarization~\cite{Kaushal2019DemystifyingMV,Kaushal2019AFT,Gygli2015VideoSB,ji2019video}. While most of these works have focused on generic summarization, some have also studied query-focused video summarization~\cite{sharghi2016query,sharghi2017query,vasudevan2017query,xiao2020convolutional,jiang2019hierarchical}, and query-focused document summarization~\cite{lin2011class,li2012multi}.
% and update-summarization of documents ~\cite{dang2008overview,delort2012dualsum,li2015improving}.
To the best of our knowledge, \model{} is the first attempt to offer a unified treatment to the different flavors of summarization.

%\textbf{Properties of CG, MI and CMI:} CG, MI, and CMI are non-negative and monotone in one argument with the other fixed~\cite{levin2020online,iyer2021submodular}. CMI and MI are not necessarily submodular in one argument (with the others fixed)~\cite{krause2008near, iyer2021submodular}. However, several of the instantiations we define below turn out to be submodular.

\vspace{-1ex}
\section{The \model{} Framework}
\label{sec:model}
% \vspace{-1ex}
\subsection{Preliminaries}
\textbf{Submodular functions:} Let $\Vcal = \{1, 2, 3,...,n \}$ denote the \emph{ground-set} and $f$ denote a set function $f:2^{\Vcal} \xrightarrow{} \Re$. %We say that $f$ is {normalized} if $f(\emptyset) = 0$ and $f$ is {sub-additive} if $f(S) + f(\Tcal) \geq f(S \cup T)$, holds for all $S,T \subseteq \Omega$.  
The function $f$ is \emph{submodular} if it satisfies the diminishing marginal returns property; namely $f(j | \Xcal) \geq f(j | \Ycal), \forall \Xcal \subseteq \Ycal \subseteq \Vcal, j \notin \Ycal$~\cite{fujishige2005submodular}. %for all $S, T \subseteq \Vcal$, it holds that $f(S) + f(\Tcal) \geq f(S \cup T) + f(S \cap \Tcal)$. An identical characterization of submodularity is  Also, $f$ is supermodular if $-f$ is submodular, and $f$ is said to be {monotone} if $f(X) \leq f(\Ycal)$, $\forall X\subseteq Y\subseteq \Vcal$ (equivalently, $f(j|\Acal) \geq 0$ for all $j \notin \Acal$ and $\Acal \subseteq \Vcal$). 
Submodularity (along with monotonicity) ensures that a greedy algorithm achieves a $1 - 1/e$ approximation when $f$ is maximized~\cite{nemhauser1978analysis}. 
%Facility location (FL), set cover (SC), probabilistic set cover (PSC), graph cut (GC), log determinants (LogDet), {\em etc.} are some examples of submodular functions. %Due to close connections between submodularity and entropy, submodular functions can also be viewed as \emph{information functions}~\cite{zhang1998characterization}.
\\
\textbf{Submodular Conditional Gain (CG):} Given sets $\Acal, \Pcal \subseteq \Vcal$, the CG, $f(\Acal | \Pcal)$, is the gain in function value by adding $\Acal$ to $\Pcal$. Thus $f(\Acal | \Pcal) = f(\Acal \cup \Pcal) - f(\Pcal)$. Intuitively, $f(\Acal|\Pcal)$ measures how different $\Acal$ is from $\Pcal$, where $\Pcal$ is the \emph{conditioning set} or the \emph{private set}.\\
\textbf{Submodular Mutual Information (MI):} Given sets $\Acal, \Qcal \subseteq \Vcal$, the MI~\cite{levin2020online,iyer2021submodular} is defined as $I_f(\Acal; \Qcal) = f(\Acal) + f(\Qcal) - f(\Acal \cup \Qcal)$. Intuitively, this measures the similarity between $\Qcal$ and $\Acal$ where $\Qcal$ is the query set.\\
\textbf{Submodular Conditional Mutual Information (CMI):}  CMI is defined using CG and MI as $I_f(\Acal; \Qcal | \Pcal) %= f(\Acal | \Ccal ) + f(\Bcal | \Ccal) - f(\Acal \cup \Bcal | \Ccal)$.
= f(\Acal \cup \Pcal) + f(\Qcal \cup \Pcal) - f(\Acal \cup \Qcal \cup \Pcal) - f(\Pcal)$. 
Intuitively, CMI jointly models the mutual similarity between $\Acal$ and $\Qcal$ and their collective dissimilarity from $\Pcal$. 

\noindent \textbf{Properties: } CG, MI, and CMI are non-negative and monotone in one argument with the other fixed~\cite{levin2020online,iyer2021submodular}. CMI and MI are not necessarily submodular in one argument (with the others fixed)~\cite{iyer2021submodular}. However, several of the instantiations we define below turn out to be submodular. With this background, we present our unique and novel formulations, leading to \model{}. 

\subsection{Guidance from an Auxiliary Set}
We formulate the above submodular information measures to handle the case when the \emph{guidance} can come from an auxiliary set $\Vcal^{\prime}$ different from the ground set $\Vcal$ -- a requirement common in several guided subset selection tasks. %For targeted data subset selection, $\Vcal$ is the source set of data instances and the target is a subset of data points (validation set or the specific set of examples of interest). In the case of targeted summarization, $\Vcal$ is the set of data points that the user wants to summarize (say images or video frames or video shots or sentences) and target is the query set (for query-focused summarization), private set (for topic-irrelevant or privacy-preserving summarization) or conditioning set (for update-summarization). 
Let $\Omega  = \Vcal \cup \Vcal^{\prime}$. We define a set function $f: 2^{\Omega} \rightarrow \Re$. Although $f$ is defined on $\Omega$, the discrete optimization problem will only be defined on subsets $\Acal \subseteq \Vcal$. To find an optimal subset (i) given a query set $\Qcal \subseteq \Vcal^{\prime}$, we can define $g_{\Qcal}(\Acal) = I_f(\Acal; \Qcal)$, $\Acal \subseteq \Vcal$ and maximize the same; (ii) given a private set $\Pcal \subseteq \Vcal^{\prime}$, we can define $h_{\Pcal}(\Acal) = f(\Acal | \Pcal )$, $\Acal \subseteq \Vcal$, as the function to be maximized. %As we shall see below, these offer a rich class of models for both motivating applications. 

%Comment our restricted submodularity, MOVE TO APPENDIX!!

\subsection{Restricted Submodularity to Enable a Richer Class of MI and CG Functions} \label{sec:gmi} While submodular functions are expressive, many natural choices are not submodular everywhere. We do not need $f$ to be submodular everywhere on $\Omega$, since the sets we are optimizing on, are subsets of $\Vcal$. Instead of requiring the submodular inequality to hold for all pairs of sets $(\Xcal, \Ycal) \in 2^{\Omega}\times2^{\Omega}$, we can consider only subsets  of this power set pivoting on $\Vcal \subseteq \Omega$. 
In particular, define a subset $\Ccal \subseteq 2^{\Omega}$. Then \emph{restricted submodularity} on $\Ccal$ satisfies $f(\Xcal) + f(\Ycal) \geq f(\Xcal \cup \Ycal) + f(\Xcal \cap \Ycal), \forall (\Xcal, \Ycal) \in \Ccal$. Instances of restricted submodularity in the form of intersecting and crossing submodular functions have been considered in the past~\cite{fujishige2005submodular}. We consider the following form of restricted submodularity. % to define Generalized Submodular Mutual Information functions (GMI). %Examples of restricted submodularity include (i) intersecting submodular functions, where $\mathcal C$ consists of sets $\Acal, B$ such that $A \cap \Bcal \neq \emptyset$ and (ii) crossing submodularity, where $\mathcal C$ consists of sets $(\Acal, \Bcal) \in 2^{2\Omega}: A \cap \Bcal \neq \emptyset$ and $\Acal \cup \Bcal \neq \Omega$. To define GMI, we consider restricted submodularity as follows. 
Given sets $\Vcal$ and $\Vcal^{\prime}$ as above, define $\Ccal(\Vcal, \Vcal^{\prime}) \subseteq 2^{\Omega}$ to be such that the sets $(\Xcal, \Ycal) \in \Ccal(\Vcal, \Vcal^{\prime})$ satisfy either of the following conditions: i) $\Xcal \subseteq \Vcal$ or $\Xcal \subseteq \Vcal^{\prime}$ and $\Ycal$ is \emph{any} set, or ii) $\Xcal$ is \emph{any} set and $\Ycal \subseteq \Vcal$ or $\Ycal \subseteq \Vcal^{\prime}$. We call the MI of a restricted submodular function as Generalized Mutual Information function (GMI). We use this notion of GMI %allows us to instantiate submodular mutual information with a function $f$, which is not submodular everywhere on $\Omega$, but only on sub-lattices. For example, we use this 
to define Concave Over Modular (COM).
% and Query-Saturation (Q-SAT) functions.% which have interesting connections with past work (Section~\ref{sec:tsum}). 
The GMI function is non-negative and monotone (see Appendix~\ref{app:gsmi}). %We state the properties of GMI in the following Lemma and defer the proof to Appendix~\ref{app:gsmi}. %The generalized submodular mutual information on $\mathcal C(\Vcal, \Vcal^{\prime})$, which we denote as $I_f^{\mathcal C}(\Acal; \Bcal)$ satisfies non-negativity and monotonicity properties.

% \vspace{-1ex}
% \begin{lemma} \label{lemma:GMI}
% Given a restricted submodular function $f$ on $\Ccal(\Vcal, \Vcal^{\prime})$, $I_f(\Acal; \Bcal) \geq 0$ for $\Acal \subseteq \Vcal, \Bcal \subseteq \Vcal^{\prime}$. Also, $I_f(\Acal; \Bcal)$ is monotone in $\Acal \subseteq \Vcal$ for fixed $\Bcal \subseteq \Vcal^{\prime}$ (equivalently, $I_f(\Acal; \Bcal)$ is monotone in $\Bcal \subseteq \Vcal^{\prime}$ for fixed $\Acal \subseteq \Vcal$).
% \end{lemma}

\subsection{Instantiations \& Parameterizations in \model{}} In this section, we discuss the expressions for different instantiations of the above measures using different functions. We refer to them as \textbullet MI or \textbullet CG or \textbullet CMI where \textbullet\ is the submodular function using which the respective MI, CG or CMI measure is instantiated. 
%We propose new instantiations for LogDet, COM and Q-SAT. We formulate instantiations for Graph Cut (GC) and Facility Location (FL) for our setting of distinct summary space ($\Vcal$) and auxiliary space ($\Vcal^{\prime}$). We derive a new mutual information variant for FL which has interesting characteristics. 
While different submodular functions naturally model different characteristics such as representation, coverage, {\em etc.}~\cite{Kaushal2019DemystifyingMV,Kaushal2019AFT}, the instantiations presented here additionally model similarity and dissimilarity to query and private sets respectively. These instantiations have parameters $\lambda, \eta$ and/or $\nu$, that govern the interplay among different characteristics. In several instantiations, we invoke a similarity matrix $S$ where $S_{ij}$ measures the similarity between elements $i$ and $j$ of sets that will be correspondingly specified. \emph{The rich class of functions in \model{} thus helps model a broad spectrum of semantics.} The mathematical expressions for each function are summarized in \tabref{tab:SIM_inst}. Below, we provide further notations and intuitions for using these functions.
% The abbreviation for each such class, its expansion and its mathematical expression are summarized in Table~\ref{tab:instantiations} in the Appendix.\looseness-1

\begin{table*}[!htb]
    \caption{Instantiations of \model{}. Note that the functions formulate similarity with query set $Q$ and dissimilarity with private set $P$ which are the building blocks for targeted data subset selection.
    \vspace{-2ex}}
    % The parameter $\eta$ can be used to tune the degree of relevance of $\Acal$ to $\Qcal$ and $\nu$ to tune the degree of irrelevance of $\Acal$ to $\Pcal$.
    \label{tab:SIM_inst}
    \begin{subtable}{.45\linewidth}
      \centering
        \caption{Instantiations of MI functions}
        \label{tab:smi_inst}
        \begin{tabular}{|c|c|c|}
        \hline
        \textbf{MI} & \textbf{$I_f(\Acal;\Qcal)$} \\ \hline
        \scriptsize{\textsc{Flvmi}}             & \scriptsize{$\sum\limits_{i \in \Vcal}\min(\max\limits_{j \in \Acal}S_{ij}, \eta  \max\limits_{j \in \Qcal}S_{ij})$}                \\
        \scriptsize{\textsc{Flqmi}}             & \scriptsize{$\sum\limits_{i \in \Qcal} \max\limits_{j \in \Acal} S_{ij} + $ $ \eta \sum\limits_{i \in \Acal} \max\limits_{j \in \Qcal} S_{ij}$}                \\
        \scriptsize{\textsc{Gcmi}}              & \scriptsize{$2 \lambda \sum\limits_{i \in \Acal} \sum\limits_{j \in \Qcal} S_{ij}$}                \\
        \scriptsize{\textsc{Logdetmi}}          & \scriptsize{$\log\det(S_{\Acal}) -\log\det(S_{\Acal} - \eta^2 S_{\Acal,\Qcal}S_{\Qcal}^{-1}S_{\Acal,\Qcal}^T)$}           \\
        \scriptsize{COM}               & \scriptsize{$\eta \sum_{i \in \Acal} \psi(\sum_{j \in \Qcal}S_{ij}) + \sum_{j \in \Qcal} \psi(\sum_{i \in \Acal} S_{ij})$}
        % \scriptsize{$ \sum_{i \in \Acal} \psi(\sum_{j \in \Qcal}S_{ij}) + $} \\ & \scriptsize{$\sum_{j \in \Qcal} \psi(\sum_{i \in \Acal} S_{ij})$} 
        \\ \hline              
        \end{tabular}
    \end{subtable}%
    \begin{subtable}{.6\linewidth}
      \centering
        \caption{Instantiations of CG and CMI functions}
        \label{tab:scg_scmi_inst}
        \begin{tabular}{|c|c|}
        \hline
        \textbf{CG} & \textbf{$f(\Acal|\Pcal)$} \\ \hline
        \scriptsize{\textsc{Flcg}}       & \scriptsize{$\sum\limits_{i \in \Vcal} \max(\max\limits_{j \in \Acal} S_{ij}-$ $ \max\limits_{j \in \Pcal} S_{ij}, 0)$}                \\
        \scriptsize{\textsc{Logdetcg}}       & \scriptsize{$\log\det(S_{\Acal} - \nu^2 S_{\Acal, \Pcal}S_{\Pcal}^{-1}S_{\Acal, \Pcal}^T)$} \\
        \scriptsize{\textsc{Gccg}}   & \scriptsize{$f(\Acal) - 2 \lambda \nu \sum\limits_{i \in \Acal, j \in \Pcal} S_{ij}$}                \\ \hline
        \end{tabular}
        \vspace{0.5ex}
        
        \begin{tabular}{|c|c|}
        \hline
        \textbf{CMI} & \textbf{$I_f(\Acal;\Qcal|\Pcal)$} \\ \hline
        \scriptsize{\textsc{Flcmi}}       & \scriptsize{$\sum\limits_{i \in \Vcal} \max(\min(\max\limits_{j \in \Acal} S_{ij},$ $ \max\limits_{j \in \Qcal} S_{ij})$} \scriptsize{$-  \max\limits_{j \in \Pcal} S_{ij}, 0)$}                \\\scriptsize{\textsc{Logdetcmi}}   & \scriptsize{$\log \frac{\det(I - S_{\Pcal}^{-1}S_{\Pcal, \Qcal} S_{\Qcal}^{-1}S_{\Pcal, \Qcal}^T)}{\det(I - S_{\Acal \cup \Pcal}^{-1} S_{\Acal \cup \Pcal, Q} S_{\Qcal}^{-1} S_{\Acal \cup \Pcal, Q}^T)}$} \\ \hline               
        \end{tabular}
    \end{subtable} 
\end{table*}

\noindent \textbf{Log Determinant (LogDet):} %We state the following lemma and defer its proof to the Appendix~\ref{app:logdet}).
%We refer to MI, CG and CMI applied to the LogDet function as \textsc{Logdetmi}, \textsc{Logdetcg} and \textsc{Logdetcmi} respectively. % and present their expressions in the fourth row of Table~\ref{tab:instantiations}. 
Let $S_{\Acal, \Qcal}$ be the cross-similarity matrix between the items in sets $\Acal$ and $\Qcal$. 
% Also, denote $S_{\Acal\Bcal} = S_{\Acal \cup \Bcal}$. %\todo{Discussed with Vishal the ambiguity in reading resuling from  A, B  (which is $A X B$ matrix) and
%AB (which is $A \cup B X A \cup B$ matrix) Have suggested we use $S_x$ for one and $S^x$ for the other} 
%Note the parameters $\eta$ and $\nu$ which we explain next. 
%\todo{Some dangling sentences such as this. I have tried addressing few, but I have no idea what to do with few such cases. Still I have tried rephrasing this} 
We construct a similarity matrix $S^{\eta,\nu}$ (on a base matrix $S$) in such a way that the cross-similarity between $\Acal$ and $\Qcal$ is multiplied by $\eta$ ({\em i.e.}, $S^{\eta,\nu}_{\Acal,\Qcal} = \eta S_{\Acal,\Qcal}$) to control the trade-off between query-relevance and diversity. Similarly, the cross-similarity between $\Acal$ and $\Pcal$ by $\nu$ ({\em i.e.}, $S^{\eta,\nu}_{\Acal,\Pcal} = \nu S_{\Acal,\Pcal}$) to control the strictness of privacy constraints. Higher values of $\nu$ ensure stricter privacy constraints, such as  in the context of privacy-preserving summarization, by tightening the extent of dissimilarity of the subset from the private set. Given the standard form of LogDet as $f(\Acal) = \log\det(S^{\eta,\nu}_{\Acal})$, we provide the \model{} expressions in \tabref{tab:SIM_inst}. For simplicity of notation, CMI is presented with $\nu = \eta = 1$. We defer the derivation and the general expression for CMI to the Appendix~\ref{app:logdet}.\looseness-1  %in the Lemma below, and defer the general expression and proof to 
%, and $S_{\Acal \Pcal\nu}$ as the similarity matrix obtained by multiplying $\nu$ to the cross similarity entries. Similarly, denote $S_{\Acal \Pcal\nu, Q\eta}$ as the cross similarity obtained by multiplying $\nu$ to the cross similarity between $\Acal$ and $\Pcal$ and 

%COMMENTED BECAUSE EXPRESSION IS NOW IN TABLE
% \vspace{-1ex}
% \begin{lemma} \label{lemma:logdet}
% Using a similarity matrix defined above and with $f(\Acal) = \log\det(S^{\eta,\nu}_{\Acal})$, we have: $I_f(\Acal; \Qcal) = \log\det(S_{\Acal}) -\log\det(S_{\Acal} - \eta^2 S_{\Acal,\Qcal}S_{\Qcal}^{-1}S_{\Acal,\Qcal}^T)$, and $f(\Acal | \Pcal ) = \log\det(S_{\Acal} - \nu^2 S_{\Acal, \Pcal}S_{\Pcal}^{-1}S_{\Acal, \Pcal}^T)$. Similarly, $I_f(\Acal; \Qcal |\Pcal ) = \log \frac{\det(I - S_{\Pcal}^{-1} S_{\Pcal, \Qcal} S_{\Qcal}^{-1} S_{\Pcal, \Qcal}^T)}{\det(I - S_{\Acal \Pcal}^{-1} S_{\Acal \Pcal, \Qcal} S_{\Qcal}^{-1} S_{\Acal \Pcal, \Qcal}^T)}$
% \end{lemma}

%\vspace{-0.5ex}
\noindent \textbf{Facility Location (FL):} We introduce two variants of the MI functions for the FL function which is defined as: $f(\Acal) = \sum_{i \in \Omega} \max_{j \in \Acal} S_{ij}$. The first variant is defined over $\Vcal$ (FL\textbf{V}MI) ~\cite{iyer2021submodular}, in \tabref{tab:SIM_inst}(a). % and presented in the fifth row of Table~\ref{tab:instantiations}. 
We derive another variant defined over $\Qcal$ (FL\textbf{Q}MI) which considers only cross-similarities between data points and the target. This MI expression has interesting characteristics different from those of \textsc{Flvmi}. In particular, whereas \textsc{Flvmi} gets saturated ({\em i.e.}, once the query is satisfied, there is no gain in picking another query-relevant data point), \textsc{Flqmi} just models the pairwise similarities of target to data points and vice versa. Moreover, \textsc{Flqmi} only requires a $\Qcal \times \Vcal$ kernel, which makes it very efficient to optimize. We provide the expressions in \tabref{tab:SIM_inst} and defer the derivations to Appendix~\ref{app:fl}. %In this case, the expression for CG and CMI don't make sense, since they require computing terms over $\Vcal^{\prime}$, which we do not have access to. 
We multiply the similarity kernel $S$ used in MI and CG expressions of FL by $\eta$ and $\nu$ as done in the case of LogDet.\looseness-1
%We get these by appropriately multiplying the cross-similarities between $\Acal$ and $\Qcal$ by them.

% \vspace{-1ex}
% \begin{lemma} \label{lemma:fl}
% Let $S$ be the similarity kernel such that $S_{ij} = I(i == j)$, if both $i, j \in \Vcal$ or both $i, j \in \Vcal^{\prime}$ and the facility location function $f(\Acal) = \sum_{i \in \Omega} \max_{j \in \Acal} S_{ij}, \Acal \subseteq \Omega$. The expression for MI (\textsc{Flqmi}) is obtained to be $I_f(\Acal; \Qcal) = \sum_{i \in \Qcal} \max_{j \in \Acal} S_{ij} + \eta \sum_{i \in \Acal} \max_{j \in \Qcal} S_{ij}$. The CG and CMI expressions are not particularly useful in this case.
% \end{lemma}

%\vspace{-1ex}

%\vspace{-0.5ex}
\noindent \textbf{Concave Over Modular (COM): } The notion of GMI functions (\secref{sec:gmi}) allows us to characterize a rich class of concave over modular functions as GMI functions. Define a set function $f_{\eta}(\Acal)$ as: $f_{\eta}(\Acal) = \eta \sum_{i \in \Vcal^{\prime}} \max(\psi(\sum_{j \in \Acal \cap \Vcal} S_{ij}), \psi(\sqrt{n}\sum_{j \in \Acal \cap \Vcal^{\prime}} S_{ij})) \nonumber +  \sum_{i \in \Vcal} \max(\psi
    (\sum_{j \in \Acal \cap \Vcal^{\prime}} S_{ij}), \psi(\sqrt{n}\sum_{j \in \Acal \cap \Vcal} S_{ij}))$, 
%\vspace{-1ex}
where $\psi$ is a concave function and $f_{\eta}(\Acal)$ is restricted submodular. We state the expression for its GMI function in \tabref{tab:SIM_inst}(a) and provide the derivation in Appendix~\ref{app:com}.

\noindent \textbf{Graph Cut (GC): } The GC function is defined as $f(\Acal) = \sum \limits_{i \in \Acal, j \in \Vcal} S_{ij} - \lambda \sum\limits_{i, j \in \Acal} S_{ij}$. The parameter $\lambda$ captures the trade-off between diversity and representativeness. The \model{} expressions of GC are presented in \tabref{tab:SIM_inst}. Note that the CMI expression for GC is not useful as it does not involve the private set and is exactly the same as the MI version (proof in Appendix~\ref{app:gccsmi}). Like in the LogDet case, we introduce an additional parameter $\nu$ in \textsc{Gccg} to control the strictness of privacy constraints. Again, this is easily modeled in the GC objective by multiplying the cross-similarity between data points and the private instances by $\nu$. 

\noindent \textbf{Computational Complexity: } In terms of compute complexity, \textsc{Gcmi} and \textsc{Flqmi} are linear in $|\Vcal|$ (since $\Qcal$ is typically small). However, \textsc{Flvmi} and \textsc{Logdetmi} are quadratic in the size of the unlabeled set due to requiring the kernel. Hence, for massive datasets, \textsc{Gcmi} and \textsc{Flqmi} are preferable (more details in Appendix~\ref{app:scalability}). We also discuss (in Appendix~\ref{app:scalability}) a partitioning based approach where we divide the datasets into smaller partitions and run the $\Vcal \times \Vcal$ kernel based functions (\textsc{Flvmi} and \textsc{Logdetmi}) on the individual partitions, thereby making them more scalable.

\begin{figure}
%\vspace{-4ex}
    \centering 
  \includegraphics[width=0.47\textwidth]{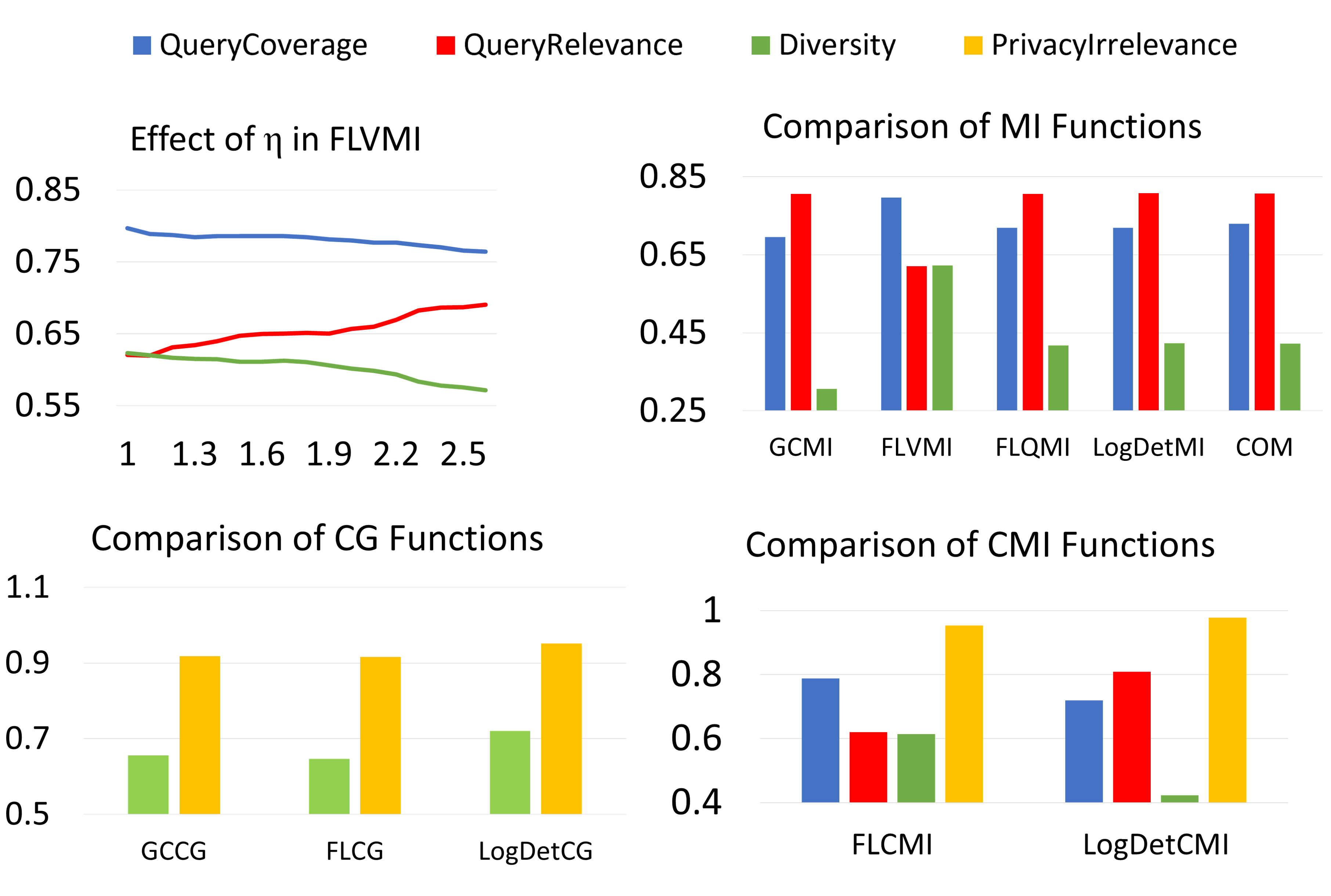}
%\vspace{-2ex}
\caption{Behavior of different functions in \model{} and effect of parameters. All plots share the legend. }
\vspace{-2ex}
\label{fig:syntheticcombined}
\end{figure}

\subsection{Modeling Semantics of \model} 
To empirically verify the intuitive understanding of the expressions, we maximize the different functions in \model{} individually on a synthetically created dataset with different parameters, and study the characteristics of the subsets qualitatively and quantitatively. For evaluation, we define \emph{query-coverage} to be the fraction of queries covered by the subset, \emph{query-relevance} to be the fraction of the subset pertaining to the queries, \emph{diversity} to be the measure of how diverse are the points within the selected subset, and \emph{privacy-irrelevance} to be the fraction of the subset \emph{not} matching the private set. We present representative results in Fig.~\ref{fig:syntheticcombined} and defer detailed results to Appendix~\ref{app:synthetic}. For MI functions, we verify that increasing $\eta$ tends to increase query-relevance while reducing query-coverage and diversity (top-left, Fig.~\ref{fig:syntheticcombined}). Also, while \textsc{Gcmi} lies at one end of the spectrum favoring query-relevance, \textsc{Flvmi} lies at the other end favoring diversity and query coverage. \textsc{Flqmi}, \textsc{Logdetmi} and COM lie somewhere in between (top-right, Fig.~\ref{fig:syntheticcombined}). As expected, increasing $\nu$ increases privacy-irrelevance for CG functions. We also observe that \textsc{Logdetcg} outperforms \textsc{Flcg} and \textsc{Gccg} both in terms of diversity and privacy-irrelevance (bottom-left, Fig.~\ref{fig:syntheticcombined}). For CMI functions, we see that \textsc{Flcmi} tends to favor query-coverage and diversity in contrast to query-relevance and privacy-irrelevance, while \textsc{Logdetcmi} favors query-relevance and privacy-irrelevance over query-coverage and diversity (bottom-right, Fig.~\ref{fig:syntheticcombined}).

\vspace{-1ex}
\section{\model{} for Guided Subset Selection}
\label{sec:tasks}
In this section, we discuss the use of \model{} for guided data subset selection and illustrate its utility for targeted learning and guided summarization. 

\subsection{Targeted Learning}
\label{sec:tss}

We first apply \model{} to \emph{targeted learning}, where the goal is to improve a model's accuracy on some target classes at a given additional labeling cost and without compromising on the overall accuracy (see \figref{fig:motivating-applications}(a)). Let $\Ecal$ be an initial training set of labeled instances, $\Tcal$ be the target set of examples on which the user desires good performance, $\Pcal$ be the private set of examples that the user wants to avoid, and $\Ucal$ be a large unlabeled dataset. 
We maximize a CMI function $I_f(\Acal; \Tcal | \Pcal)$ towards computing an optimal subset $\hat{\Acal} \subseteq \Ucal$ of size $k$ \textit{similar} to $\Tcal$ and \textit{dissimilar} to $\Pcal$. Note that when $\Pcal= \emptyset$, CMI is equivalent to MI. We then augment $\Ecal$ with labeled $\hat{\Acal}$ and re-train the model. % to achieve better accuracy without compromising on the accuracy of other classes. 
Through the aforementioned (\secref{sec:model}), the diverse class of MI functions in \model{} offers a natural and effective approach for targeted subset selection by using the query set $\Qcal$ as the target set $\Tcal$. The approach is outlined in \AlgRef{algo:tss}. 
% It may be noted that Step 2 could involve using any feature vectors for the sake of computing the kernels in Step 3. 
Similar to~\cite{ash2020deep, killamsetty2020glister,killamsetty2021grad,mirzasoleiman2020coresets}, we use last-layer gradients of the model to represent the data points and use them to compute the similarity kernel $S$. Specifically, we define pairwise similarities $S_{ij} = \langle \nabla_{\theta} \Lcal_i(\theta), \nabla_{\theta} \Lcal_j(\theta) \rangle$, where $\Lcal_i(\theta) = \Lcal(x_i, y_i, \theta)$ is the loss on the $i$th data point and $\theta$ denotes model parameters.
% owing to their recent success~\cite{ash2020deep, killamsetty2020glister,killamsetty2021grad,mirzasoleiman2020coresets, kothawade2021similar}. 
%The CMI function may also be combined with a function ($\gamma g(\Acal)$) explicitly modeling diversity (say, Disparity-Sum or total pairwise sum), which is useful for functions such as \textsc{Gcmi} that do not model diversity by themselves. 
Note that the target need not correspond to class(es) but could be any attribute of data that the user is interested in. For instance, the target could be mining images with people at night (here \textit{night} is an attribute). \looseness-1

\begin{algorithm}%[H]
\begin{algorithmic}[1]
\REQUIRE $\Ecal$: initial labeled set, $\Ucal$: large unlabeled set, $\Tcal$: a target subset/slice, $\Pcal$: a set to be avoided, $k$: selection budget, $\mathcal L$: loss function
%\STATE $Train\gets$ initial training set of labeled instances
%\STATE $\Ecal\gets$ large data lake containing unlabeled instances
%\STATE $Target\gets$ examples of types on which better accuracy is desired
\STATE Train model with loss $\mathcal L$ on labeled set $\Ecal$ to obtain model parameters $\theta_\Ecal$ \textcolor{blue}{\{Obtain initial accuracy\}}
\STATE Compute the gradients $\{\nabla_{\theta_\Ecal} \mathcal L(x_i, y_i), i \in \Ucal\}$ (using hypothesized labels) and $\{\nabla_{\theta_\Ecal} \mathcal L(x_i, y_i), i \in \Tcal\}$  \textcolor{blue}{\{Obtain vectors for computing kernel in Step 3\}}
\STATE Compute the similarity kernels $S$ 
% (this includes kernel of the elements within $\Ucal$, within $\Tcal$ and between $\Ucal$ and $\Tcal$) 
and define a CMI function $I_f(\Acal;\Tcal | \Pcal)$ 
% and a diversity function $g$ 
\textcolor{blue}{\{Instantiate Functions\}}
\STATE $\hat{\Acal} \gets \max_{\Acal \subseteq \Ucal, |\Acal|\leq k} (I_f(\Acal;\Tcal|\Pcal)$ 
% + \gamma g(\Acal))$ 
\textcolor{blue}{\{Obtain the subset  optimally matching the target\}}
\STATE Obtain the labels of the elements in $\hat{\Acal}$ as $L(\hat{\Acal})$ \textcolor{blue}{\{Procure labels on the selected subset\}}
\STATE Train a model on the combined labeled set $\Ecal \cup L(\hat{\Acal})$ \textcolor{blue}{\{Augment training data\}} % for improved accuracy on target
\end{algorithmic}
\caption{Application of \model{} for\textbf{ Targeted Learning}}
\vspace{-1ex}
\label{algo:tss}
\end{algorithm}
\vspace{1ex}

\subsection{Guided Summarization}
\label{sec:tsum} 

Next, we apply \model{} for \emph{guided summarization}. In this task, we are given a set $\Vcal$ of data points (images, sentences of a document, or frames/shots of a video), and the goal is to find a summary $\Acal \subseteq \Vcal$ with certain desired characteristics. In query-focused summarization, we find a summary that is semantically similar to the query set, while in privacy-preserving summarization, the obtained summary should \textit{not} contain data points that are similar to the private set. \looseness-1 
% The \emph{target} here assumes different semantics for different flavors of summarization - query set ($\Qcal$) for query-focused summarization and private set ($\Pcal$) for privacy-preserving summarization. 
% In the context of update-summarization, the \emph{target} is $\Acal_0$, the summary that the user has already seen, and the goal is to find a summary different from $\Acal_0$. 

%\vspace{-2ex}
\noindent \textbf{\model's Unified Framework for Guided Summarization:}
%\label{subsec:unified-summ}
%\vspace{-1ex}
Given sets $\Qcal$ and $\Pcal$, and a submodular function $f$, consider the following {\it master optimization problem} involving CMI:
%\begin{align}\label{eqn:master-summ-prob}
    $\max_{\Acal: |\Acal| \leq k}  I_f(\Acal; \Qcal | \Pcal)$. 
%\end{align}
We discuss how the different flavors of summarization can be seen as special cases of this master optimization problem. Setting $\Qcal \leftarrow \Vcal$ and $\Pcal \leftarrow \emptyset$ yields generic summarization. Similarly, setting $\Qcal \leftarrow \Qcal$ and $\Pcal \leftarrow \emptyset$ yields query-focused summarization with a query-set $\Qcal$. Setting $\Qcal \leftarrow \Vcal$ and $\Pcal \leftarrow \Pcal$ gives privacy-preserving summarization. 
% and for update-summarization we set $\Tcal = \Acal_0, \Bcal = \Vcal$. 
This framework allows us to address yet another flavor: \emph{joint query-focused and privacy preserving summarization} where we set $\Qcal \leftarrow \Qcal$ and $\Pcal \leftarrow \Pcal$.

\noindent \textbf{Parameter Learning in \model{} for Guided Summarization:} As discussed in \secref{sec:model}, different instantiations of \model{} along with their parameters offer a wide spectrum of modeling characteristics. Hence, when used individually, each imparts certain characteristics to the summaries. For summarization, we thus propose learning a mixture model supervised by summaries generated by humans. We learn a mixture of \model{} functions~\cite{lin2012learning,Kaushal2019DemystifyingMV,Kaushal2019AFT,tschiatschek2014learning} where the weights and the internal parameters $\lambda, \nu, \eta$ of the functions are jointly learned. We denote our parameter vector by $\Theta = (w, \eta, \lambda, \nu)$, and our \model{} mixture model by $F(\Theta)=\sum_i w_i f_i(\Acal, \gamma, \eta, \nu)$, with each $f_i$ being either one of the functions in \model{} or one of pure diversity and representation functions such as Disparity-Sum and FL. %respectively. 
Then, given $N$ training examples, $(\Vcal^{(n)}, \Ycal^{(n)})_{n = 1}^N$ 
%where $V^{(n)}$ is $n^{th}$ video or image collection or text document, $\Ycal^{(n)}$ is the human summary and 
we apply gradient descent to learn the parameters $\Theta$ by optimizing the following large-margin formulation: $\min\limits_{\Theta \geq 0} \frac{1}{N} \sum_{n=1}^{N} \Lcal_n(\Theta) + \frac{\lambda}{2}||\Theta||^2$, where $\Lcal_n(\Theta)$ is the generalized hinge loss associated with training example $n$: $\Lcal_n(\Theta) = \max\limits_{\Ycal \subset \Vcal^{(n)}, |\Ycal| \leq k} (F(\Ycal, x^{(n)}, \Theta) + l_n(\Ycal)) - F(\Ycal^{(n)}, x^{(n)}, \Theta)$. Here, $\Ycal^{(n)}$ is a human summary for the $n^{th}$ ground set $\Vcal^{(n)}$  (video, image collection, or text document), with corresponding features $x^{(n)}$. %This objective is chosen so that each human reference summary scores higher than any other summary by some margin $l_n(\Ycal)$. 
%The parameters  are learnt using gradient descent. 
The specific objective functions and gradient computations in case of query-focused, privacy-preserving, and joint query-focused and privacy-preserving summarizations are presented in Appendix~\ref{app:learning}. For generic summarization, we add the standard submodular functions modeling representation, diversity, coverage, {\em etc.} in the mixture. For query-focused summarization and privacy-preserving summarization, we instead use the MI and CG versions of the functions respectively as defined above\footnote{For the query-focused and the privacy-preserving case, CMI degenerates to MI and CG, respectively.}. Once the parameters are learned, we instantiate the mixture model with the parameters and maximize it to get the desired summaries. \looseness-1

\subsection{Connections to Past Work} \label{sec:connections}
\model{} generalizes past work in both targeted learning and guided summarization. We summarize the connections below; for more details, see Appendix~\ref{app:connections}.

\noindent \textbf{Targeted Learning: } A number of approaches like \textsc{Glister}~\cite{killamsetty2020glister} and \textsc{Grad-Match}~\cite{killamsetty2021grad} can be considered with a validation set, and hence used in the targeted setting. Similarly, \textsc{Craig}~\cite{mirzasoleiman2020coresets} can be extended to consider a validation set. In Appendix~\ref{app:connectionsTL}, we show that \AlgRef{algo:tss} in fact generalizes \textsc{Craig} (using \textsc{Flqmi}), \textsc{Glister} (using \textsc{Com}), and \textsc{Grad-Match} (using \textsc{Gcmi} + Diversity).% when considered with the targeted set as the validation set. 

\noindent \textbf{Guided Summarization: } A number of past works for summarization have inadvertently used instances of \model. The query-DPP considered in~\cite{sharghi2016query,sharghi2017query} is a special case of \textsc{Logdetmi}. Similarly, the graph-cut based query-relevance term in~\cite{vasudevan2017query,lin2012submodularity}, and in ~\cite{li2012multi} is actually \textsc{Gcmi}. %, while the submodular function used by~\cite{li2012multi} in update-summarization is \textsc{Gccg}. 
Furthermore, the joint diversity and query-relevance term in~\cite{lin2011class} is an instance of COM (with the square-root as the concave function, see Appendix~\ref{app:sum-gen}).\looseness-1

\begin{figure*}[ht!]
\centering
\includegraphics[width = 16cm, height=1cm]{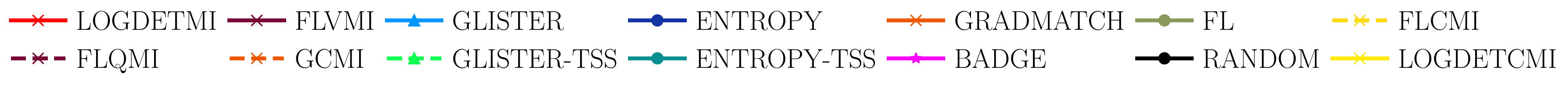}
\centering
% \begin{center}
\hspace*{-0.6cm}
\begin{subfigure}[]{0.33\textwidth}
\includegraphics[width = \textwidth]{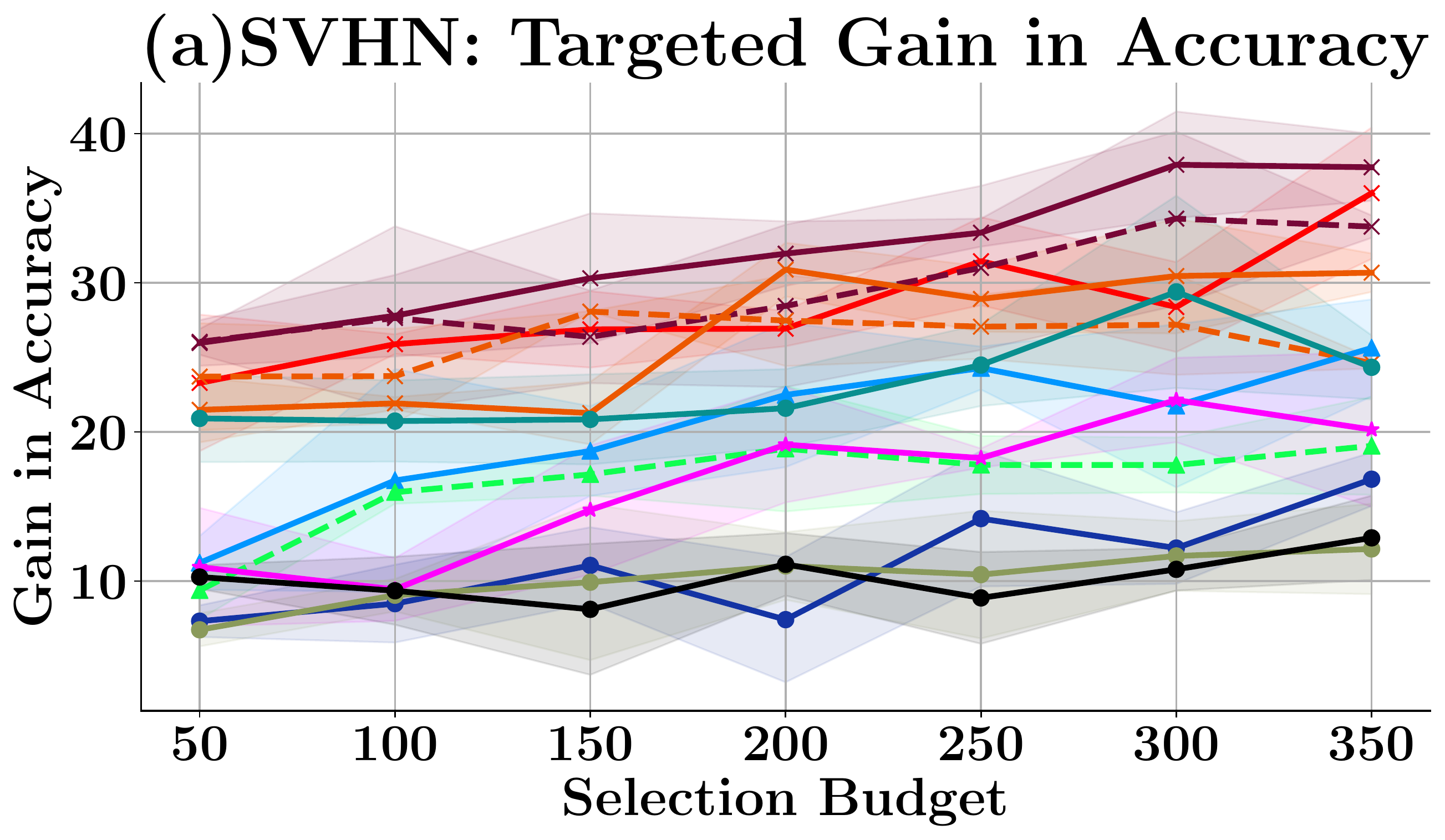}
\end{subfigure} 
\begin{subfigure}[]{0.33\textwidth}
\includegraphics[width = \textwidth]{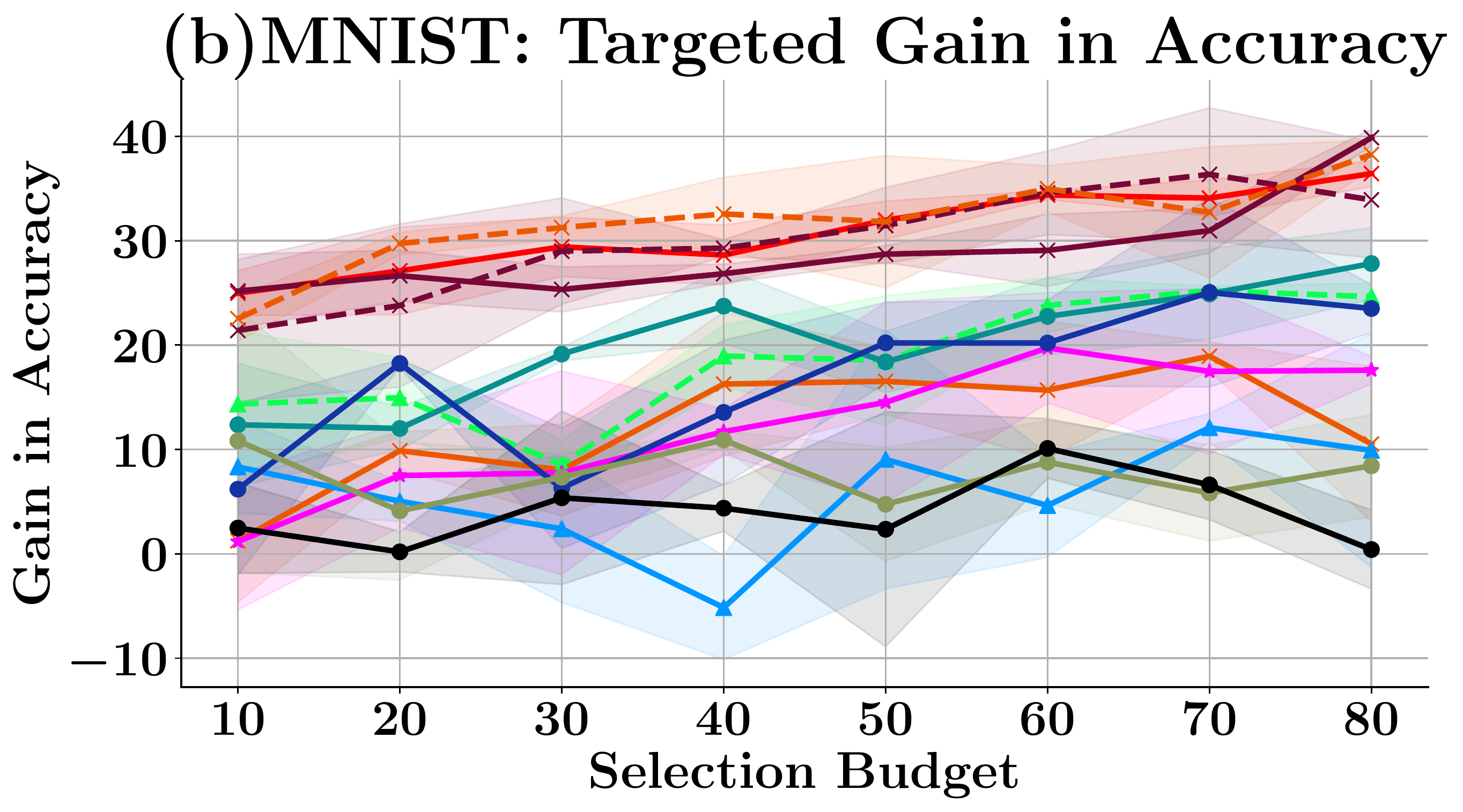}
\end{subfigure}
\begin{subfigure}[]{0.33\textwidth}
\includegraphics[width = \textwidth]{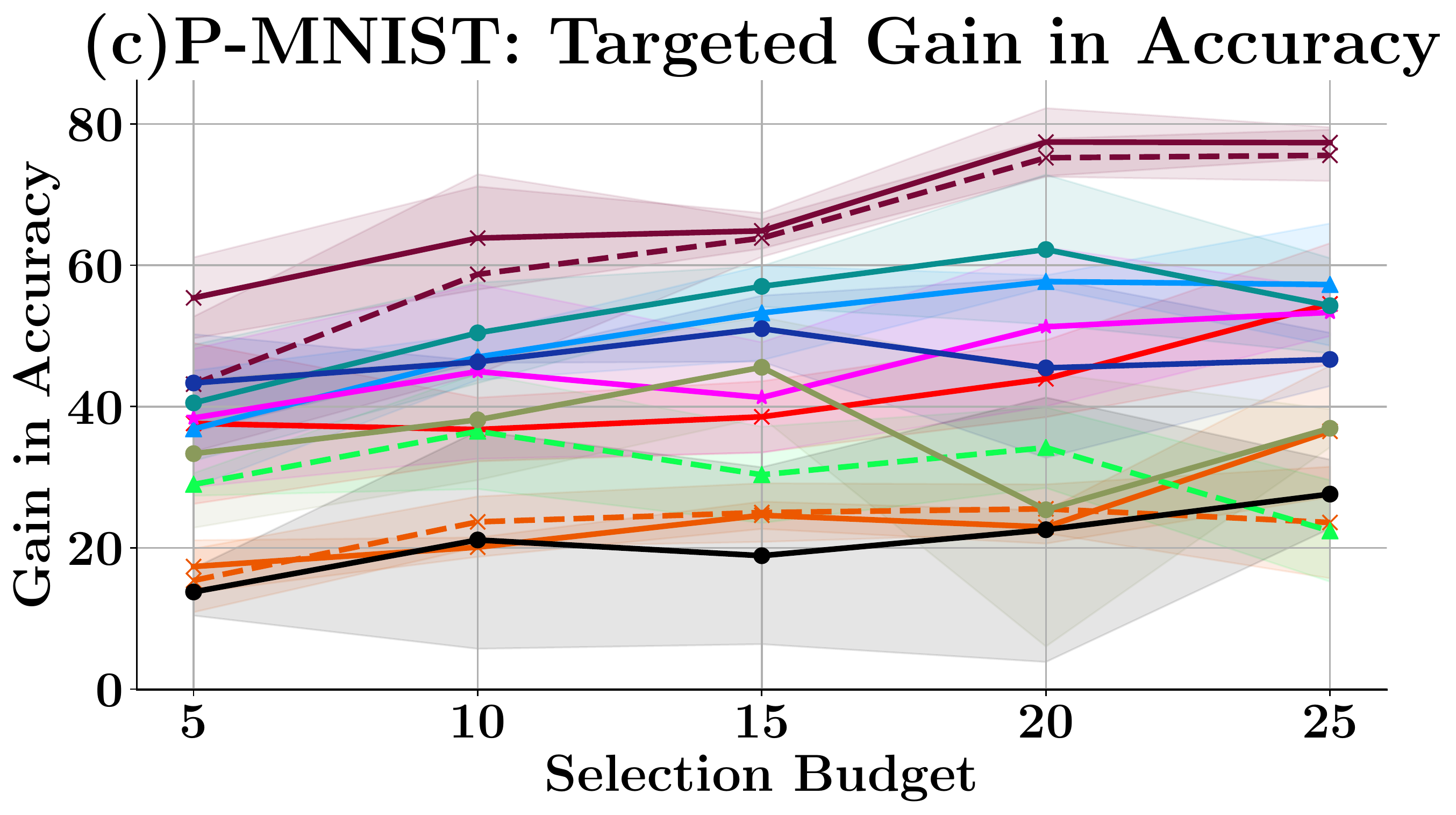}
\end{subfigure}
\begin{subfigure}[]{0.33\textwidth}
\includegraphics[width = \textwidth]{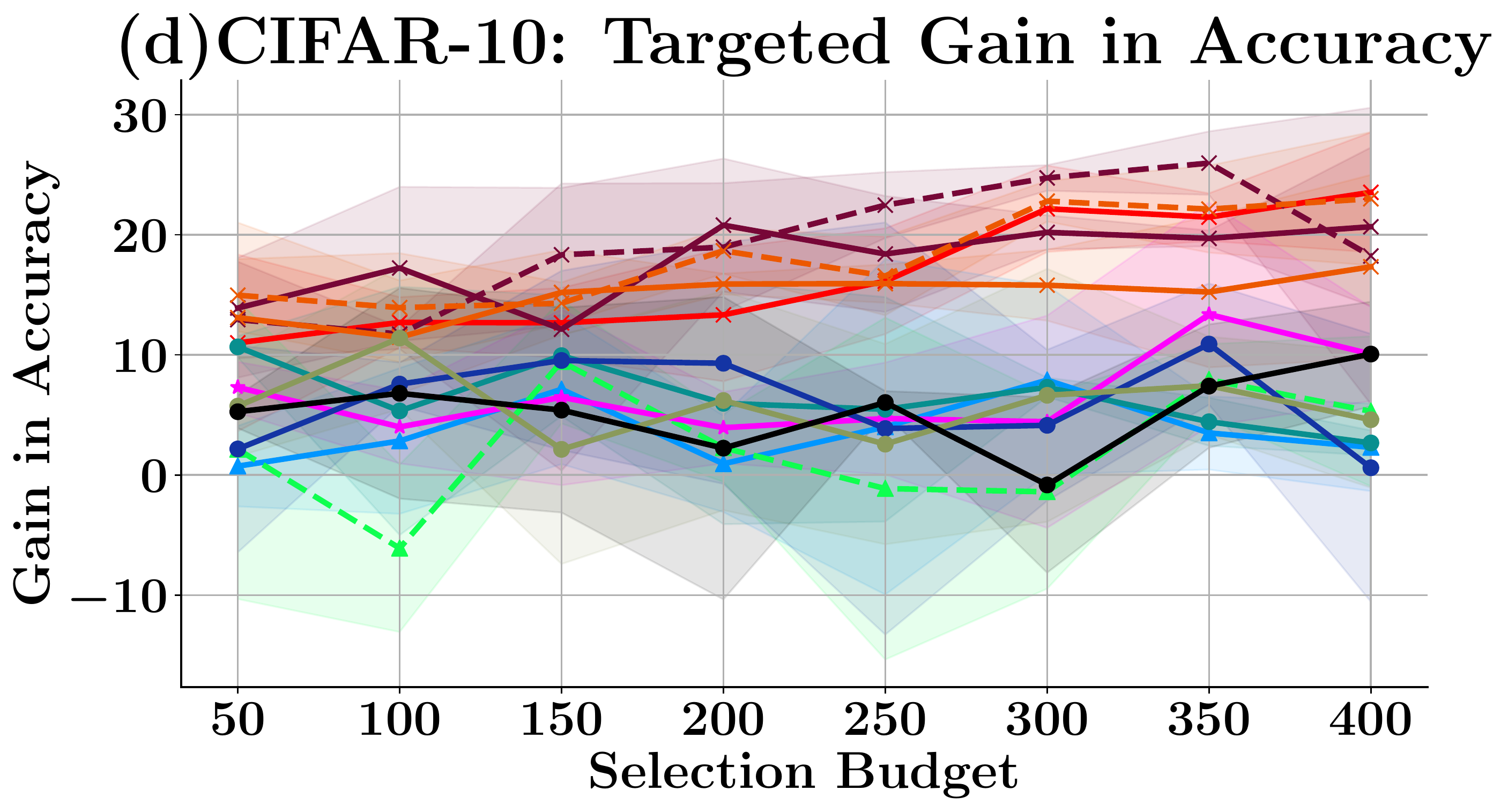}
\end{subfigure}
\begin{subfigure}[]{0.33\textwidth}
\includegraphics[width = \textwidth]{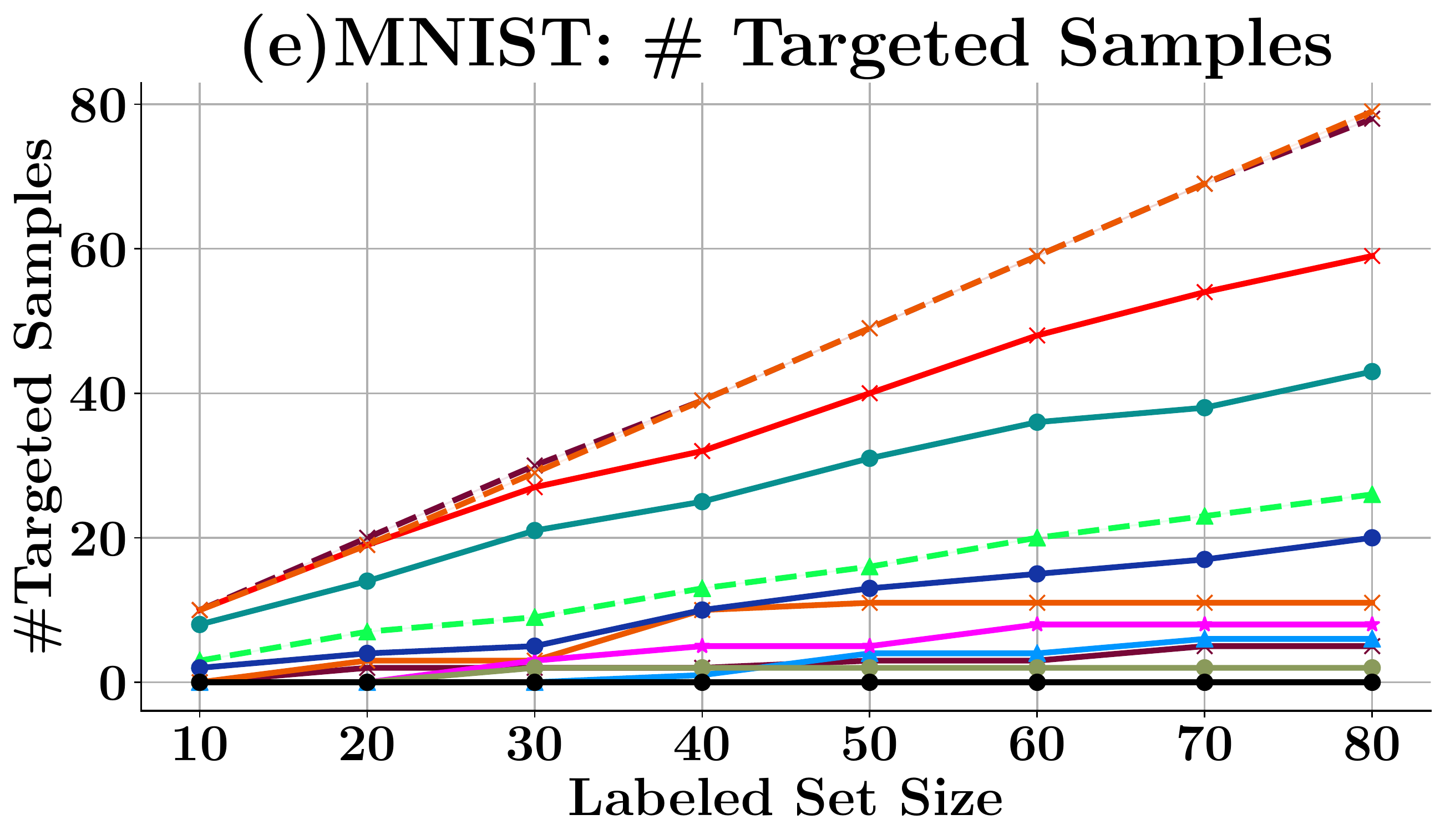}
\end{subfigure} 
\begin{subfigure}[]{0.325\textwidth}
\includegraphics[width = \textwidth]{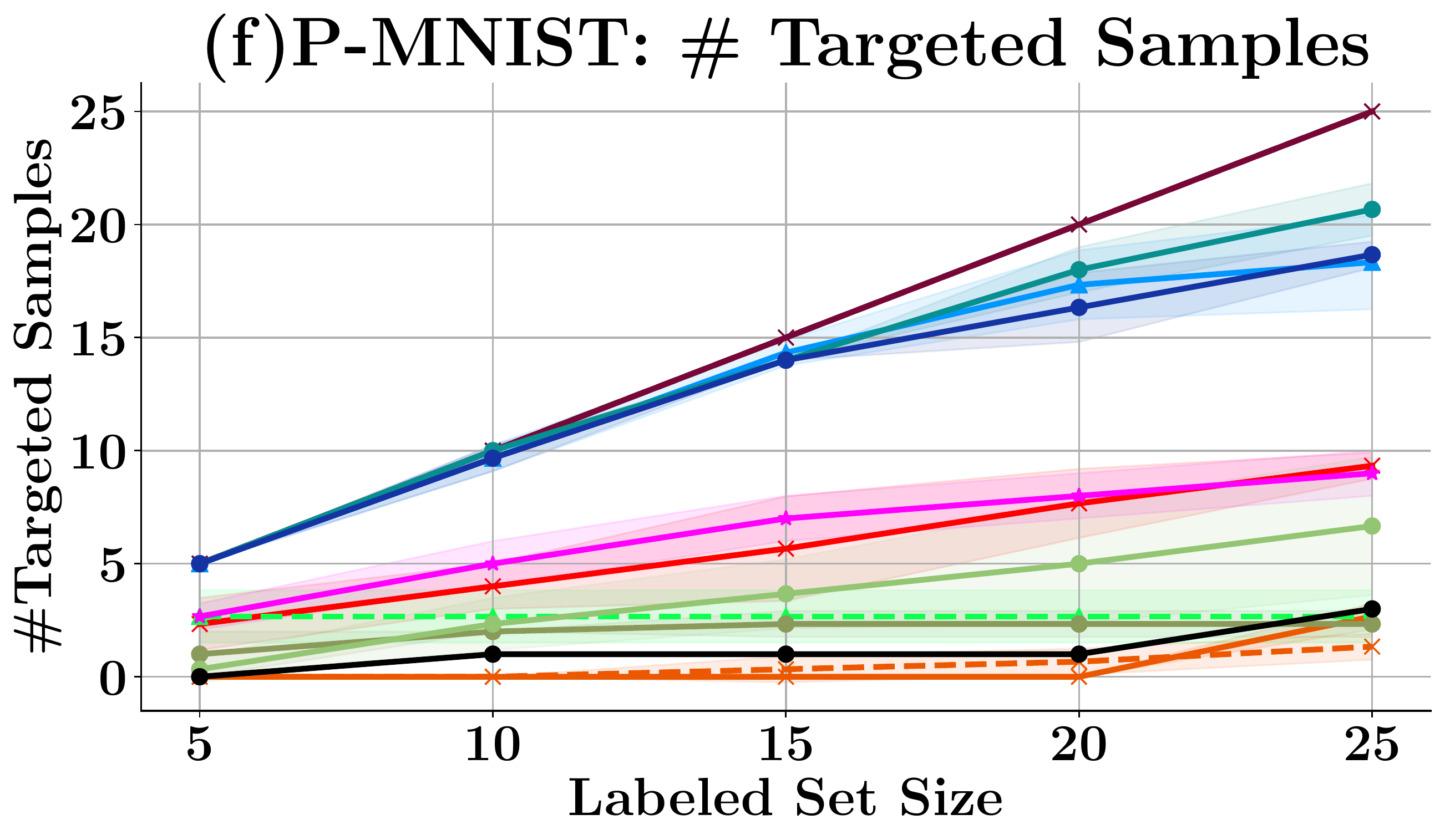}
\end{subfigure}
\begin{subfigure}[]{0.33\textwidth}
\includegraphics[width = \textwidth]{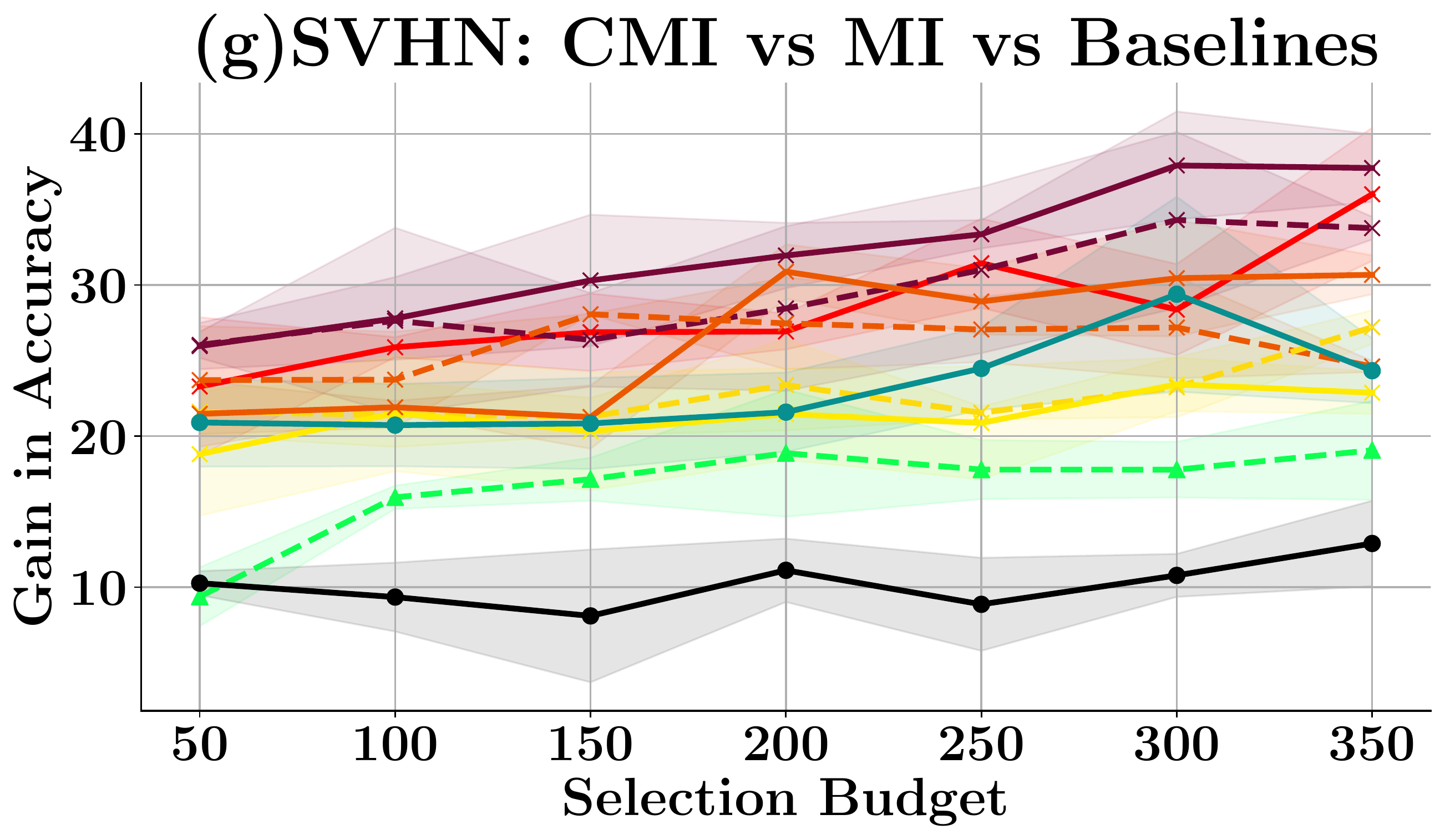}
\end{subfigure} 
\begin{subfigure}[]{0.33\textwidth}
\includegraphics[width = \textwidth]{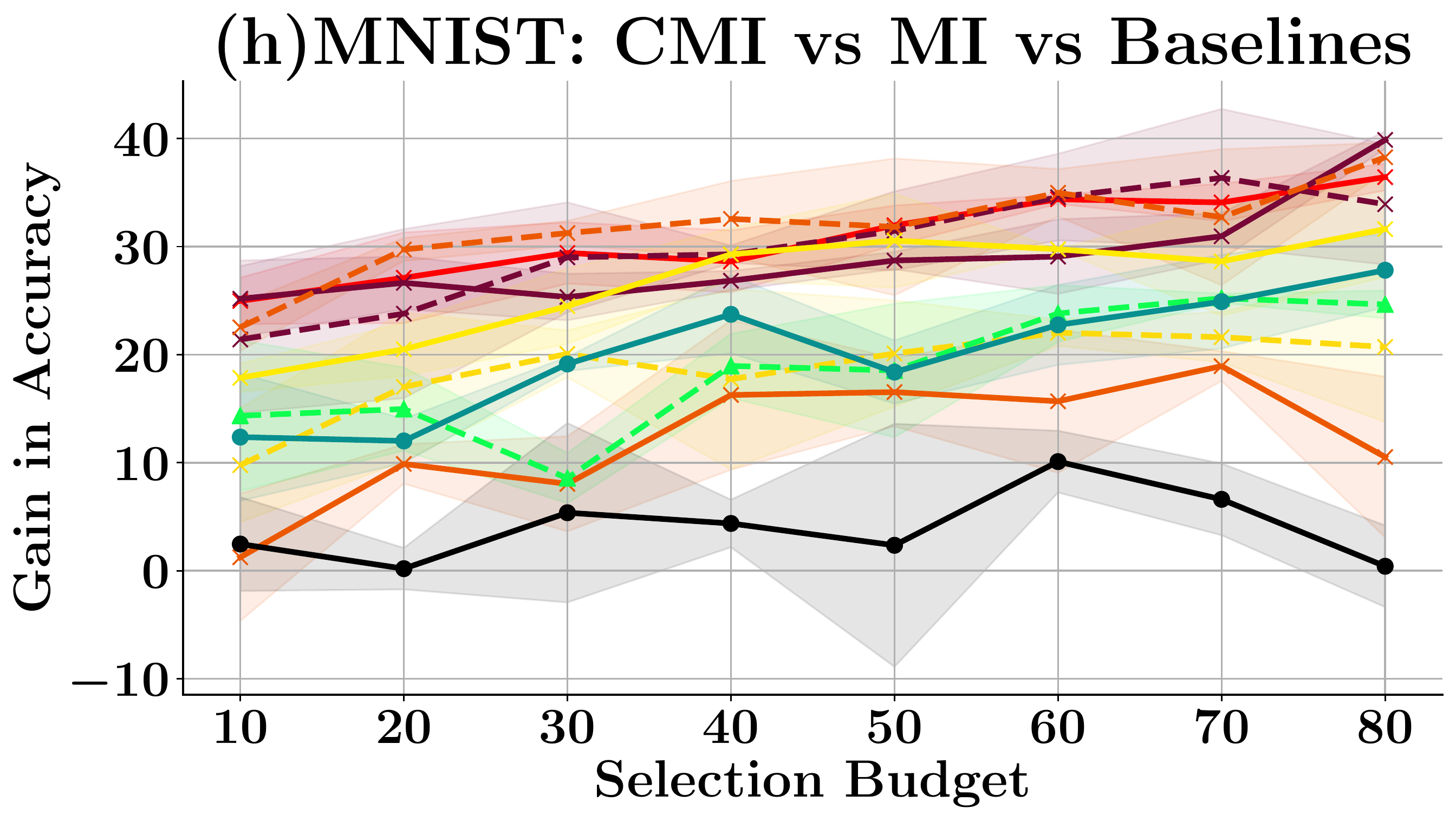}
\end{subfigure}
\begin{subfigure}[]{0.33\textwidth}
\includegraphics[width = \textwidth]{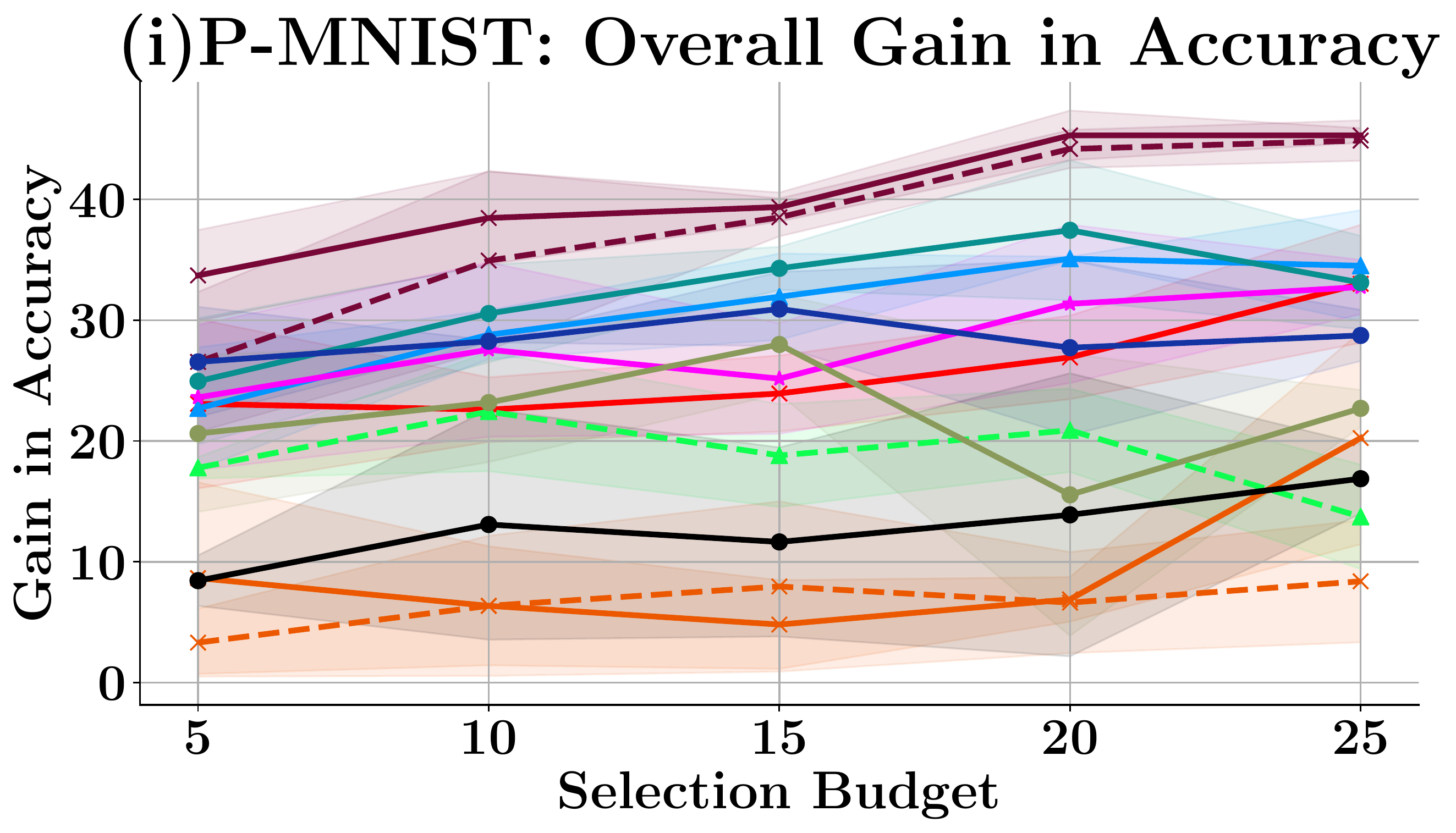}
\end{subfigure}
% \end{center}
\caption{Targeted learning with \model{} on SVHN, MNIST, P-MNIST (Pneumonia-MNIST) and CIFAR-10. Plots (a-d) compare average gain in accuracy on targeted classes. Plots (e-f) compare the number of data points selected from the targeted classes. Plots (g-h) show an ablation study to compare the performance of MI and CMI functions. Plot (i) shows the overall gain in accuracy on P-MNIST. The MI and CMI functions (specifically \textsc{Flqmi} and \textsc{Flvmi}) obtain larger gains in targeted (plots a-d) and overall accuracy (plot i) than other baselines by selecting larger and more diverse examples from targeted classes (plots e-f).
\vspace{-3ex}}
\label{fig:TL}
\end{figure*}

\vspace{-1ex}
\section{Experiments and Results}
\label{sec:exp}
\subsection{Targeted Learning}\label{sec:exp-tss}
In this section, we demonstrate the effectiveness of \model{} for targeted learning on the CIFAR-10~\cite{krizhevsky2009learning}, MNIST~\cite{lecun1998gradient}, SVHN~\cite{netzer2011reading}, and P-MNIST (Pneumonia-MNIST)~\cite{yang2021medmnist, kermany2018identifying} image classification datasets.

\noindent\textbf{Custom dataset: }  To simulate a real-world setting, we randomly select some {\em target} classes and split the train set into {\em labeled}, {\em target}, and an {\em unlabeled} set such that (i) the {\em labeled} set has \emph{class imbalance} and poorly represents the target classes, 
%two randomly picked classes (target), 
%(ii) the poorly represented classes do not perform well on the validation set and 
(ii) the {\em target} set has a small number of data points from the target classes, and (iii) the {\em unlabeled} set is a large set whose labels we do not use (resembling a large pool of unlabeled data in real-world scenarios). For CMI functions, we additionally use a {\em private set}, which has a small number of data points from the non-target classes. The performance is measured on the standard test set from the respective datasets. Let the set $\Ccal$ consist of data points from the target classes and the set $\Dcal$ consist of data points from the non-target classes. We create an initial labeled set $\Ecal$, such that $|\Dcal_\Ecal| = \rho|\Ccal_\Ecal|$ and an unlabeled set which follows the same distribution $|\Dcal_\Ucal| = \rho|\Ccal_\Ucal|$, where $\rho$ is the imbalance ratio. We use $\rho=20$ and $|\Tcal|=10$ (total number of samples from target classes) for all experiments. For CIFAR-10, MNIST and SVHN, we randomly select 2 classes as targets, while for the binary classification task in P-MNIST, we select the \emph{pneumonia} class as the target. For MNIST and SVHN, $|\Ccal_\Ecal| + |\Dcal_\Ecal| = 1620$, $|\Ccal_\Ucal| + |\Dcal_\Ucal| = 24.3K$. For CIFAR-10, $|\Ccal_\Ecal| + |\Dcal_\Ecal| = 8400$, $|\Ccal_\Ucal| + |\Dcal_\Ucal| = 24.3K$. For P-MNIST, $|\Ccal_\Ecal| + |\Dcal_\Ecal| = 105$, $|\Ccal_\Ucal| + |\Dcal_\Ucal| = 1100$. 
These data splits were chosen to simulate low accuracy on target classes and at the same time to maintain the imbalance ratio in labeled and unlabeled datasets.
%For computing the kernels in Algorithm~\ref{algo:tss} we use the gradients $\{\nabla_{\theta_E} \mathcal L(x_i, y_i), i \in \Ucal\}$ (using hypothesized labels) and $\{\nabla_{\theta_E} \mathcal L(x_i, y_i), i \in \Tcal\}$.

% \begin{figure}
%         \centering 
%   \includegraphics[width=0.49\textwidth]{neurips/images/budget_effect_corrGrad.pdf}
% %\vspace{-2ex}
% \caption{Comparison of different methods for targeted subset selection for different budgets on CIFAR-10 and MNIST. X-axis: budgets, Y-axis: gain in model accuracy for target classes. MI based approaches (lines in {\color{BrickRed}{red})} significantly outperform others across all subset sizes. Better resolution image available in Appendix~\ref{app:tss-exp}.}
% \vspace{-2ex}
% \label{fig:gain-size}
% \end{figure}

\begin{figure*}[ht!]
    \centering 
% \begin{subfigure}{0.24\textwidth}
%   \includegraphics[width=\linewidth]{images/generic_final.png}
%   \caption{}
%   \label{fig:generic-summ}
% \end{subfigure}\hfil % <-- added
\begin{subfigure}{0.3\textwidth}
  \includegraphics[width=\linewidth]{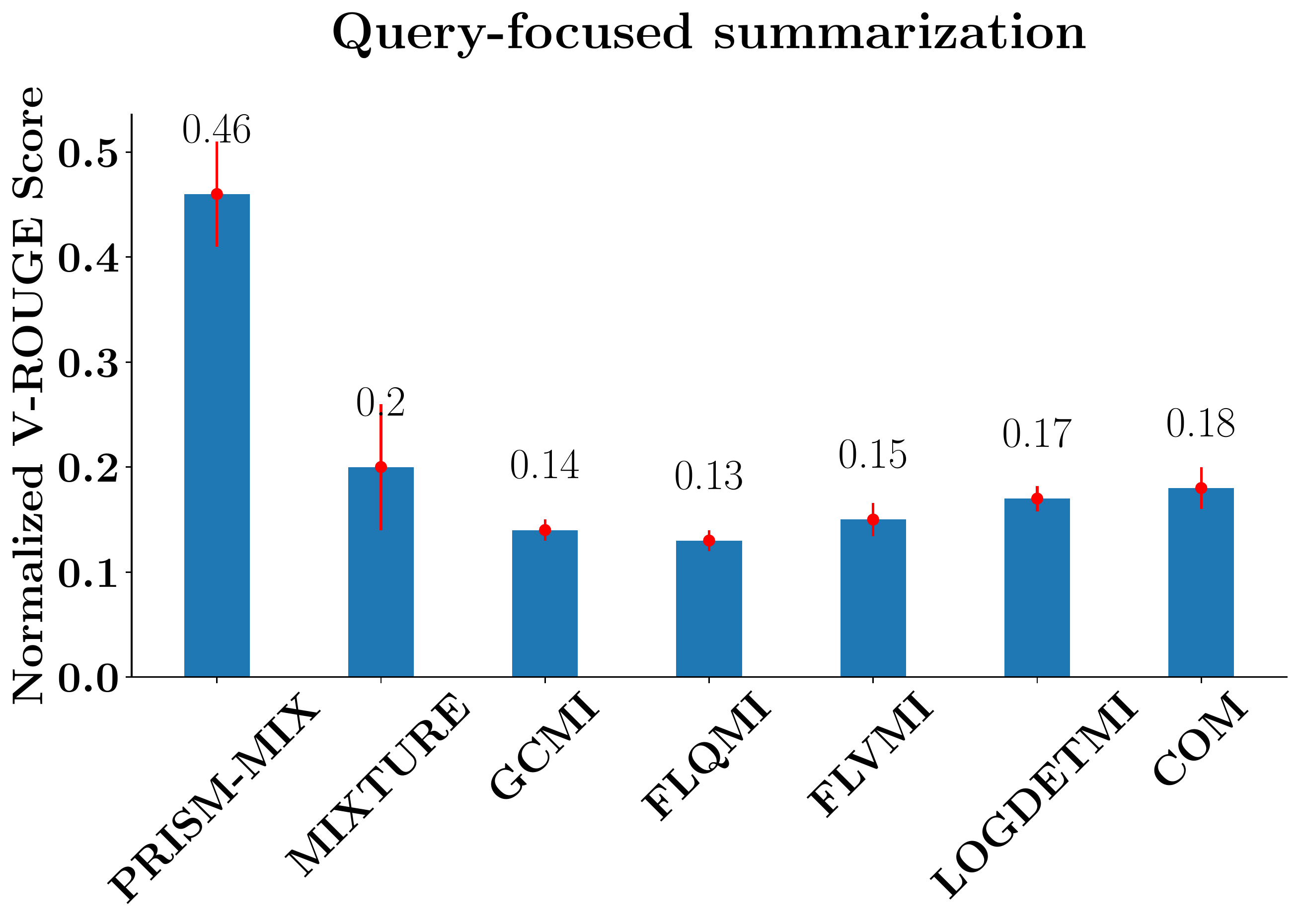}
%   \caption{}
  \label{fig:query-summ}
\end{subfigure}
\begin{subfigure}{0.3\textwidth}
  \includegraphics[width=\linewidth]{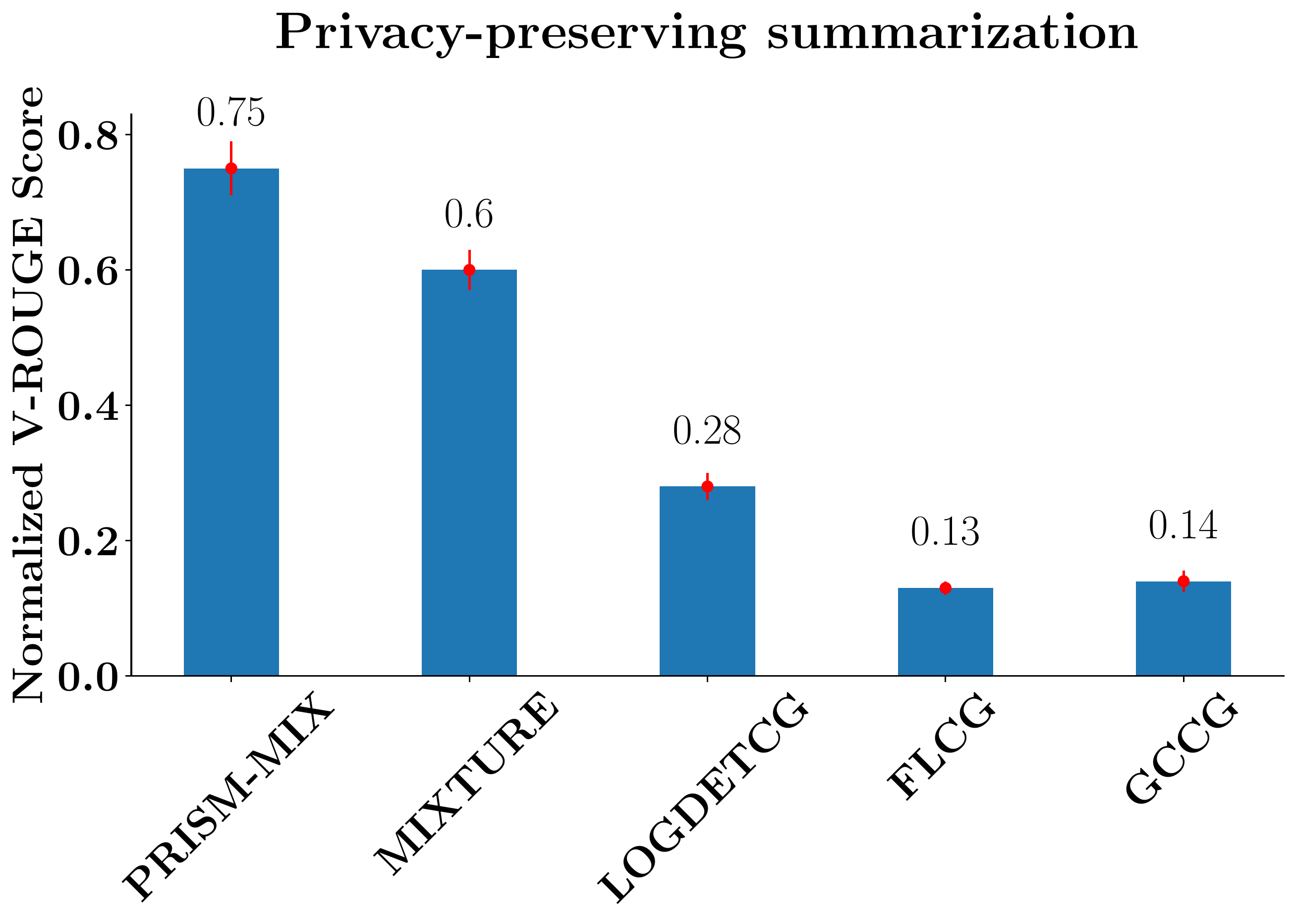}
%   \caption{}
  \label{fig:privacy-summ}
\end{subfigure}\hfil % <-- added
\begin{subfigure}{0.3\textwidth}
  \includegraphics[width=\linewidth]{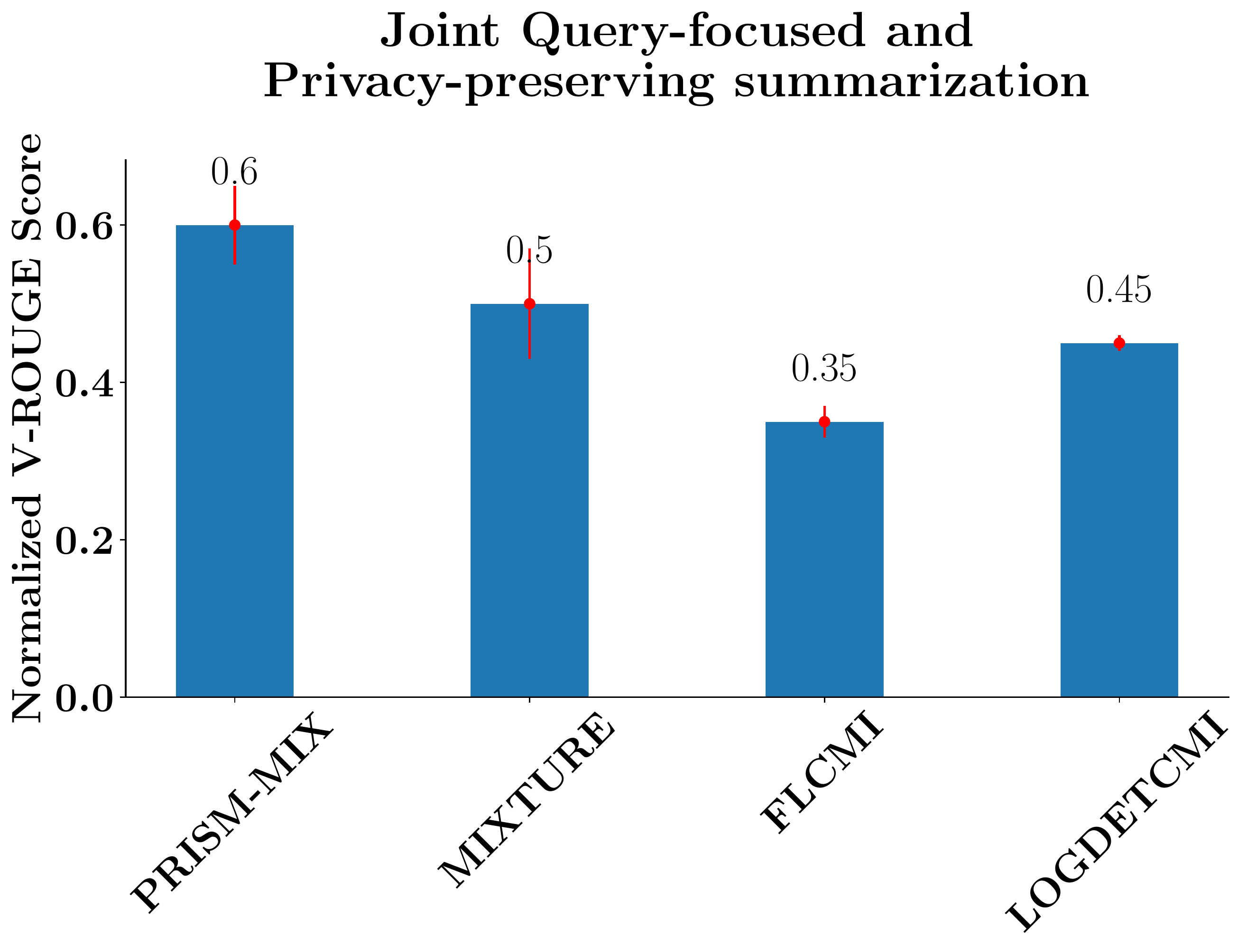}
%   \caption{}
  \label{fig:joint-summ}
\end{subfigure}\hfil % <-- added
\caption{Guided summarization results on a real-world image collections dataset: because of the joint learning of the parameters, the proposed model (\textsc{Prism-Mix}) outperforms others in all flavors of summarization.}
\vspace{-2ex}
\label{fig:image_results}
\end{figure*}

\vspace{-1ex}
\noindent\textbf{Baselines and Implementation details:} We compare use of MI and CMI functions in \AlgRef{algo:tss} with other existing approaches. Specifically, for MI functions we use \textsc{Logdetmi}, \textsc{Gcmi}, \textsc{Flvmi} and, \textsc{Flqmi}.
% and \textsc{Gcmi} + Diversity (equivalent to an intuitive approach of minimizing average gradient difference with the target, see %Eq.~\ref{eqn:tdss-grad-diff} and 
% Lemma~\ref{lemma:grad-error}). %In each case, the similarity kernel matrix is computed using gradients of the unlabaled dataset (using hypothesized labels) and the target. 
As baselines, we consider acquisition functions from the active learning literature; {\em viz.}, entropy sampling (\textsc{Entropy}), \textsc{Badge} \cite{ash2020deep}, and \textsc{Glister-Active} \cite{killamsetty2020glister}. We run the active learning baselines only for one iteration to be consistent with our targeted learning setting ({\em i.e.}, we select from the unlabeled set only once). Since these active learning baselines do not explicitly have information of the target set, to further strengthen the comparison, we also compare against two variants that are target-aware. The first is `targeted entropy sampling' (\textsc{Entropy-Tss}), where a product of the uncertainty and the similarity with the target is used to identify the subset, and the second is \textsc{Glister} for targeted subset selection (\textsc{Glister-Tss}), %(Lemma~\ref{lemma:glister-tss-spcase}) 
where the target set is used in the bi-level optimization. We also compare against \textsc{Grad-Match} \cite{killamsetty2021grad}, which mines for a subset such that the weighted difference in the gradients with the target set is minimized. 
% For these approaches, we use the code published by their authors~\cite{distil-online}. 
Lastly, we also include vanilla FL and random sampling as baselines. 
% Finally, we also compare with pure diversity/representation functions (FL, GC, LogDet, Disparity-Sum (DSUM)) and random sampling. 
For all datasets except MNIST, we train a ResNet-18 model~\cite{he2016deep}. For MNIST, we train a LeNet model~\cite{lecun1989backpropagation}. We use the cross-entropy loss and the SGD optimizer until training accuracy exceeds 99\%. After augmenting the train set with the labeled version of the selected subset and re-training the model, we report the average gain in accuracy for the target classes and the overall gain in accuracy across all classes. The numbers reported are averaged over 10 runs of randomly picking any two classes for the target. We run \AlgRef{algo:tss} for different budgets and also study the effect of budget on the performance. We set the internal parameters to default values of 1. All experiments were run on an NVIDIA RTX 2080Ti GPU. \looseness-1

\noindent \textbf{Results: } 
% In Table 1, % Table~\ref{tab:cifar-mnist-results}, 
% we report the results for a budget of 400 for CIFAR-10 and 70 for MNIST.
% To match the real world setting, we set the target set to be much smaller than the budget; {\em viz.}, around 10\% of the budget.
% % amounting to 44 for CIFAR-10 and 6 for MNIST.
We report the effect of budget on the average gain in accuracy of the target classes in Fig.~\ref{fig:TL}(a-d). On all datasets, MI functions yield the best improvement in terms of accuracy on the target classes, {\em viz.}, $\approx$ 20-30\% gain over the model's performance before re-training with the added targeted subset. While this gain is $\approx$ 12\% higher than that of other methods, this also simultaneously improves the overall accuracy by $\approx$ 2-10\% over other methods. Owing to their richer modeling, the MI functions consistently outperform all baselines across all budgets. This is also evident by the fact that MI functions select the most number of data points from the targeted classes (see \figref{fig:TL}(e-f)). Further, recall the discussion on the behavior of different MI functions in Section~\ref{sec:model}. As expected, \textsc{Flvmi}, \textsc{Flqmi} and \textsc{Logdetmi} functions that model both query-relevance and diversity, perform better than a) functions which tend to prefer relevance ({\em viz.}, \textsc{Gcmi}, \textsc{Entropy-Tss}) and b) functions which tend to prefer diversity/representation ({\em viz.}, \textsc{Badge} and FL).
% , GC, DSUM, LogDet). 
%\todo{Should this not point have come out earlier in the paper, even if somewhat abstractly? It was an insight of this form that I had mentioned for the contribution}
Also, we observe that as the budget is increased, the MI functions outperform other methods by greater margins on the target class accuracy (Fig.~\ref{fig:TL}). We run targeted learning for higher budgets on all datasets, and we observe that the MI functions achieve $20 \times$ to $50 \times$ labeling efficiency in obtaining the same accuracy on rare classes when compared to random and $2\times$ to $4\times$ compared to the best performing baseline (see Appendix~\ref{app:tss-exp} for more details). Additionally, we perform an ablation study to compare the performance of MI functions with the CMI functions and observe that they are at par with each other (see \figref{fig:TL}(g-i)). Finally, we do a pairwise t-test to compare the performance of all methods and observe that the MI functions (particularly \textsc{Flvmi} and \textsc{Flqmi}) statistically significantly outperform all baselines (see Appendix~\ref{app:tss-exp}). From a computational perspective, \textsc{Flqmi} and \textsc{Gcmi} are the fastest in terms of running time and scalability and hence \textsc{Flqmi} is the preferred MI function given its scalability and consistent performance.\looseness-1
% This is expected, as the other methods are not effective in factoring in the target.
%We provide more details on the experimental setup and further discussion on these results in Appendix~\ref{app:tss-exp}.

% \begin{figure*}
% \centering
% \includegraphics[width = 14cm, height=1cm]{plots/cifar10_classimb_legend.pdf}
% \centering
% % \begin{center}
% \begin{subfigure}[]{0.33\textwidth}
% \includegraphics[width = \textwidth]{plots/svhn_scmi_smi.pdf}
% \end{subfigure} 
% \begin{subfigure}[]{0.33\textwidth}
% \includegraphics[width = \textwidth]{plots/mnist_smi_scmi.pdf}
% \end{subfigure}
% \begin{subfigure}[]{0.33\textwidth}
% \includegraphics[width = \textwidth]{plots/cifar10_scmi_smi.pdf}
% \end{subfigure}
% % \end{center}
% \caption{Ablation study.
% \vspace{-3ex}}
% \label{fig:TL}
% \end{figure*}

\subsection{Guided Summarization}\label{sec:exp-tsum}

\noindent \textbf{Dataset and Implementation Details: } We use the image-collections dataset of~\cite{tschiatschek2014learning}. The dataset has 14 image collections with 100 images each and provides ~50-250 human summaries per collection. We extend it by acquiring dense noun concept annotations (objects and scenes) for every image by pseudo-labelling using pre-trained off-the-shelf networks (Yolov3 pre-trained on OpenImagesv6 and ResNet50 pre-trained on Place365) followed by human correction. We designed query and private sets in a spirit similar to~\cite{sharghi2017query} and acquired query-focused, privacy-preserving, and joint query-focused and privacy-preserving human summaries for every image collection. To instantiate the mixture model components, we extract concepts from images using the aforementioned pre-trained off-the-shelf networks and represent them, as well as the concept queries, by a $|\Ccal|$-dimensional vector, where $\Ccal$ is the universe of concepts. %While more complex queries and methods of learning joint embedding between text and images could be employed, we chose simpler alternatives to stick to the main focus area of this work. 
Our mixture model (\textsc{Prism-Mix}) has four components which are the appropriate instantiations (MI/CG/CMI) of functions - GC, LogDet, FL and COM. The mixture weights as well as the internal parameters ($\lambda, \nu, \eta$) are learned using the {\em train} set. Following~\cite{tschiatschek2014learning}, we perform leave-one-out cross validation and report average V-ROUGE across 14 runs. We also normalize V-ROUGE {\em s.t.} the human average is 1 and the random average is 0. %All experiments were run on an NVIDIA RTX 2080Ti GPU. 
We provide further details of the dataset and implementation in Appendix~\ref{app:real-sum}. \looseness-1

\noindent \textbf{Results: } We present the guided summarization results in Fig.~\ref{fig:image_results}. As discussed in Section~\ref{sec:connections}, the individual components of our mixture model have been used as models in previous works on document and video summarization. Hence, we compare the performance of \textsc{Prism-Mix} against the performance of the individual components as our baselines. 
% Thus, in the absence of any explicit past work on targeted summarization for image collections, we compare the performance of \textsc{Prism-Mix} against the performance of the individual components. 
Also, to confirm the positive effect of jointly learning the parameters of \model{} along with the mixture weights, we compare \textsc{Prism-Mix} against a mixture model (\textsc{Mixture}) of exactly the same components, but with only the model weights $w$ being learned. Other internal parameters ($\lambda, \eta, \nu$)
%internal parameters 
are set to fixed values of 1. We observe that \textsc{Prism-Mix} outperforms other techniques, including \textsc{Mixture} on all flavors of summarization (see \figref{fig:image_results}). This is expected, as the joint learning of parameters offered by \model{} (\secref{sec:tsum}) enables producing summaries that can better imitate the complexities of the ground-truth summaries. \looseness-1 %, hence confirming the effectiveness of proposed learning framework based on \model{}. %, especially of the joint learning of the parameters. 

\vspace{-1ex}
\section{Conclusion}
\label{sec:conclusion}
We presented \model{}, %, a novel framework using parameterized submodular information measures for targeted subset selection. The 
a rich class of functions for guided subset selection. \model{} allows to model a broad spectrum of semantics across query-relevance, diversity, query-coverage and privacy-irrelevance. We demonstrated its effectiveness in targeted learning as well as in guided summarization. We showed that \model{} has interesting connections to several past work, further reinforcing its utility. Through several experiments on targeted learning and guided summarization for diverse datasets, we empirically verified the superiority of \model{} over existing methods.
% A potential limitation of \model{} is that it assumes a) availability of an explicit target set; b) its representation in the ground set; and c) the availability of appropriate features of data points and target so as to be able to compute the similarity kernels required by the different functions. 

% \section*{References}

\section*{Acknowledgments and Disclosure of Funding}
We thank anonymous reviewers and Nathan Beck for providing constructive feedback. This work is supported by the National Science Foundation under Grant No. IIS-2106937, a startup grant from UT Dallas, and by a Google and Adobe research award. Any opinions, findings, and conclusions or recommendations expressed in this material are those of the authors and do not necessarily reflect the views of the National Science Foundation, Google or Adobe.

\bibliography{main.bib}

\newpage
\clearpage
%%%%%%%%%%%%%%%%%%%%%%%%%%%%%%%%%%%%%%%%%%%%%%%%%%%%%%%%%%%%

%%%%%%%%%%%%%%%%%%%%%%%%%%%%%%%%%%%%%%%%%%%%%%%%%%%%%%%%%%%%

\onecolumn
\appendix

\setcounter{page}{1}

\begin{center}
    \Huge{Appendix}
\end{center}
% Start the appendix part
\parttoc % Insert the appendix TOC

\section{Summary of Notations}\label{app:notation-summary}

We present a summary of notations used throughout this paper in Table~\ref{tab:main-notations}

%BIG TABLE STARTS HERE
 \begin{table*}[!h]
 \centering
 \arrayrulecolor[rgb]{0.192,0.192,0.192}
 %\begin{adjustbox}{max width=0.48\textwidth}
 \begin{tabular}{|l|l|p{0.5\textwidth}|} 
 \toprule
 %\hline 
 %\hline 
 \multicolumn{1}{|l|}{Topic} & Notation & Explanation \\ \hline 
 \toprule
 &  $\Vcal$ & Ground set of $n$ instances\\ 
 \multicolumn{1}{|p{0.20\textwidth}|}{\model{} }
 & $\Vcal^{\prime}$ & Auxiliary set containing private set or query set\\ 
 & $\Omega$ & $\Vcal \cup \Vcal^{\prime}$\\ 
 & $\Acal$ & A subset of $\Vcal$\\ 
 & $S_{\Acal, \Bcal}$ & Cross-similarity matrix between the items in sets $\Acal$ and $\Bcal$\\ 
 & $S_{\Acal\Bcal}$ & Similarity matrix for items in $\Acal \cup \Bcal$\\ 
 & $\lambda$ & Parameter governing trade-off between representation and diversity in GC\\ 
 & $\eta$ & Parameter governing trade-off between query-relevance and diversity in MI and CMI functions\\ 
 & $\nu$ & Parameter governing hardness of privacy constraints in CG and CMI functions\\ 
 \hline\hline
 & $\Ecal$ & Initial set of $|\Ecal|$ labeled instances\\ 
 \multicolumn{1}{|p{0.20\textwidth}|}{\model{} for Task 1: Improving a model's accuracy (\AlgRef{algo:tss})}
  
 & $\Ucal$ & Set of $|\Ucal|$ instances in unlabeled data set \\ 
 &  $\Tcal$ & Set of $|\Tcal|$ instances in the target/query set \\ 
 &  $\gamma g(\Acal)$ & Diversity function that can be added to MI function in Algorithm~\ref{algo:tss} with weight $\gamma$ \\ 
%  & $\Lcal(.,\theta)$ & Loss function for the machine learning model with associated parameter vector $\theta$ \\
 \hline\hline
 & $\Pcal$ (as $\Tcal$) & Private set or conditioning set for targeted summarization\\ 
 \multicolumn{1}{|p{0.20\textwidth}|} {\model{} for Task 2: Targeted summarization}
 & $\Qcal$ (as $\Tcal$) & Query set for targeted summarization\\ 
 & $\Fcal(\Theta)$ & Mixture model using \model{} with parameters $\Theta$\\
 & $\Lcal_n(\Theta)$ & Generalized hinge loss of training example $n$, with parameter $\Theta = (\wb, \eta, \lambda, \nu)$\\
 
%  \multicolumn{1}{|c|}{Loss Functions}
% & $\Lcal(.,\theta)$ & Loss function for the \modelmachine learning model with associated parameter vector $\theta$ \\
% & $\Lcal_n(\Theta)$ & Generalized hinge loss of training example $n$, with parameter $\Theta = (\wb, \eta, \lambda, \nu)$, This used in \modelsum\\
% & $L$ & Generic reference to the loss function which when evaluated on $x_i$ is referred to as $L^i$.   \\ 
 \bottomrule
 \end{tabular}
 %\end{adjustbox}
 \arrayrulecolor{black}
 \caption{Summary of notations used throughout this paper}
 \label{tab:main-notations}
 \end{table*}

\section{Proofs for Generalized Submodular Mutual Information (GMI) Functions - Restricted Submodularity}

\subsection{Properties of Generalized Submodular Mutual Information Functions}
\label{app:gsmi}

\begin{lemma} 
Given a restricted submodular function $f$ on $\mathcal C(\Vcal, \Vcal^{\prime})$, $I_f(\Acal; \Bcal) \geq 0$ for $\Acal \subseteq \Vcal, \Bcal \subseteq \Vcal^{\prime}$. Also, $I_f(\Acal; \Bcal)$ is monotone in $\Acal \subseteq \Vcal$ for fixed $\Bcal \subseteq \Vcal^{\prime}$ (equivalently, $I_f(\Acal; \Bcal)$ is monotone in $\Bcal \subseteq \Vcal^{\prime}$ for fixed $\Acal \subseteq \Vcal$).
\end{lemma}

\begin{proof}
The non-negativity of the generalized submodular mutual information follows from the definition. In particular, since $\Acal \subseteq \Vcal$ and $\Bcal \subseteq \Vcal^{\prime}$, it holds that $$I_f(\Acal; \Bcal) = f(\Acal) + f(\Bcal) - f(\Acal \cup \Bcal) \geq 0$$ because $f$ is restricted submodular on $C(\Vcal, \Vcal^{\prime})$. Next, we prove the monotonicity. We have, $$I_f(\Acal \cup j; \Bcal) - I_f(\Acal; \Bcal) = [f(j | \Acal) - f(j | \Acal \cup \Bcal)], \forall j \in \Vcal \backslash A$$ Given that $f$ is restricted submodular on $\mathcal C(\Vcal, \Vcal^{\prime})$, it holds that $$[f(j | \Acal) - f(j | \Acal \cup \Bcal)] \geq 0$$ since $$f(\Acal \cup \{j\}) + f(\Acal \cup \Bcal)\geq f(\Acal) + f(\Acal \cup \Bcal \cup j)$$ which follows since the submodularity inequality holds as long as one of the sets is a subset of either $\Vcal$ or $\Vcal^{\prime}$ (i.e. both subsets do not non-empty intersection with $\Vcal$ and $\Vcal^{\prime}$. Thus $$I_f(\Acal \cup \{j\}; \Bcal) - I_f(\Acal; \Bcal) \geq 0$$ and hence monotone. %The non-negativity follows from the monotonicity of GMI and the fact that $I_f(\emptyset, \Bcal) = 0$.
\end{proof}

\subsection{Concave Over Modular as GMI}
\label{app:com}
%Define the following query based submodular function:
% \begin{align}
%   F_{\eta}(\Acal; \Qcal) = \eta \sum_{i \in \Acal} \psi(\sum_{j \in \Qcal}S_{ij}) + \sum_{j \in \Qcal} \psi(\sum_{i \in \Acal} S_{ij})
% \end{align}
% This is a very general class of functions ~\cite{bilmes2017deep} and does very well in query-focused extractive document summarization~\cite{lin2011class,lin2012submodularity} with $\eta = 0$ and the concave function as square root. Next, define the following restricted submodular function:
% \begin{align}
%     f_{\eta}(S) = \eta \sum_{i \in \Vcal^{\prime}} \max(\psi(\sum_{j \in S \cap \Vcal} S_{ij}), \psi(\sqrt{n}\sum_{j \in S \cap \Vcal^{\prime}} S_{ij})) +  \sum_{i \in \Vcal} \max(\psi(\sum_{j \in S \cap \Vcal^{\prime}} S_{ij}), \psi(\sqrt{n}\sum_{j \in S \cap \Vcal} S_{ij}))
% \end{align}
% The following result connects GSMI with $f_{\eta}$ with $F_{\eta}(\Acal; \Qcal)$.
% \begin{lemma}
% The function $f_{\eta}(S)$ is a restricted submodular function on $\mathcal C(\Vcal, \Vcal^{\prime})$. Furthermore the GSMI with $f_{\eta}$ is exactly $F_{\eta}(\Acal; \Qcal)$, given a kernel matrix which satisfies $S_{ij} = 1(i == j)$ for $i, j \in \Vcal$ or $i, j \in \Vcal^{\prime}$.
% \end{lemma}
\begin{numlem} 
The function $f_{\eta}(\Acal)$ is a restricted submodular function on $\Ccal(\Vcal, \Vcal^{\prime})$. Furthermore, the GMI with $f_{\eta}$ is exactly $I_{f_{\eta}}(\Acal; \Qcal) = \eta \sum_{i \in \Acal} \psi(\sum_{j \in \Qcal}S_{ij}) + \sum_{j \in \Qcal} \psi(\sum_{i \in \Acal} S_{ij})$, given a kernel matrix which satisfies $S_{ij} = 1(i == j)$ for $i, j \in \Vcal$ or $i, j \in \Vcal^{\prime}$. The CG and CMI expressions are not particularly useful in this case.
\end{numlem}
\begin{proof}
Assume that the kernel matrix $S_{ij} \leq 1, \forall i, j \in \Omega$. Also, we are given that $S_{ij} = 1(i == j)$ for $i, j \in \Vcal$ or $i, j \in \Vcal^{\prime}$. Next, notice that:
\begin{align}
    f(\Acal) = \eta \sum_{i \in \Vcal^{\prime}} \psi(\sum_{j \in \Acal} S_{ij}) + \sum_{i \in \Acal} \psi(\sqrt{n})
\end{align}
This holds because the $S$ kernel is an identity kernel within $\Vcal$ and $\Vcal^{\prime}$ and only has terms in the cross between the two sets. Similarly, \begin{align}
    f(\Qcal) =  \sum_{i \in \Vcal} \psi(\sum_{j \in \Qcal} S_{ij}) + \eta \sum_{i \in \Qcal} \psi(\sqrt{n})
\end{align}
Finally, we obtain $f(\Acal \cup \Qcal)$:
\begin{align}
    f(\Acal \cup \Qcal) =  \eta \sum_{i \in \Vcal^{\prime} \setminus \Qcal} \psi(\sum_{j \in \Acal} S_{ij}) + \eta \sum_{i \in \Qcal} \psi(\sqrt{n}) + \sum_{i \in \Vcal \setminus \Acal} \psi(\sum_{j \in \Qcal} S_{ij}) + \sum_{i \in \Acal} \psi(\sqrt{n})
\end{align}
Combining all the three terms together, we obtain $f(\Acal) + f(\Qcal) - f(\Acal \cup \Qcal) = \eta \sum_{i \in \Acal} \psi(\sum_{j \in \Qcal}S_{ij}) + \sum_{j \in \Qcal} \psi(\sum_{i \in \Acal} S_{ij}) = F_{\eta}(\Acal; \Qcal)$.

Finally, to show that $f_{\eta}(S)$ is restricted submodular, notice that $f(\Acal)$ is submodular if $\Acal$ is restricted to either $\Vcal$ or $\Vcal^{\prime}$. Similarly, given sets $\Acal \subseteq \Vcal, \Bcal \subseteq \Vcal^{\prime}$, it holds that $f(\Acal) + f(\Bcal) - f(\Acal \cup \Bcal) = I_f(\Acal; \Bcal) \geq 0$ which implies the restricted submodularity of $f_{\eta}(S)$.

Like FL (v2), the expressions for CG and CMI don't make sense in COM since they require computing terms over $\Vcal^{\prime}$, which we do not have access to.
\end{proof}

\section{Proofs for instantiations used in \model{}}

\subsection{Log-Determinant Based Information Measures: \textsc{Logdetmi}, \textsc{Logdetcg} and \textsc{Logdetcmi}}
\label{app:logdet}
% \begin{numlem}
% Denote $S_{\Acal\Qcal}$ as the pairwise similarity matrix between the items in $\Acal$ and $\Qcal$. Then, $I_f(\Acal; \Qcal) = -\log \det(I - S_{\Acal}^{-1}S_{\Acal\Qcal}S_{\Qcal}^{-1}S_{\Acal\Qcal}^T) = \log\det(S_{\Acal}) - \log\det(S_{\Acal} - S_{\Acal\Qcal}S_{\Qcal}^{-1}S_{\Acal\Qcal}^T)$. The Conditional Gain takes the form $f(\Acal | \Pcal ) = \log\det(S_{\Acal} - S_{\Acal\Pcal}S_{\Pcal}^{-1}S_{\Acal\Pcal}^T)$ and the Conditional Submodular Mutual Information can be written as $I_f(\Acal; \Qcal |\Pcal ) = \log \frac{\det(I - S_{\Pcal}^{-1} S_{\Pcal, \Qcal} S_{\Qcal}^{-1} S_{\Pcal, Q}^T)}{\det(I - S_{\Acal \cup \Pcal}^{-1} S_{\Acal \cup \Pcal, \Qcal} S_{\Qcal}^{-1} S_{\Acal \cup \Pcal, \Qcal}^T)}$
% \end{numlem}
\begin{lemma}
Setting $f(\Acal) = \log\det(S_{\Acal})$, we have: $I_f(\Acal; \Qcal) = \log\det(S_{\Acal}) - \log\det(S_{\Acal} - \eta^2 S_{\Acal, \Qcal}S_{\Qcal}^{-1}S_{\Acal, \Qcal}^T)$, and $f(\Acal | \Pcal ) = \log\det(S_{\Acal} - \nu^2 S_{\Acal, \Pcal}S_{\Pcal}^{-1}S_{\Acal, \Pcal}^T)$. Similarly, $I_f(\Acal; \Qcal |\Pcal ) = \log \frac{\det(I - S_{\Pcal}^{-1} S_{\Pcal, \Qcal} S_{\Qcal}^{-1} S_{\Pcal, \Qcal}^T)}{\det(I - S_{\Acal \Pcal}^{-1} S_{\Acal \Pcal, \Qcal} S_{\Qcal}^{-1} S_{\Acal \Pcal, \Qcal}^T)}$
\end{lemma}
\begin{proof}
Given a positive semi-definite matrix $S$, the Log-Determinant Function is $f(\Acal) = \log \det(S_{\Acal})$ where $S_{\Acal}$ is a sub-matrix comprising of the rows and columns indexed by $\Acal$. The following expressions follow directly from the definitions. The MI is: $I_f(\Acal; \Qcal) = \log \frac{\det(S_{\Acal}) \det(S_{\Qcal})}{\det(S_{\Acal \cup \Qcal})}$, CG is: $f(\Acal | \Pcal ) = \log \frac{\det(S_{\Acal \cup \Pcal})}{\det(S_{\Pcal})}$ and CMI is $I_f(\Acal; \Qcal |\Pcal ) = \log \frac{\det(S_{\Acal \cup \Pcal}) \det(S_{Q \cup \Pcal})}{\det(S_{\Acal \cup \Qcal \cup \Pcal}) \det(S_{\Pcal})}$.

Next, note that using the Schur's complement, $\det(S_{\Acal \cup \Bcal}) = \det(S_{\Acal}) \det(S_{\Acal \cup \Bcal} \backslash S_{\Acal})$ where,
$$S_{\Acal \cup \Bcal} \backslash S_{\Acal} = S_{\Bcal} - S_{\Acal \Bcal}^T S_{\Acal}^{-1} S_{\Acal \Bcal}$$
where $S_{\Acal \Bcal}$ is a $|\Acal| \times |\Bcal|$ matrix and includes the cross similarities between the items in sets $\Acal$ and $\Bcal$. Similarly, 
$$S_{\Acal \cup \Bcal} \backslash S_{\Bcal} = S_{\Acal} - S_{\Acal \Bcal} S_{\Bcal}^{-1} S_{\Acal \Bcal}^T$$
As a result, the Mutual Information becomes:
\begin{align*}
    I_f(\Acal; \Qcal) &= -\log [\det(S_{\Acal} - S_{\Acal \Qcal} S_{\Qcal}^{-1} S_{\Acal \Qcal}^T) \det(S_{\Acal}^{-1})] \\
            &= -\log \det(I - S_{\Acal \Qcal} S_{\Qcal}^{-1} S_{\Acal \Qcal}^T S_{\Acal}^{-1}) \\
            &= -\log \det(I - S_{\Acal}^{-1} S_{\Acal \Qcal} S_{\Qcal}^{-1} S_{\Acal \Qcal}^T )
\end{align*}
For the CG,  
\begin{align*}
    f(\Acal | \Pcal ) &= f(\Acal) - I_f(\Acal; \Pcal) \\
    &= \log \det(S_{\Acal}) - \log\det(S_{\Acal}) + \log \det(S_{\Acal} - S_{\Acal \Qcal} S_{\Qcal}^{-1} S_{\Acal \Qcal}^T) \\
    &= \log \det(S_{\Acal} - S_{\Acal \Qcal} S_{\Qcal}^{-1} S_{\Acal \Qcal}^T)
\end{align*}

Similarly, the proof of the conditional submodular mutual information follows from the simple observation that:
$$I_f(\Acal; \Qcal |\Pcal ) = I_f(\Acal \cup \Pcal; \Qcal) - I_f(\Qcal; \Pcal)$$
Plugging in the expressions of the mutual information of the log-determinant function, from above, we have,
\begin{align*}
    I_f(\Acal \cup \Pcal; \Qcal) &= -\log \det(I - S_{\Acal \cup \Pcal}^{-1} S_{\Acal \cup \Pcal Q} S_{\Qcal}^{-1} S_{\Acal \cup \Pcal Q}^T ) \\
    I_f(\Qcal; \Pcal) &= -\log \det(I - S_{\Qcal}^{-1} S_{\Qcal \Pcal} S_{\Pcal}^{-1} S_{\Qcal \Pcal}^T ) \\
    \therefore I_f(\Acal; \Qcal |\Pcal ) &= \log \det(I - S_{\Pcal}^{-1} S_{\Pcal, \Qcal} S_{\Qcal}^{-1} S_{\Pcal, \Qcal}^T) - \log \det(I - S_{\Acal \cup \Pcal}^{-1} S_{\Acal \cup \Pcal, Q} S_{\Qcal}^{-1} S_{\Acal \cup \Pcal, Q}^T) \\
    &= \log \frac{\det(I - S_{\Pcal}^{-1} S_{\Pcal, \Qcal} S_{\Qcal}^{-1} S_{\Pcal, \Qcal}^T)}{\det(I - S_{\Acal \cup \Pcal}^{-1} S_{\Acal \cup \Pcal, \Qcal} S_{\Qcal}^{-1} S_{\Acal \cup \Pcal, \Qcal}^T)}
\end{align*}
\end{proof}

The proof for CMI implicitly assumes $\eta = \nu = 1$. A simple way to solve this, is as follows. Denote $S_{\Acal \Pcal \nu}$ as the similarity matrix obtained by multiplying $\nu$ to the cross similarity entries. Similarly, denote $S_{\Acal \Pcal\nu, \Qcal\eta}$ as the cross similarity obtained by multiplying $\nu$ to the cross similarity between $\Acal$ and $\Pcal$ and $\eta$ to the cross similarity between $\Acal$ and $\Qcal$. The CMI function with this choice of a similarity matrix is:
\begin{align}
    I_f(\Acal; \Qcal |\Pcal ) = \log \frac{\det(I - S_{\Pcal}^{-1} S_{\Pcal, \Qcal} S_{\Qcal}^{-1} S_{\Pcal, \Qcal}^T)}{\det(I - S_{\Acal \Pcal \nu}^{-1} S_{\Acal \Pcal\nu, \Qcal \eta} S_{\Qcal}^{-1} S_{\Acal \Pcal \nu, \Qcal \eta}^T)}
\end{align}

\subsection{Facility Location based Information Measures: \textsc{Flvmi}, \textsc{Flqmi}, \textsc{Flcg}}
\label{app:fl}
\begin{numthm}
\label{FLMIGen}
Given a similarity kernel $S$, a set $\Ucal \subseteq \Omega$ and the facility location (FL) function $f(\Acal) = \sum_{i \in \Ucal} \max_{j \in \Acal} S_{ij}, \Acal \subseteq \Omega$ the Mutual Information for FL can be written as $I_f(\Acal;\Qcal)=\sum_{i \in \Ucal}\min(\max_{j \in \Acal}S_{ij}, \max_{j \in \Qcal}S_{ij})$.  Similarly, the CG for facility location can be written as $f(\Acal|P)= \sum_{i \in \Ucal} \max(\max_{j \in \Acal} S_{ij} - \max_{j \in P} S_{ij}, 0)$ and the expression for Conditional Submodular Mutual Information can be written as: $I_f(\Acal; \Qcal |\Pcal ) = \sum_{i \in \Ucal} \max(\min(\max_{j \in \Acal} S_{ij}, \max_{j \in \Qcal} S_{ij}) - \max_{j \in P} S_{ij}, 0)$.
\end{numthm}
\begin{proof}
Here we have the facility location set function, $f(\Acal) = \sum_{i \in \Ucal} \max_{j \in \Acal} S_{ij}$ where $s$ is similarity kernel and $U \subseteq \Omega$. Then,
\begin{align*}
I_f(\Acal;\Qcal) &= f(\Acal) + f(\Qcal) - f(\Acal \cup \Qcal) \\ 
&= \sum_{i \in \Ucal} \max_{j \in \Acal} S_{ij} + \max_{j \in \Qcal} S_{ij} - \max_{j \in \Acal \cup \Qcal} S_{ij} \\
&= \sum_{i \in \Ucal} \max_{j \in \Acal} S_{ij} + \max_{j \in \Qcal} S_{ij} - \max(\max_{j \in \Acal} S_{ij}, \max_{j \in \Qcal} S_{ij}) \\
&= \sum_{i \in \Ucal} \min(\max_{j \in \Acal} S_{ij}, \max_{j \in \Qcal} S_{ij})
\end{align*}
For the Conditional Gain we have
\begin{align*}
    f(\Acal|\Pcal) &= \sum_{i \in \Ucal} \max(\max_{j \in \Acal} S_{ij}, \max_{j \in \Pcal} S_{ij}) -  \max_{j \in \Pcal} S_{ij} \\  
 &= \sum_{i \in \Ucal} \max(0, \max_{j \in \Acal} S_{ij} - \max_{j \in \Pcal} S_{ij})
\end{align*}
Finally, we can get the expression for $I_f(\Acal; \Qcal |\Pcal )$ as
\begin{align*}
    I_f(\Acal; \Qcal |\Pcal ) &= f(\Acal|\Pcal) + f(\Qcal|\Pcal) - f(\Acal \cup \Qcal | \Pcal ) \\
    &= \sum_{i in U}[\max(\max_{j in \Acal}S_{ij} - \max_{j in \Pcal}S_{ij}, 0) + \max(\max_{j \in \Qcal}S_{ij} - \max_{j \in \Pcal}S_{ij}, 0) - \max(\max_{j \in \Acal \cup \Qcal}S_{ij} - \max_{j \in P}S_{ij}, 0)] \\
    &= \sum_{i \in \Ucal} \max(\min(\max_{j \in \Acal} S_{ij}, \max_{j \in \Qcal} S_{ij}) - \max_{j \in P} S_{ij}, 0)
\end{align*}
The last step follows from the observation that in $\max(a-c, 0) + \max(b-c, 0) - \max(\max(a,b)-c,0)$ the last term is either the first term or the second term (hence cancelling that out) depending on whether $a > b$ or not.
\end{proof}
Now, we can obtain the expression of \textsc{Flvmi}, \textsc{Flqmi}, \textsc{Flcg} and \textsc{Flcmi} as special cases.
\begin{numcor}
Setting $U = V$ in the expression of MI, CG and CMI in Theorem~\ref{FLMIGen}, we obtain the expression for \textsc{Flvmi} as $I_f(\Acal;\Qcal)=\sum_{i \in \Vcal}\min(\max_{j \in \Acal}S_{ij}, \max_{j \in \Qcal}S_{ij})$, \textsc{Flcg} as $f(\Acal|\Pcal)= \sum_{i \in \Vcal} \max(\max_{j \in \Acal} S_{ij} - \max_{j \in \Pcal} S_{ij}, 0)$ and \textsc{Flcmi} as $I_f(\Acal; \Qcal |\Pcal ) = \sum_{i \in \Vcal} \max(\min(\max_{j \in \Acal} S_{ij}, \max_{j \in \Qcal} S_{ij}) - \max_{j \in P} S_{ij}, 0)$
\end{numcor}
This corollary follows directly from Theorem~\ref{FLMIGen}. We can similarly, also obtain the expression for \textsc{Flqmi}.
% \begin{numlem}
% With $U = \Omega$, and the similarity matrix $S$ is such that $S_{ij} = I(i == j)$, if both $i, j \in \Vcal$ or both $i, j \in \Vcal^{\prime}$, we obtain the expression for \textsc{Flqmi} as $I_f(\Acal; \Qcal) = \sum_{i \in \Qcal} \max_{j \in \Acal} S_{ij} + \eta \sum_{i \in \Acal} \max_{j \in \Qcal} S_{ij}$
% \end{numlem}
\begin{lemma}
Given a similarity kernel $S$ such that $S_{ij} = I(i == j)$, if both $i, j \in \Vcal$ or both $i, j \in \Vcal^{\prime}$ and the facility location (FL) function $f(\Acal) = \sum_{i \in \Omega} \max_{j \in \Acal} S_{ij}, \Acal \subseteq \Omega$ we obtain the expression for MI (\textsc{Flqmi}) as $I_f(\Acal; \Qcal) = \sum_{i \in \Qcal} \max_{j \in \Acal} S_{ij} + \eta \sum_{i \in \Acal} \max_{j \in \Qcal} S_{ij}$. The CG and CMI expressions are not particularly useful in this case.
\end{lemma}
\begin{proof}
Assuming $s_{ii} = 1$ is the maximum similarity score in the kernel, for the alternative formulation under the assumption that $U = \Omega$, we can break down the sum over elements in ground set $\Omega$ as follows. For any $i \in \Acal, \max_{j \in \Acal} S_{ij} = 1$ and hence the minimum (over sets $\Acal$ and $\Qcal$) will just be the term corresponding to Q (and a similar argument follows for terms in $\Qcal$). Then we have,

\begin{align}
   I_f(\Acal;\Bcal) &= \sum_{i \in \Omega\setminus (\Acal \cup \Qcal)} \min(\max_{j \in \Acal} S_{ij}, \max_{j \in \Qcal} S_{ij}) 
   + \sum_{i \in \Acal \setminus \Qcal} \max_{j \in \Qcal} S_{ij} 
   + \sum_{i \in \Qcal \setminus \Acal} \max_{j \in \Acal} S_{ij}
   + \sum_{i \in \Acal \cap \Qcal} 1  \\
   &= \sum_{i \in \Vcal \setminus \Acal} \min(\max_{j \in \Acal} S_{ij}, \max_{j \in \Qcal} S_{ij}) + \sum_{i \in \Vcal^{\prime} \setminus \Qcal} \min(\max_{j \in \Acal} S_{ij}, \max_{j \in \Qcal} S_{ij}) + \sum_{i \in \Acal} \max_{j \in \Qcal} S_{ij}
   + \sum_{i \in \Qcal} \max_{j \in \Acal} S_{ij}
\end{align}
This follows because $\Acal \cap \Qcal = \emptyset$. Finally, note that $$\sum_{i \in \Vcal \setminus \Acal} \min(\max_{j \in \Acal} S_{ij}, \max_{j \in \Qcal} S_{ij}) + \sum_{i \in \Vcal^{\prime} \setminus \Qcal} \min(\max_{j \in \Acal} S_{ij}, \max_{j \in \Qcal} S_{ij}) = 0$$ since $\forall i \in \Vcal \setminus \Acal, j \in \Acal, S_{ij} = 0$ and similarly $\forall i \in \Vcal^{\prime} \setminus \Qcal, j \in \Qcal, S_{ij} = 0$. This leaves us with $I_f(\Acal; \Qcal) = \sum_{i \in \Qcal} \max_{j \in \Acal} S_{ij} + \eta \sum_{i \in \Acal} \max_{j \in \Qcal} S_{ij}$

The expressions for CG and CMI don't make sense in this version of FL since they require computing terms over $\Vcal^{\prime}$, which we do not have access to.
\end{proof}

\subsection{Expression for GCCMI is not useful}
\label{app:gccsmi}
\begin{lemma}
When $f$ is the graph-cut function, $I_f(\Acal; \Qcal |\Pcal ) = I_f(\Acal; \Qcal)$. In other words, the CMI function does not depend on the private set $\Pcal$.
\end{lemma}
\begin{proof}
For deriving the expression for Conditional Submodular Mutual Information we proceed as follows,

Let $$g(\Acal) = f(\Acal|P) = f(\Acal) - 2\lambda \sum_{i \in \Acal}\sum_{j \in P}S_{ij}$$  
Then, 
\begin{align*}
    I_f(\Acal;\Qcal|P) &= I_g(\Acal;\Qcal) \\
    &= I_f(\Acal;\Qcal) - 2\lambda \sum_{i \in \Acal \cap \Qcal, j \in \Pcal}S_{ij}
\end{align*}

Since $\Acal, \Qcal$ are disjoint, the second term is 0 and the first term doesn't have any effect of $\Pcal$. Thus, the Conditional Submodular Mutual Information for Graph Cut (GCCMI) isn't useful.
\end{proof}

\section{Scalability of \model{}} \label{app:scalability}

Below, we provide a detailed analysis of the complexity of creating and optimizing the different functions in \model{}. Denote $|\Xcal|$ as the size of set $\Xcal$. Also, let $|\Ucal| = n$ (the ground set size, which is the size of the unlabeled set in this case) and $B$ be the selection budget. In the main paper, we provided the high level intuition of the complexity, ignoring the terms of $|\Pcal|$ and $|\Qcal|$ since they would be typically much smaller than the number of unlabeled points $n$. For completeness, we provide the detailed complexity below:

\begin{itemize}%[leftmargin=*]
    \item \textbf{Facility Location: } We start with \textsc{Flvmi}. The complexity of creating the kernel matrix is $O(n^2)$. The complexity of optimizing it is $\tilde{O}(n^2)$ (using memoization~\cite{iyer2019memoization})\footnote{$\tilde{O}$: Ignoring log-factors} if we use the stochastic greedy algorithm~\cite{mirzasoleiman2015lazier} and $O(n^2k)$ with the naive greedy algorithm. The overall complexity is $\tilde{O}(n^2)$.
    For \textsc{Flqmi}, the cost of creating the kernel matrix is $O(n|\Qcal|)$, and the cost of optimization is also $\tilde{O}(n|\Qcal|)$ (with naive greedy, it is $O(nB |\Qcal|)$). The complexity of \textsc{Flcg} is $O([n + |\Pcal|]^2)$ to compute the kernel matrix and $\tilde{O}(n^2)$ for optimizing (using the stochastic greedy algorithm). Finally, for \textsc{Flcmi}, the complexity of computing the kernel matrix is $O([n + |\Qcal| + |\Pcal|]^2)$, and the complexity of optimization is $\tilde{O}(n^2)$.
    \item \textbf{Log-Determinant: } We start with \textsc{Logdetmi}. The complexity of the kernel matrix computation (and storage) is $O(n^2)$. The complexity of optimizing the LogDet function using the stochastic greedy algorithm is $\tilde{O}(B^2 n)$, so the overall complexity is $\tilde{O}(n^2 + B^2n)$. For \textsc{Logdetcg}, the complexity of computing the matrix is $O([n + |\Pcal|]^2$, and the complexity of optimization is $\tilde{O}([B + |\Pcal|]^2 n)$. For the \textsc{Logdetcmi} function, the complexity of computing the matrix is $O([n + |\Pcal| + |\Qcal|]^2$, and the complexity of optimization is $\tilde{O}([B + |\Pcal| + |\Qcal|]^2 n)$. 
    \item \textbf{Graph-Cut: } Finally, we study GC functions. For \textsc{Gcmi}, we require an $O(n|\Qcal|)$ kernel matrix, and the complexity of the stochastic greedy algorithm is also $\tilde{O}(n|\Qcal|)$. Finally, for \textsc{Gccg}, the complexity of creating the kernel matrix is $O(n|^2 + n|\Pcal|)$, and the complexity of the stochastic greedy algorithm is $\tilde{O}(n^2 + n|\Pcal|)$. 
\end{itemize}
We end with a few comments. First, most of the complexity analysis above is with the stochastic greedy algorithm~\cite{mirzasoleiman2015lazier}. If we use the naive or lazy greedy algorithm, the worst-case complexity is a factor $B$ larger. Secondly, we ignore log-factors in the complexity of stochastic greedy since the complexity is actually $O(n\log 1/\epsilon)$, which achieves a $1 - 1/e - \epsilon$ approximation. Finally, the complexity of optimizing and constructing the FL, LogDet, and GC functions can be obtained from the CG versions by setting $\Pcal = \emptyset$.

\subsection{Details on Partitioning Approach}
For larger datasets, we can choose to partition the unlabeled set into chunks in order to meet the scale of the dataset. This is because many of the techniques (specifically LogDet functions, \textsc{Flvmi}, \textsc{Flcg}, \textsc{Flcmi}, \textsc{Gccg}) all have $O(n^2)$ space complexity. For $n$ in the range of a few million to a few billion data points (which is not uncommon in big-data applications today), we need to scale our algorithms to be linear in $n$ and not quadratic. For this, we propose a simple partitioning approach where the unlabeled data is chunked into $p$ partitions. In this strategy, we perform unlabeled instance acquisition on each chunk using a proportional fraction of the full budget. By performing acquisition on the full unlabeled set, almost all strategies will exhaust the available compute resources. Hence, to execute most selection strategies, we can partition the unlabeled set into equally sized chunks, such that each partition has around 10k to 20k instances. As $n$ grows, the number of partitions would also grow, so that $n/p$ is roughly constant. The complexity of most approaches discussed above would then be $O(n^2/p)$ ($O(n^2/p^2)$ for each chunk, repeated $p$ times), and if $n/p = r$ is a constant, then the complexity $O(nr)$ would be linear in $n$. We then acquire a number of unlabeled instances from each chunk whose ratio with the full selection budget is equal to the ratio between the chunk size and the full unlabeled set. The acquired instances from each chunk are then combined to form the full acquired set of unlabeled instances.

%PRISM MODELING A BROAD SPECTRUM OF SEMANTICS
\section{\model{} Modeling a Broad Spectrum of Semantics}
\label{app:synthetic}
Here we provide additional details on our experiments for empirically studying the modeling capabilities of different functions in \model{} and the role of different parameters.

\subsection{Experimental Setup}
We create a synthetic dataset to understand the behaviour of the different functions in \model, and the corresponding control parameters. We generate 10 different collections of points in a 2D space that emulates the space of images and queries/private instances. In each collection, there are 100 points representing images, 8 points representing queries and 2 points representing the private instances. These 110 points in each set are distributed in 18 clusters with different number of points in each cluster. The standard deviation is varied from one set to another and the 8 query points and 2 private instances for each set are randomly sampled without replacement, one each from 10 randomly selected clusters.

For different functions in \model{} and for different settings of the internal parameters we maximize the function to produce a summary and compute the relevant measures averaged across different budgets (5-40 in intervals of 5) and different collections. 

\textbf{Scoring Functions.} To characterize query-focused and privacy-preserving summaries we define the following. \emph{Query-saturation} is a phenomenon where the function doesn't see any gains in picking more query-relevant items after having picked a few. \emph{Query Coverage} is calculated as the fraction of query points covered by the summary and measures if a summary doesn't starve any query by always picking elements matching some other queries. A query point is said to be covered by a summary if there exists a selection in the summary which belongs to the same cluster as the query point. We quantify the \emph{Diversity} of a summary by calculating the fraction of unique clusters covered by the summary. Next, we define \emph{Query-Relevance} to be the fraction of points selected which match some query point and we define \emph{Privacy-Irrelevance} as the fraction of points selected which do not match any private instance.

\begin{figure}[h]
    \centering 
  \includegraphics[width=0.7\textwidth]{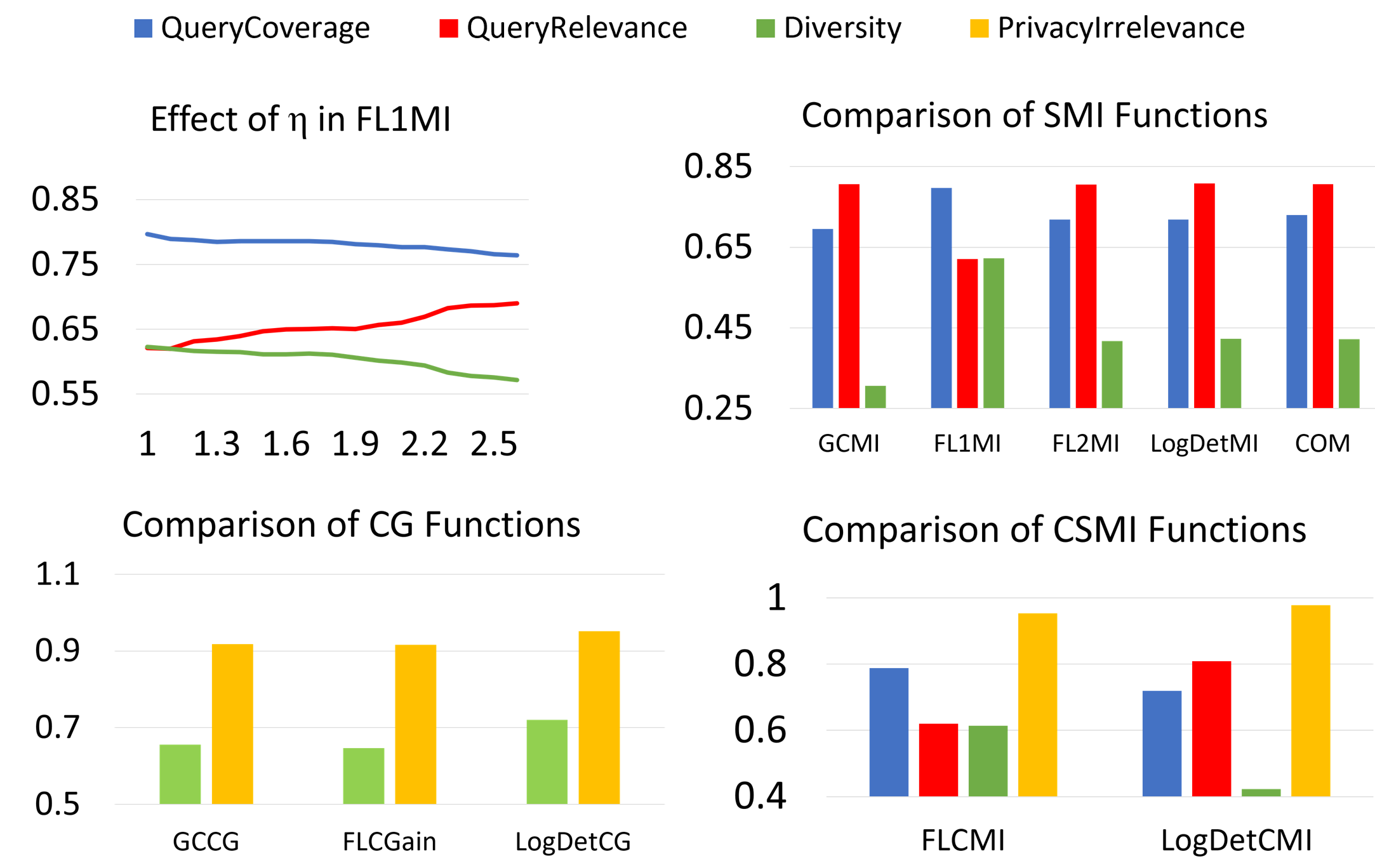}
%\vspace{-2ex}
\caption{Comparison of different functions in \model{} and effect of parameters. All plots share the legend. }
%\vspace{-2ex}
\label{fig:syntheticcombinedapp}
\end{figure}

\subsection{Additional Quantitative Results}
In Fig.~\ref{fig:syntheticcombinedapp} (also in the main text) we presented the behavior of the first variant (v1) of FLMI, as we change the internal parameter $\eta$ (top-left). %In the main paper we have reported the effect of $\eta$ on the query-relevance vs diversity in the case of \textsc{Flvmi}. 
Here we present similar observations for other functions like \textsc{Flqmi} and \textsc{Logdetmi} (Fig.~\ref{fig:fl2mi-logdetmi} a, b).
\begin{figure}[h!]
    \centering 
\begin{subfigure}{0.4\textwidth}
  \includegraphics[width=\linewidth]{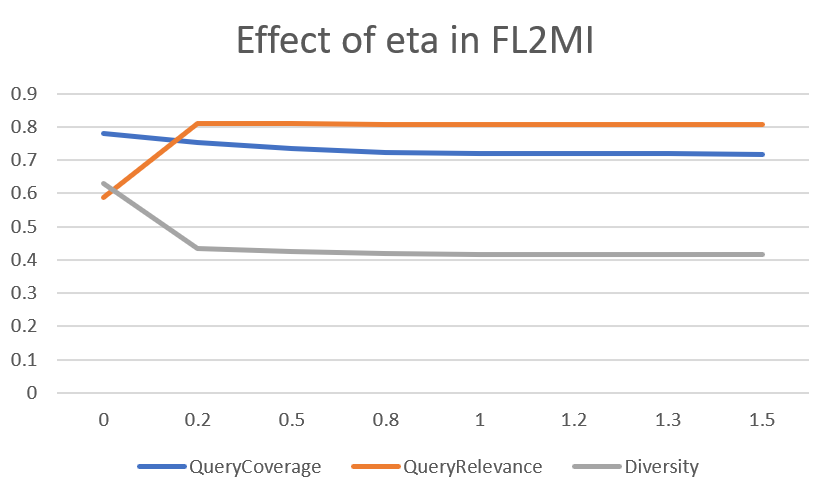}
  \caption{\textsc{Flqmi}}
\end{subfigure} % <-- added
\begin{subfigure}{0.4\textwidth}
  \includegraphics[width=\linewidth]{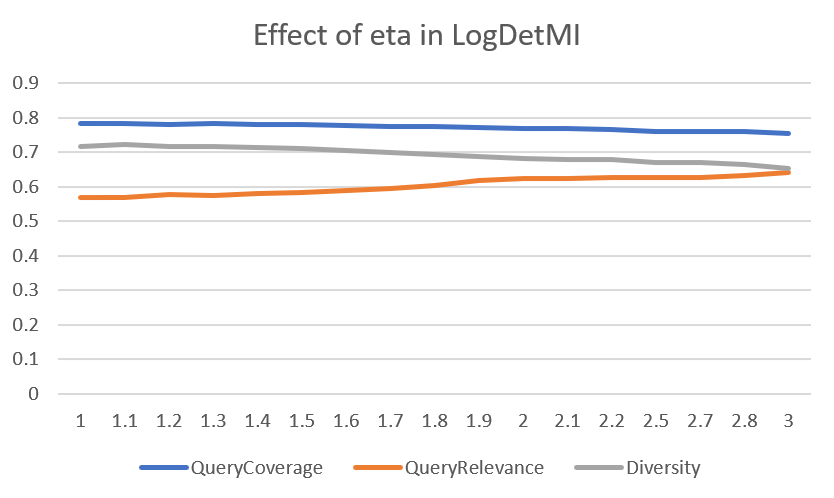}
  \caption{\textsc{Logdetmi}}
\end{subfigure} % <-- added
\begin{subfigure}{0.4\textwidth}
  \includegraphics[width=\linewidth]{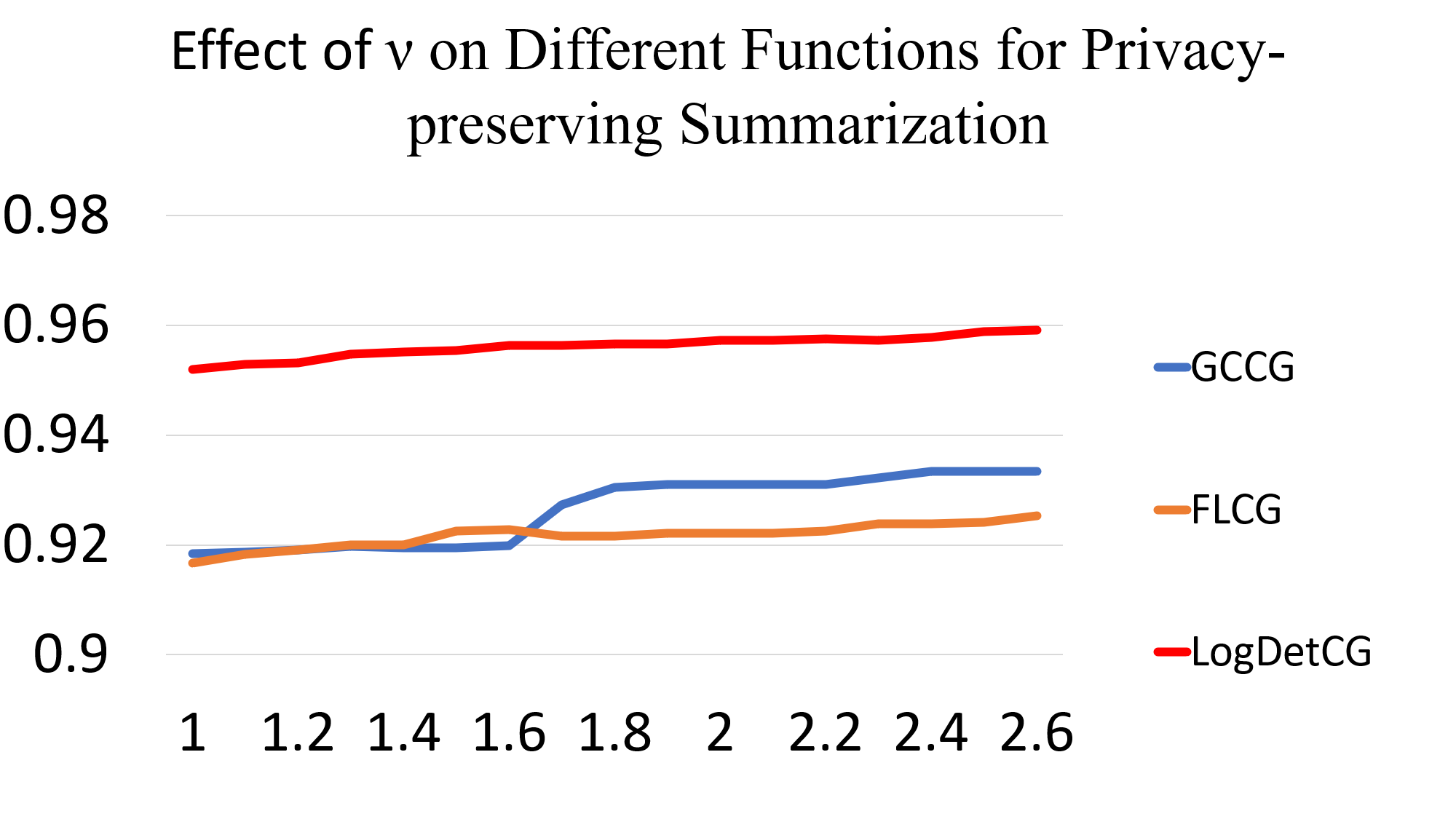}
  \caption{}
  \label{fig:privacy-nu}
  \end{subfigure}
\caption{Effect of $\eta$ on (a)\textsc{Flqmi} and (b)\textsc{Logdetmi} and the effect of $\nu$ on different CG functions (c)}
\label{fig:fl2mi-logdetmi}
\end{figure}

In Fig.~\ref{fig:syntheticcombinedapp} (top-right), we compare query-coverage, diversity and query-relevance for \textsc{Gcmi}, \textsc{Flvmi} and \textsc{Flqmi}, \textsc{Logdetmi} and COM, fixing the value of $\eta = 1$ wherever applicable. In each case, we also compare a version which adds a very small diversity term to these functions (to measure the effect of saturation of the MI functions) (Fig.~\ref{fig:query-comparison}). We make the following observations. \textsc{Gcmi}, \textsc{Flqmi}, \textsc{Logdetmi} and COM favor query-relevance over diversity and query-coverage, while \textsc{Flvmi} favors diversity and query-coverage over query-relevance. Furthermore, we also observe that COM does not change as much with the addition of diversity, which suggests that it is not saturated, while \textsc{Logdetmi} significantly changes its behaviour with the addition of diversity. In almost all cases, we see that adding a small diversity term reduces query-relevance in favor of query-coverage and diversity, which is also something we expect.

\begin{figure}[h!]
    \centering 
    \begin{subfigure}{0.4\textwidth}
  \includegraphics[width=\linewidth]{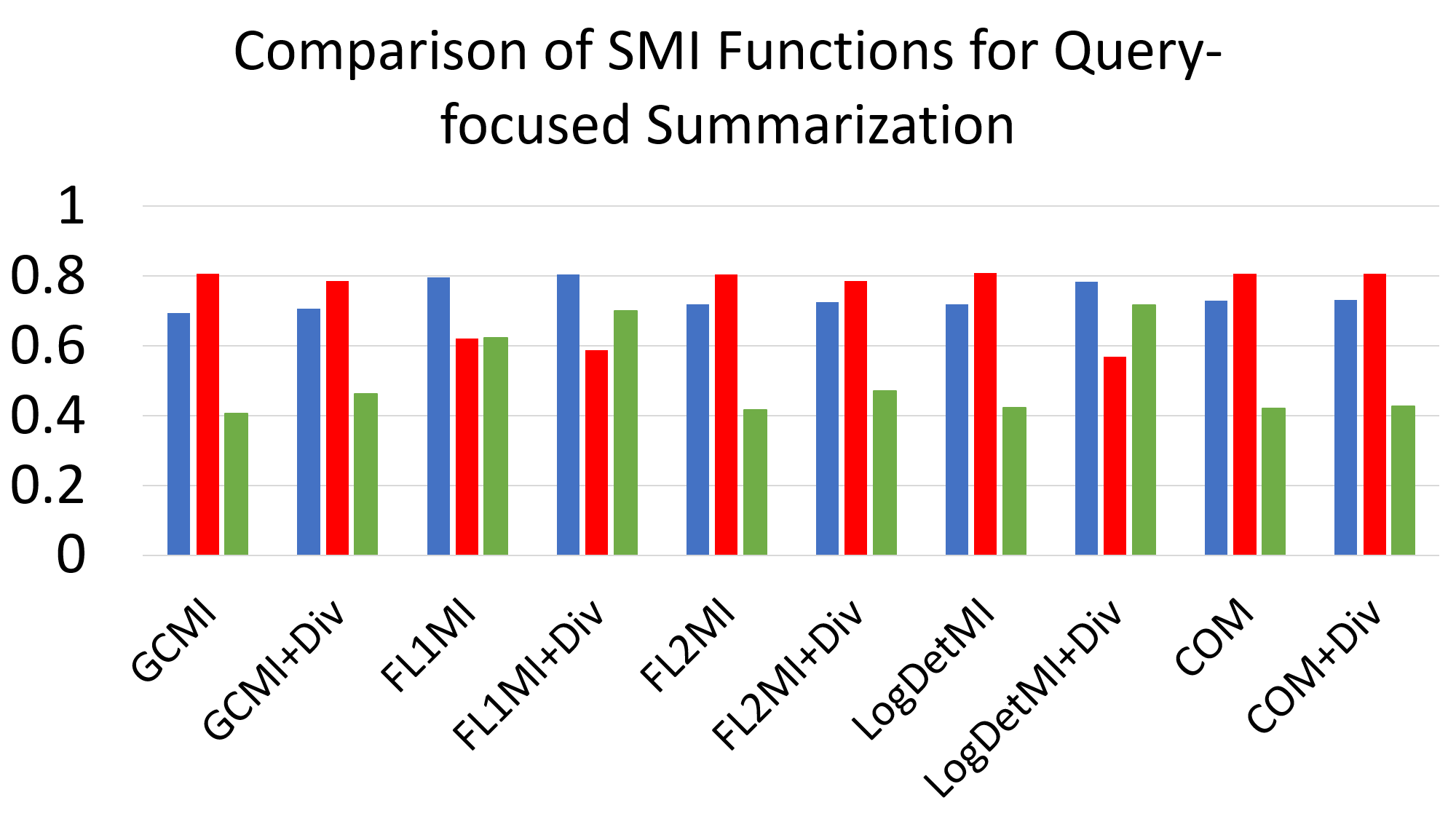}
  \caption{}
  \label{fig:query-comparison}
\end{subfigure}
\begin{subfigure}{0.4\textwidth}
  \includegraphics[width=\linewidth]{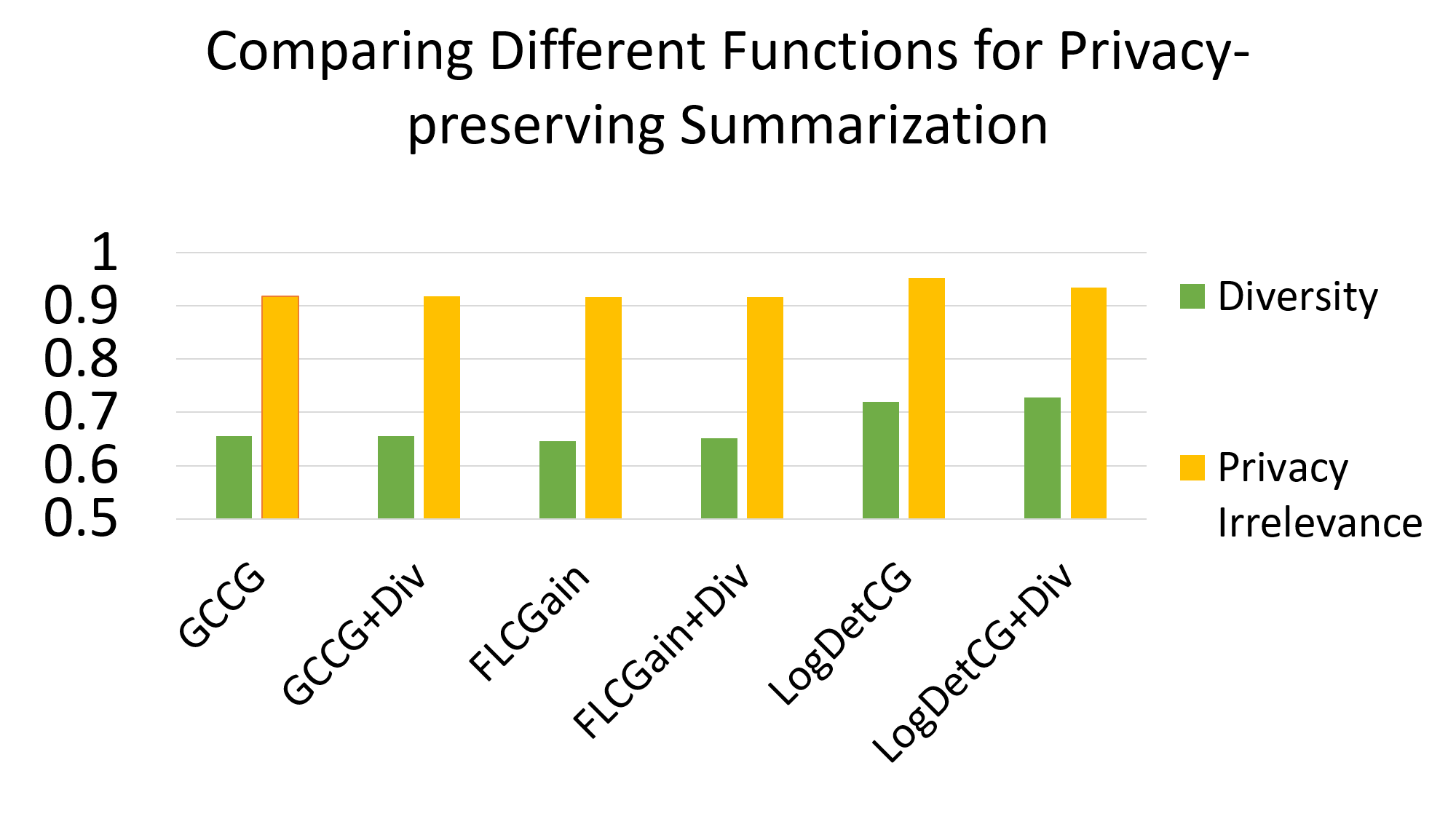}
  \caption{}
  \label{fig:privacy-comp}
\end{subfigure}
\begin{subfigure}{0.4\textwidth}
  \includegraphics[width=\linewidth]{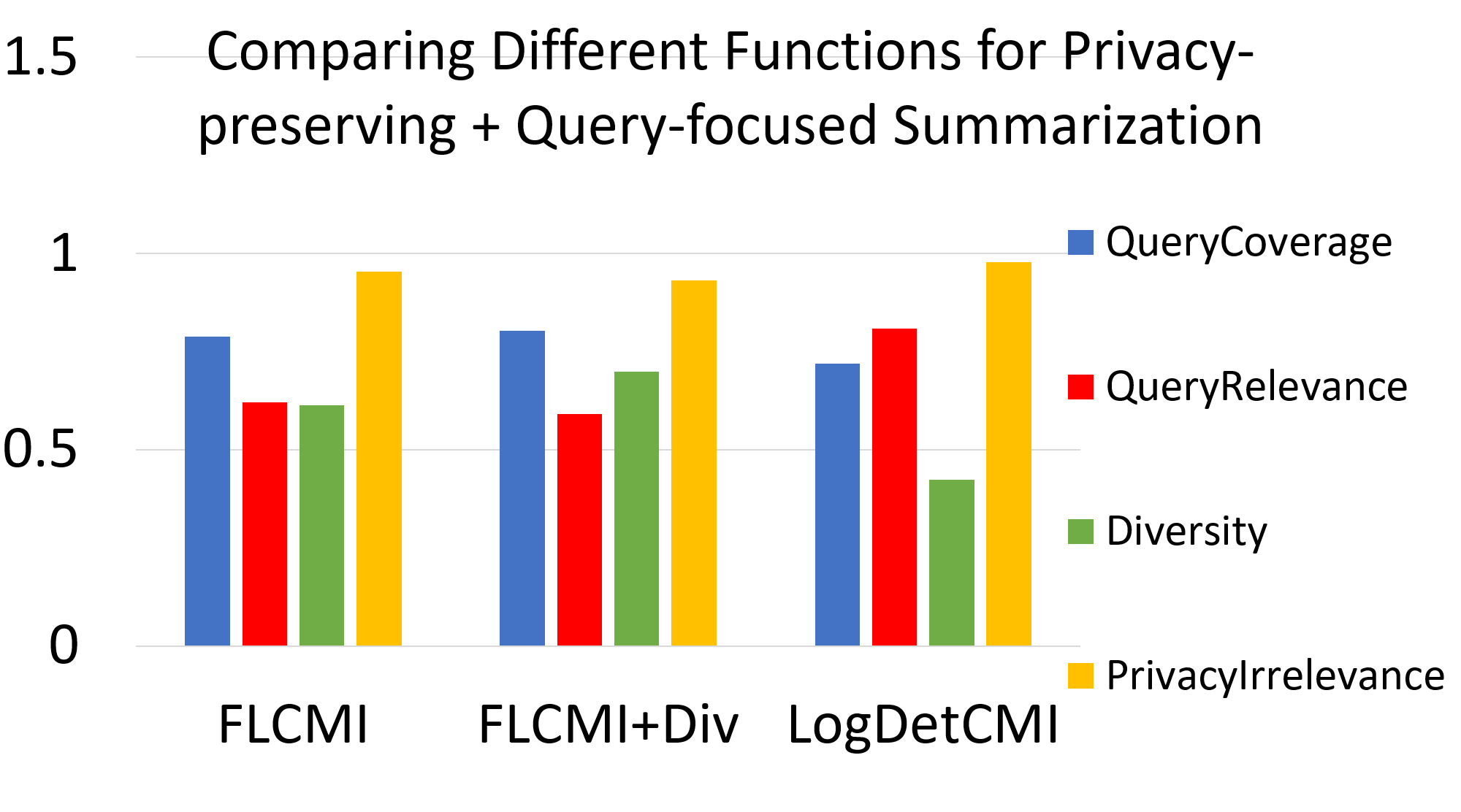}
  \caption{}
  \label{fig:joint-comp}
\end{subfigure} % <-- added
\caption{Behavior of different functions in \model{} and effect of parameters. Different from Fig.~\ref{fig:syntheticcombinedapp}, we also add a version which adds a very small diversity term to these functions (to measure the effect of saturation of the MI functions). See text for details.}
\label{fig:synthetic}
\end{figure}

Next, we report the results for privacy-preserving summarization. Fig.~\ref{fig:privacy-nu} shows the effect of $\nu$ on the \textbf{privacy-irrelevance} term in \textsc{Gccg}, \textsc{Flcg} and \textsc{Logdetcg}. As we expect, increasing $\nu$ increases the privacy-irrelevance score, thereby ensuring a stricter privacy-irrelevance constraint. Fig.~\ref{fig:privacy-comp} compares the Diversity and Privacy-Irrelevance score with different choices of functions (GC, FL and LogDet) for a fixed value $\nu = 1$. Again, we also compare these to their variants where we add a small amount of diversity, and unlike the MI case, we see that the CG functions do not saturate and adding a small diversity does not change the selection. Finally, we see the trend here that Log-Det outperforms FL and GC both in terms of diversity and privacy irrelevance. In Fig.~\ref{fig:joint-comp} we report the results for joint-summarization. We show the comparison for different functions (\textsc{Flcmi}, \textsc{Flcmi} + Div and \textsc{Logdetcmi}). Similar to the private and query versions, we observe that \textsc{Flcmi} tends to favor query-coverage and diversity in contrast to query-relevance and privacy-irrelevance, while \textsc{Logdetcmi} favors query-relevance and privacy-irrelevance over query-coverage and diversity.  

\subsection{Qualitative Analysis}

In Fig.~\ref{fig:visualization} we show the visualization of the 100 image points (black) and 8 query points (green) of collection number 3 in our synthetic dataset along with the 10 selected summary points (blue) selected by \textsc{Flqmi} at $\eta=0.0$ and $\eta=0.2$, labeled as per the order of their selection. F, R, D, I stand for Query-coverage, Query-relevance, Diversity and Privacy-irrelevance respectively. As discussed above, we can see that as soon as $\eta$ is increased, the summary produced by \textsc{Flqmi} becomes more query-relevant and less diverse.

\begin{figure}[h!]
    \centering 
\begin{subfigure}{0.4\textwidth}
  \includegraphics[width=\linewidth]{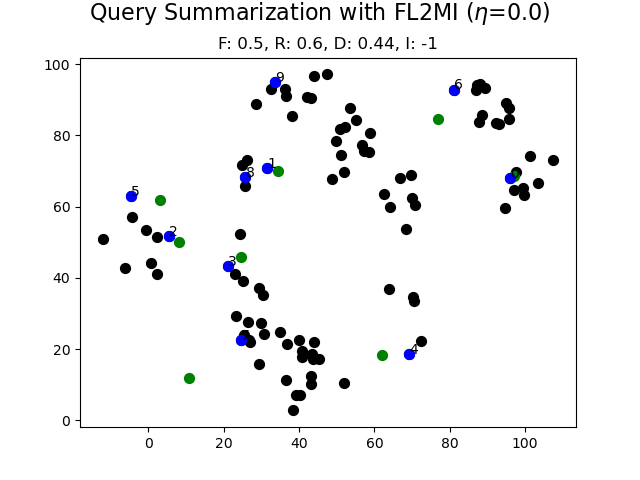}
  \caption{$\eta=0.0$}
\end{subfigure} % <-- added
\begin{subfigure}{0.4\textwidth}
  \includegraphics[width=\linewidth]{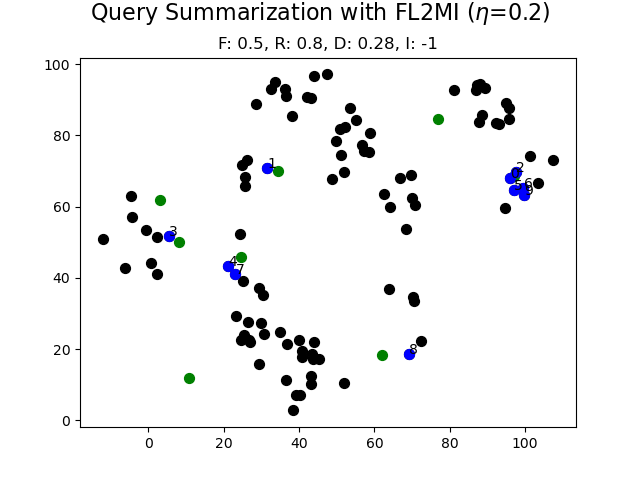}
  \caption{$\eta=0.2$}
\end{subfigure} % <-- added
\caption{Visualization of \textsc{Flqmi} behavior with varying $\eta$ on collection number 3 of the synthetic dataset}
\label{fig:visualization}
\end{figure}

\section{Unified Learning Framework for Guided Summarization using \model{}}
\label{app:learning}

Below we present the specific forms of the mixture model and the objective function and computation of gradients in the different cases of Generic, Query-Focused, Privacy-Preserving and Joint Summarization

\subsubsection{Generic Summarization}
We denote our dataset of $N$ training examples as $(\Ycal^{(n)}, \Vcal^{(n)}, x^{(n)})$ where $n=1 \dots N$, $\Ycal^{(n)}$ is a human summary for the $n^{th}$ ground set (image collection) $\Vcal^{(n)}$ with features $x^{(n)}$.

We denote our mixture model in case of generic summarization as $$F(Y, x^{(n)}, w, \lambda) = \sum_{i=1}^M w_i f_i(Y, x^{(n)}, \lambda_i) $$ where $f_1 \dots f_M$ are the instantiations of different submodular functions, $w_i$ their weights and $\lambda_1 \dots \lambda_M$ are their internal parameters respectively, for example the $\lambda$ in case of Graph Cut function defined above.

So the parameters vector in case of generic summarization becomes $\Theta = (w_1 \dots w_m, \lambda_1 \dots \lambda_M)$

Then, $\Lcal_n(\Theta) = \max_{Y \subset \Vcal^{(n)}, |\Ycal| \leq k} [\sum_{i=1}^M w_i f_i(Y, x^{(n)}, \lambda_i) + l_n(\Ycal)] - \sum_{i=1}^M w_i f_i(\Ycal^{(n)}, x^{(n)}, \lambda_i) $

For the purpose of learning the parameters $w_i$ and $\lambda_i$, we compute the gradients as, $$\frac{\partial L_n}{\partial w_i} = f_i(\hat{Y_n}, x^{(n)}, \lambda_i) - f_i(\Ycal^{(n)}, x^{(n)}, \lambda_i)$$ and $$\frac{\partial L_n}{\partial \lambda_i} = w_i \frac{\partial f_i(\hat{Y_n}, x^{(n)}, \lambda_i)}{\partial \lambda_i}  - w_i \frac{\partial f_i(\Ycal^{(n)}, x^{(n)}, \lambda_i)}{\partial \lambda_i} $$ where $$\hat{Y_n} = \argmax_{Y \subset \Vcal^{(n)}, |\Ycal| \leq k} F(Y, x^{(n)}, w, \lambda) + l_n(\Ycal) $$

For the gradients with respect to the respective internal parameters of individual function components $\frac{\partial f_i}{\partial \lambda_i}$, consider the generalized graphcut  $f(Y, x^{(n)}, \lambda) =   \sum_{i \in \Vcal, j \in Y} S_{ij}^{(n)} - \lambda \sum_{i, j \in Y} S_{ij}^{(n)}$ as an example. We compute its gradient as,\\
$$\frac{\partial f (Y, x^{(n)}, \lambda)}{\partial \lambda} = -\sum_{i, j \in Y} S_{ij}^{(n)} $$

\subsubsection{Query-Focused Summarization}

We denote our dataset of $N$ training examples as $(\Ycal^{(n)}, \Vcal^{(n)}, x^{(n)},  \Qcal^{(n)})$ where $n=1 \dots N$, $\Ycal^{(n)}$ is a human query-summary for the query $ \Qcal^{(n)}$ on the $n^{th}$ ground set (image collection) $\Vcal^{(n)}$ with features $x^{(n)}$.

We denote our mixture model in case of query summarization as $$F(Y,  \Qcal^{(n)}, x^{(n)}, w, \lambda, \eta) = \sum_{i=1}^M w_i I_{f_i}(Y,  \Qcal^{(n)}, x^{(n)}, \lambda_i, \eta_i) $$ where $f_1 \dots f_M$ are the instantiations of different submodular mutual information functions, $w_i$ their weights, $\lambda_1 \dots \lambda_M$ are their internal parameters respectively and $\eta_1 \dots \eta_M$ are their query-relevance-diversity tradeoff parameters. 

So the parameters vector in case of query-focused summarization becomes $\Theta = (w_1 \dots w_m, \lambda_1 \dots \lambda_M, \eta_1 \dots \eta_M)$

Then, $$\Lcal_n(\Theta) = \max_{Y \subset \Vcal^{(n)}, |\Ycal| \leq k} [\sum_{i=1}^M w_i I_{f_i}(Y, \Qcal^{(n)}, x^{(n)}, \lambda_i, \eta_i) + l_n(\Ycal)] - \sum_{i=1}^M w_i I_{f_i}(\Ycal^{(n)}, \Qcal^{(n)}, x^{(n)}, \lambda_i, \eta_i) $$

For the purpose of learning the parameters $w_i$, $\lambda_i$ and $\eta_i$ we compute the gradients as, $$\frac{\partial L_n}{\partial w_i} = I_{f_i}(\hat{Y_n}, \Qcal^{(n)}, x^{(n)}, \lambda_i, \eta_i) - I_{f_i}(\Ycal^{(n)}, \Qcal^{(n)}, x^{(n)}, \lambda_i, \eta_i)$$ $$\frac{\partial L_n}{\partial \lambda_i} = w_i \frac{\partial I_{f_i}(\hat{Y_n}, \Qcal^{(n)}, x^{(n)}, \lambda_i, \eta_i)}{\partial \lambda_i}  - w_i \frac{\partial I_{f_i}(\Ycal^{(n)}, \Qcal^{(n)}, x^{(n)}, \lambda_i, \eta_i)}{\partial \lambda_i} $$ and $$\frac{\partial L_n}{\partial \eta_i} = w_i \frac{\partial I_{f_i}(\hat{Y_n}, \Qcal^{(n)}, x^{(n)}, \lambda_i, \eta_i)}{\partial \eta_i}  - w_i \frac{\partial I_{f_i}(\Ycal^{(n)}, \Qcal^{(n)}, x^{(n)}, \lambda_i, \eta_i)}{\partial \eta_i} $$ where $$\hat{Y_n} = \argmax_{Y \subset \Vcal^{(n)}, |\Ycal| \leq k} F(Y, \Qcal^{(n)}, x^{(n)}, w, \lambda, \eta) + l_n(\Ycal) $$\\

For the gradients of individual function components $\frac{\partial I_{f_i}}{\partial \eta_i}$ with respect to the respective query-relevance-diversity trade-off parameters $\eta_i$, we show computation for some functions as follows:

\textbf{\textsc{Flvmi}:} $I_f(Y, \Qcal^{(n)}, x^{(n)}, \eta)=\sum_{i \in \Vcal}\min(\max_{j \in Y}S_{ij}^{(n)}, \eta \max_{j \in \Qcal^{(n)}}S_{ij}^{(n)})$\\
    $$ \frac{\partial I_f (Y, \Qcal^{(n)}, x^{(n)}, \eta)}{\partial \eta} = \sum_{i \in \Vcal} (\max_{j \in \Qcal^{(n)}}S_{ij}^{(n)}*1_{\max_{j \in \Qcal^{(n)}}S_{ij}^{(n)} \leq \max_{j \in Y}S_{ij}^{(n)}})$$
    
    \textbf{\textsc{Flqmi}:} $I_f(Y, \Qcal^{(n)}, x^{(n)}, \eta) = \sum_{i \in \Qcal^{(n)}} \max_{j \in Y} S_{ij}^{(n)} + \eta \sum_{i \in Y} \max_{j \in \Qcal^{(n)}} S_{ij}^{(n)}$\\
    $$ \frac{\partial I_f (Y, \Qcal^{(n)}, x^{(n)}, \eta)}{\partial \eta} = \sum_{i \in Y} \max_{j \in \Qcal^{(n)}} S_{ij}^{(n)}$$
    
    \textbf{\textsc{Logdetmi}:} $I_f(Y, \Qcal^{(n)}, x^{(n)}, \eta) = -\log \det(I - \eta^2 S_{Y}^{-1}S_{Y \Qcal^{(n)}}S_{ \Qcal^{(n)}}^{-1}S_{Y \Qcal^{(n)}}^T)$
    
    We have, 
    $$\frac{-\partial\log\det(X)}{\partial \eta} = \frac{-1}{\det(X)} \frac{\det (X)}{\partial \eta} $$ and $\frac{\partial \det (X)}{\partial \eta} = \det(X) \mathrm{Tr}[ X^{-1} \frac{\partial X}{\partial \eta}]$. Hence, with $X = I - \eta^2 S_{Y}^{-1}S_{Y \Qcal^{(n)}}S_{ \Qcal^{(n)}}^{-1}S_{Y \Qcal^{(n)}}^T$ we have,
    $$ \frac{\partial I_f (Y,  \Qcal^{(n)}, x^{(n)}, \eta)}{\partial \eta} = \mathrm{Tr}[((I - \eta^2 S_{Y}^{-1}S_{Y \Qcal^{(n)}}S_{ \Qcal^{(n)}}^{-1}S_{Y \Qcal^{(n)}}^T))^{-1}*2\eta(S_{Y}^{-1}S_{Y \Qcal^{(n)}}S_{ \Qcal^{(n)}}^{-1}S_{Y \Qcal^{(n)}}^T)]$$

\subsubsection{Privacy-Preserving Summarization}

We denote our dataset of $N$ training examples as $(\Ycal^{(n)}, \Vcal^{(n)}, x^{(n)},  \Pcal^{(n)})$ where $n=1 \dots N$, $\Ycal^{(n)}$ is a human privacy-summary for the privacy set $ \Pcal^{(n)}$ on the $n^{th}$ ground set (image collection) $\Vcal^{(n)}$ with features $x^{(n)}$.

We denote our mixture model in case of privacy-preserving summarization as $$F(Y,  \Pcal^{(n)}, x^{(n)}, w, \lambda, \nu) = \sum_{i=1}^M w_i f_i(Y,  \Pcal^{(n)}, x^{(n)}, \lambda_i, \nu_i) $$ where $f_1 \dots f_M$ are the instantiations of different conditional gain functions, $w_i$ their weights, $\lambda_1 \dots \lambda_M$ are their internal parameters respectively and $\nu_1 \dots \nu_M$ are their privacy-sensitivity parameters. 

So the parameters vector in case of privacy-preserving summarization becomes $\Theta = (w_1 \dots w_m, \lambda_1 \dots \lambda_M, \nu_1 \dots \eta_M)$

Then, $$\Lcal_n(\Theta) = \max_{Y \subset \Vcal^{(n)}, |\Ycal| \leq k} [\sum_{i=1}^M w_i f_i(Y,  \Pcal^{(n)}, x^{(n)}, \lambda_i, \nu_i) + l_n(\Ycal)] - \sum_{i=1}^M w_i f_i(\Ycal^{(n)},  \Pcal^{(n)}, x^{(n)}, \lambda_i, \nu_i) $$

For the purpose of learning the parameters $w_i$, $\lambda_i$ and $\nu_i$ we compute the gradients as, $$\frac{\partial L_n}{\partial w_i} = f_i(\hat{Y_n},  \Pcal^{(n)}, x^{(n)}, \lambda_i, \nu_i) - f_i(\Ycal^{(n)},  \Pcal^{(n)}, x^{(n)}, \lambda_i, \nu_i)$$ $$\frac{\partial L_n}{\partial \lambda_i} = w_i \frac{\partial f_i(\hat{Y_n},  \Pcal^{(n)}, x^{(n)}, \lambda_i, \nu_i)}{\partial \lambda_i}  - w_i \frac{\partial f_i(\Ycal^{(n)},  \Pcal^{(n)}, x^{(n)}, \lambda_i, \nu_i)}{\partial \lambda_i} $$ and $$\frac{\partial L_n}{\partial \nu_i} = w_i \frac{\partial f_i(\hat{Y_n},  \Pcal^{(n)}, x^{(n)}, \lambda_i, \nu_i)}{\partial \nu_i}  - w_i \frac{\partial f_i(\Ycal^{(n)},  \Pcal^{(n)}, x^{(n)}, \lambda_i, \nu_i)}{\partial \nu_i} $$ where $$\hat{Y_n} = \argmax_{Y \subset \Vcal^{(n)}, |\Ycal| \leq k} F(Y,  \Pcal^{(n)}, x^{(n)}, w, \lambda, \nu) + l_n(\Ycal) $$

For the gradients of individual function components $\frac{\partial f_i}{\partial \nu_i}$ with respect to the respective privacy sensitivity parameters $\nu_i$, we show computation for some functions as follows:

    \noindent \textbf{FLCondGain: } $f(Y,  \Pcal^{(n)}, x^{(n)}, \nu)= \sum_{i \in \Vcal} \max(\max_{j \in Y} S_{ij}^{(n)} - \nu \max_{j \in  \Pcal^{(n)}} S_{ij}^{(n)}, 0)$\\
    $$ \frac{\partial f(Y,  \Pcal^{(n)}, x^{(n)}, \nu)}{\partial \nu}  = \sum_{i \in \Vcal} (-\max_{j \in  \Pcal^{(n)}}S_{ij}^{(n)})*1_{(\max_{j \in Y} S_{ij}^{(n)} - \nu \max_{j \in  \Pcal^{(n)}} S_{ij}^{(n)}) \geq 0}$$
    \textbf{LogDetCondGain: } $f(Y,  \Pcal^{(n)}, x^{(n)}, \nu) = \log\det(S_Y - \nu^2 S_{Y \Pcal^{(n)}}S_{ \Pcal^{(n)}}^{-1}S_{Y \Pcal^{(n)}}^T)$
    We have, 
    $$\frac{\partial\log\det(X)}{\partial \nu} = \frac{1}{\det(X)} \frac{\det (X)}{\partial \nu} $$ and $\frac{\partial \det (X)}{\partial \nu} = \det(X) \mathrm{Tr}[ X^{-1} \frac{\partial X}{\partial \nu}]$
    Hence, with $X = I - \nu^2 S_{Y}^{-1}S_{Y \Pcal^{(n)}}S_{ \Pcal^{(n)}}^{-1}S_{Y \Pcal^{(n)}}^T$\\
    we have, 
    $$\frac{\partial f(Y,  \Pcal^{(n)}, x^{(n)}, \nu)}{\partial \nu}  = -\mathrm{Tr}[(S_Y - \nu^2 S_{Y \Pcal^{(n)}}S_{ \Pcal^{(n)}}^{-1}S_{Y \Pcal^{(n)}}^T)^{-1}*2\nu(S_{Y \Pcal^{(n)}}S_{ \Pcal^{(n)}}^{-1}S_{Y \Pcal^{(n)}}^T))] $$

\subsubsection{Joint Summarization}

We denote our dataset of $N$ training examples as $(\Ycal^{(n)}, \Vcal^{(n)}, x^{(n)},  \Qcal^{(n)},  \Pcal^{(n)})$ where $n=1 \dots N$, $\Ycal^{(n)}$ is a human query-privacy-summary for the query set $ \Qcal^{(n)}$ and privacy set $ \Pcal^{(n)}$ on the $n^{th}$ ground set (image collection) $\Vcal^{(n)}$ with features $x^{(n)}$.

We denote our mixture model in case of joint query-focused and privacy-preserving summarization as $$F(Y,  \Qcal^{(n)},  \Pcal^{(n)}, x^{(n)}, w, \lambda, \eta, \nu) = \sum_{i=1}^M w_i f_i(Y,  \Qcal^{(n)},  \Pcal^{(n)}, x^{(n)}, \lambda_i, \eta_i, \nu_i) $$ where $f_1 \dots f_M$ are the instantiations of different conditional submodular mutual information functions, $w_i$ their weights, $\lambda_1 \dots \lambda_M$ are their internal parameters respectively, $\eta_i$ are their query-relevance vs diversity trade-off parameters and $\nu_1 \dots \nu_M$ are their privacy-sensitivity parameters. 

So the parameters vector in case of joint summarization becomes $\Theta = (w_1 \dots w_m, \lambda_1 \dots \lambda_M, \eta_1 \dots \eta_M, \nu_1 \dots \eta_M)$

Then, $$\Lcal_n(\Theta) = \max_{Y \subset \Vcal^{(n)}, |\Ycal| \leq k} [\sum_{i=1}^M w_i f_i(Y,  \Qcal^{(n)},  \Pcal^{(n)}, x^{(n)}, \lambda_i, \eta_i, \nu_i) + l_n(\Ycal)] - \sum_{i=1}^M w_i f_i(\Ycal^{(n)},  \Qcal^{(n)},  \Pcal^{(n)}, x^{(n)}, \lambda_i, \eta_i, \nu_i) $$

For the purpose of learning the parameters in $\Theta$ we compute the gradients as, $$\frac{\partial L_n}{\partial w_i} = f_i(\hat{Y_n},  \Qcal^{(n)},  \Pcal^{(n)}, x^{(n)}, \lambda_i, \eta_i, \nu_i) - f_i(\Ycal^{(n)},  \Qcal^{(n)},  \Pcal^{(n)}, x^{(n)}, \lambda_i, \eta_i, \nu_i)$$ $$\frac{\partial L_n}{\partial \lambda_i} = w_i \frac{\partial f_i(\hat{Y_n},  \Qcal^{(n)},  \Pcal^{(n)}, x^{(n)}, \lambda_i, \eta_i, \nu_i)}{\partial \lambda_i}  - w_i \frac{\partial f_i(\Ycal^{(n)},  \Qcal^{(n)},  \Pcal^{(n)}, x^{(n)}, \lambda_i, \eta_i, \nu_i)}{\partial \lambda_i} $$ $$\frac{\partial L_n}{\partial \eta_i} = w_i \frac{\partial f_i(\hat{Y_n},  \Qcal^{(n)},  \Pcal^{(n)}, x^{(n)}, \lambda_i, \eta_i, \nu_i)}{\partial \eta_i}  - w_i \frac{\partial f_i(\Ycal^{(n)},  \Qcal^{(n)},  \Pcal^{(n)}, x^{(n)}, \lambda_i, \eta_i, \nu_i)}{\partial \eta_i} $$ $$\frac{\partial L_n}{\partial \nu_i} = w_i \frac{\partial f_i(\hat{Y_n},  \Qcal^{(n)},  \Pcal^{(n)}, x^{(n)}, \lambda_i, \eta_i, \nu_i)}{\partial \nu_i}  - w_i \frac{\partial f_i(\Ycal^{(n)},  \Qcal^{(n)},  \Pcal^{(n)}, x^{(n)}, \lambda_i, \eta_i, \nu_i)}{\partial \nu_i} $$ where $$\hat{Y_n} = \argmax_{Y \subset \Vcal^{(n)}, |\Ycal| \leq k} F(Y,  \Qcal^{(n)},  \Pcal^{(n)}, x^{(n)}, w, \lambda, \eta, \nu) + l_n(\Ycal) $$

\section{Connections to Past Work} \label{app:connections}

\subsection{Connections related to Targeted Learning} \label{app:connectionsTL}

\subsubsection{\textsc{Glister} is a special case of Algorithm ~\ref{algo:tss} in certain cases}
\label{app:glister-lemma}

%\subsection{Applying \textsc{Glister} to Targeted Subset Selection}\label{app:glister-targeted}
In this subsection, we first study the application of \textsc{Glister} to targeted data selection. In particular, we can formulate \textsc{Glister}~\cite{killamsetty2020glister} as:
\begin{align}\label{eq:glisterexp}
    &\min_{\Acal \subseteq \Vcal, |\Acal| = k} \Lcal(\theta_{\Acal}, \Tcal) \nonumber \\
    &\mbox{where } \theta_{\Acal} = \min_{\theta} \Lcal(\theta, \Acal)
\end{align}
Recall that given a set $\Acal$, the loss $\Lcal(\theta, \Acal) = \sum_{i \in \Acal} \Lcal(x_i, y_i, \theta)$. Furthermore, if the set $\Acal$ consists of unlabeled examples, we use hypothesized labeles similar to \textsc{Glister-Active}~\cite{killamsetty2020glister}.

In a manner similar to \textsc{Glister-Active}, we can apply the targeted setting as follows. Given the current model parameters $\theta$ (obtained by training the model on the labeled set), we can apply a one step gradient approximation to~\eqref{eq:glisterexp} and we obtain:
\begin{align}\label{eq:glisteronestep}
    &\min_{\Acal \subseteq \Vcal, |\Acal| = k} \Lcal(\theta - \eta \sum_{i \in \Acal} \nabla_{\theta} \Lcal_i(\theta), \Tcal)
\end{align}
We can then directly adapt Theorem 1 from~\cite{killamsetty2020glister} and obtain the following.
\begin{lemma}
When the loss function $\Lcal$ is either the Hinge Loss, Logistic Loss, Square Loss of the Perceptron Loss, Eq.~\eqref{eq:glisteronestep} can be written as a constrained submodular maximization problem.
\end{lemma}
This means that we can obtain the solution using a simple greedy algorithm.

We now see that under certain conditions, \textsc{Glister} is a special case of Algorithm ~\ref{algo:tss}.

\begin{lemma}
\textsc{Glister} with either of Hinge Loss, Logistic Loss or the Perceptron Loss is a special case of Algorithm~\ref{algo:tss} when $I_f$ is COM ($\eta = 0$ and $\lambda = 0$).
\end{lemma}
\begin{proof}
The proof of this follows almost directly from~\cite{killamsetty2020glister}. In particular, Eq.~\eqref{eq:glisteronestep} is of the form $f(\Acal) = \sum_{j \in \Tcal} \min(c_i(\Acal), \alpha)$ when $\Lcal$ is the Hinge Loss or Perceptron Loss. Similarly, Eq.~\eqref{eq:glisteronestep} is of the form $f(\Acal) = \sum_{j \in \Tcal} C - \log(1 + \exp(c_i(\Acal))$ when $\Lcal$ is the Logistic Loss. Both these functions are concave over modular functions and hence \textsc{Glister} is a special case of Algorithm~\ref{algo:tss} when using COM as the GMI function. 
\end{proof}

\subsubsection{\textsc{TargetedCraig} is a special case of Algorithm ~\ref{algo:tss}}
\label{app:targeted-craig}
\begin{lemma}
\textsc{TargetedCraig} is a special case of Algorithm~\ref{algo:tss} when $I_f$ is \textsc{Flqmi} and $\eta = 0$.
\end{lemma}

\begin{proof}
This result follows from a simple observation that the expression for \textsc{TargetedCraig} is:
\begin{align}
f(\Acal) = \sum_{i \in \Tcal} \max_{j \in \Acal} S_{ij}
\end{align}
Note that this is exactly the same as $I_g(\Acal; \Tcal)$ where $g$ is the facility location function, $I_g$ is \textsc{Flqmi} and $\eta = 0$ since the generic expression of \textsc{Flqmi} is $I_g(\Acal; \Tcal) = \sum_{i \in \Tcal} \max_{j \in \Acal} S_{ij} + \eta \sum_{i \in \Acal} \max_{j \in \Tcal} S_{ij}$.
\end{proof}

\subsubsection{\textsc{Grad-Match} as a special case of Algorithm ~\ref{algo:tss}}
\label{app:grad-diff}

\begin{lemma}
Minimizing the gradient difference ($h(\Acal)$)%Eq.~\eqref{eqn:tdss-grad-diff}) 
can be rewritten as a special case of Algorithm~\ref{algo:tss} when $I_f(\Acal; \Tcal) = \sum_{i \in \Acal, j \in \Tcal} \langle \nabla \Lcal^U_i(\theta_E), \nabla \Lcal^T_j(\theta_E) \rangle$ is \textsc{Gcmi} and $g(\Acal) = -\sum_{i, j \in \Acal} \langle \nabla \Lcal^U_i(\theta_E), \nabla \Lcal^U_j(\theta_E) \rangle + \sum_{i \in \Acal} \lvert \nabla \Lcal^U_i(\theta_E) \rVert^2 $ is a diversity function and $\gamma = |\Tcal|/k$.
\end{lemma}
\begin{proof}
To prove this result, we expand the gradient difference expression:
\begin{align}
    h(\Acal) = \lVert \frac{1}{|\Tcal|} \nabla \Lcal(\Tcal, \theta_E) - \frac{1}{k} \nabla \Lcal(\Acal, \theta_E) \rVert^2
\end{align}
Note that since we are minimizing the gradient difference and hence $h$, define $g(\Acal) = -h(\Acal)$. Then,
\begin{align}
    g(\Acal) = -\frac{1}{|\Tcal|^2} \lVert \nabla \Lcal(\Tcal, \theta_E) \rVert^2 - \frac{1}{k^2} \lVert \nabla \Lcal(\Acal, \theta_E) \rVert^2 + 2 \frac{1}{|\Tcal|k} \langle \nabla \Lcal(\Tcal, \theta_E), \nabla \Lcal(\Acal, \theta_E) \rangle
\end{align}
We immediately see that the first term is independent of $\Acal$ and is a constant. Similarly, the third term is an instance of \textsc{Gcmi}. We now expand the second term. 
\begin{align}
    - \frac{1}{k^2} \lVert \nabla \Lcal(\Acal, \theta_E) \rVert^2 = - \frac{2}{k^2} [\sum_{i \in \Acal} \lVert \nabla \Lcal_i(\theta_E) \rVert^2 - \sum_{i, j \in \Acal} \langle \nabla \Lcal_i(\theta_E), \nabla \Lcal_j(\theta_E) \rangle
\end{align}
Expanding this out, we get that minimizing the gradient difference can be rewritten as maximizing the sum of \textsc{Gcmi} and a diversity term. 
\end{proof}

\subsection{Connections Related to Guided Summarization}\label{app:connectionsGuidedSumm}

\label{app:sum-gen}
In this section, we discuss how several past works have (unknowingly) used instances of various MI functions.
\begin{itemize}
    \item \textbf{\textsc{Gcmi}: } Several query-focused summarization works for document summarization~\cite{lin2012submodularity, li2012multi} and video summarization~\cite{vasudevan2017query} use \textsc{Gcmi}. All these papers study a simple graph-cut based query relevance term (which is a special case of our submodular mutual information framework) with a single query point: $f(\Acal) = \sum_{i \in \Acal} s_{iq}$. \textsc{Gcmi} seamlessly extends this to consider a query set.
    \item \textbf{\textsc{Gccg} and GCCMI: } The Graph Cut Conditional Gain function was used in update-summarization~\cite{li2012multi} (See Table 1 in their paper). Furthermore, the authors also consider query-focused-update summarization, in which case they use a GCCMI expression: $f(\Acal) = I_f(\Acal; Q | \Acal_0)$ where $\Acal_0$ is an existing summary and the goal is to select a summary relevant to a query $Q$ and yet different from $\Acal_0$. The same authors also study graph cut for query-focused summarization, and in both cases observe the utility of this class of functions. 
    \item \textbf{\textsc{Logdetmi}: } The query-focused summarization model used in ~\cite{sharghi2016query} is very similar to our \textsc{Logdetmi}, if we do not consider the sequential DPP model structure. In particular, if we assume $S_{\Acal} = I, S_{\Qcal} = I$ (i.e. the elements within $\cal V$ and $v$ are independent), then $I_f(\Acal; \Qcal) = \log \det(I - S_{\Acal \Qcal} S_{\Acal \Qcal}^T$), which is then similar to the query term in ~\cite{sharghi2016query} (e.g. Equation (6) in their paper with $W = I$). This shows that \textsc{Logdetmi} as a model makes sense for query-focused summarization.
    \item \textbf{COM: } \cite{lin2011class} propose a combination of query relevance and diversity term for document summarization. The expression they propose is very similar to COM if we ignore the diversity term. This has achieved state of the art results for query-focused document summarization.
    % \item \textbf{ROUGE: } ROUGE is a very common evaluation metric for document summarization~\cite{lin2004rouge,lin2012learning,lin2011class}. As shown in~\cite{lin2011class}, ROUGE metric is actually submodular. We actually observe that ROUGE is in fact exactly the query-saturation (Q-Sat) function and hence is also subsumed in our framework through GMI.
\end{itemize}
Our proposed framework significantly extends these and also provides a rich class of functions for query-focused, privacy-preserving and update summarization. 

\section{Additional Details on Targeted Learning using \model{} (Section ~\ref{sec:tss})} \label{app:tss-exp}

\subsection{Additional Details for Experimental Setup and Further Discussion of Results}% (Section ~\ref{sec:exp-tss})}

We use a common training process and hyperparameters for all datasets. Below, we present the hyperparameters used during training: \begin{itemize}
    \item \textbf{Optimization Algorithm:} SGD with Momentum
    \item\textbf{ Learning Rate:} 0.01 with Cosine Annealing
    \item \textbf{Momentum:} 0.9, Weight Decay: 5e-04
    \item \textbf{Number of Epochs:} 100. We got these numbers by taking the stopping condition to be the training accuracy (more than 99\%).
\end{itemize}

% \begin{figure}[h]
%     \centering 
%   \includegraphics[width=0.49\textwidth]{uai/images/budget_effect_largefont.pdf}
% %\vspace{-2ex}
% \caption{Comparison of different methods for targeted subset selection for different budgets on CIFAR-10 and MNIST. X-axis: budgets, Y-axis: gain in model accuracy for target classes. MI based approaches (lines in {\color{BrickRed}{red})} significantly outperform others across all subset sizes. (Section~\ref{subsec:exp-tss})}
% %\vspace{-2ex}
% \label{fig:gain-size-large}
% \end{figure}

% \textcolor{red} exact splits in case of MNIST and CIFAR, initial labeled set size, how we create the imbalance, the learning rate, number of epochs. Also do not forget to mention that we create the imbalance every time we randomly pick up two new target classes.}

% We report a better resolution image presenting the effect of budget size on the performance of various methods (Fig.~\ref{fig:gain-size-large}). 

We also make following additional observations about the results on targeted learning:
\begin{enumerate}
    \item Pure retrieval function (\textsc{Gcmi}) works better than uncertainty based functions. This is as expected because for the task at hand, i.e. targeted subset selection, relevance with target plays an important role.
    \item \textsc{Flvmi}, which tends to model more of diversity than query-relevance, performs worse than \textsc{Gcmi}.
    \item It appears counter-intuitive that the targeted version of \textsc{glister}(\textsc{glister-tss}) performs better than \textsc{glister} on MNIST, as expected, but worse than \textsc{glister} on SVHN. We think this is the case because \textsc{glister-tss} depends heavily on the target set (optimizing its performance) and thus tends to overfit when the target set has very few instances. In contrast, our MI functions work well even when the target set is very small.
\end{enumerate}
We also note that \textsc{Glister-Tss} used in this setting is not a special case of Algorithm~\ref{algo:tss} since we used the cross entropy loss.

\subsection{Details and Results for Computing Pairwise Penalty Matrix}

\figref{pen-matrix1} shows the penalty matrix results on the average rare class accuracy for MNIST and SVHN, while \figref{pen-matrix2} shows the penaltry matrix for CIFAR-10 and Pneumonia-MNIST. We see that \textsc{Logdetmi}, \textsc{Flqmi}, \textsc{Flvmi} and \textsc{Gcmi} have the smallest column sum, which indicates that most other baselines are not statistically significantly better than them. Furthermore, they also have the highest row sum (followed by some of the other MI functions), which indicates that they are statistically significantly better than other approaches.

The penalty matrices computed in this paper follow the strategy used in~\cite{ash2020deep}. In their strategy, a penalty matrix is constructed for each dataset-model pair. Each cell $(i,j)$ of the matrix reflects the fraction of training rounds that AL with selection algorithm $i$ has higher test accuracy than AL with selection algorithm, $j$ with statistical significance. As such, the average difference between the test accuracies of $i$ and $j$ and the standard error of that difference are computed for each training round. A two-tailed $t$-test is then performed for each training round: If $t>t_\alpha$, then $\frac{1}{N_{train}}$ is added to cell $(i,j)$. If $t<-t_\alpha$, then $\frac{1}{N_{train}}$ is added to cell $(j,i)$. Hence, the full penalty matrix gives a holistic understanding of how each selection algorithm compares against the others: A row with mostly high values signals that the associated selection algorithm performs better than the others; however, a column with mostly high values signals that the associated selection algorithm performs worse than the others. As a final note, \cite{ash2020deep} takes an additional step where they consolidate the matrices for each dataset-model pair into one matrix by taking the sum across these matrices, giving a summary of the AL performance for their entire paper that is fairly weighted to each experiment. We present the penalty matrices for each of the settings in the sections below.

\begin{figure}[!ht]
\centering
\includegraphics[width =0.7\textwidth]{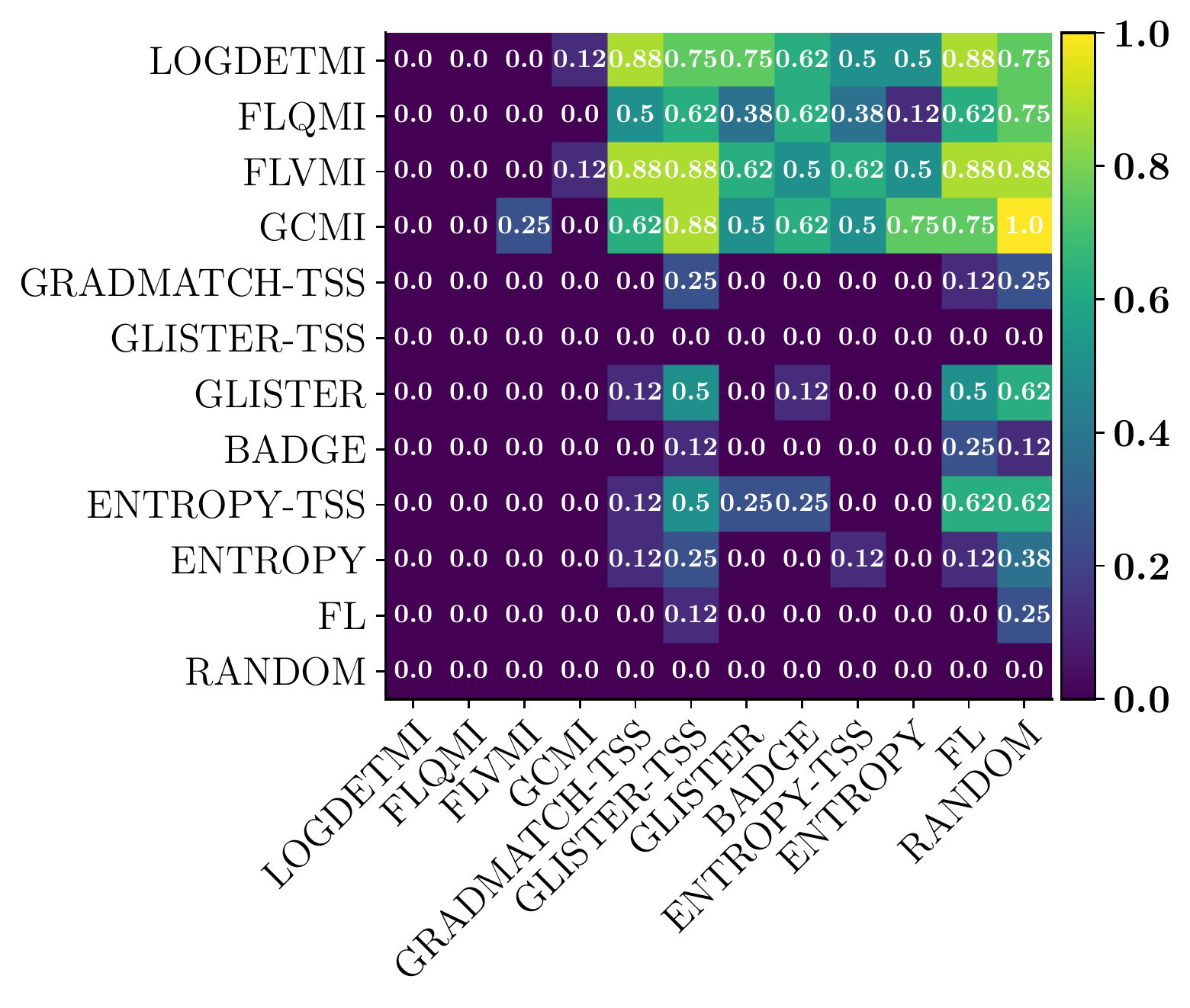}
\includegraphics[width = 0.7\textwidth]{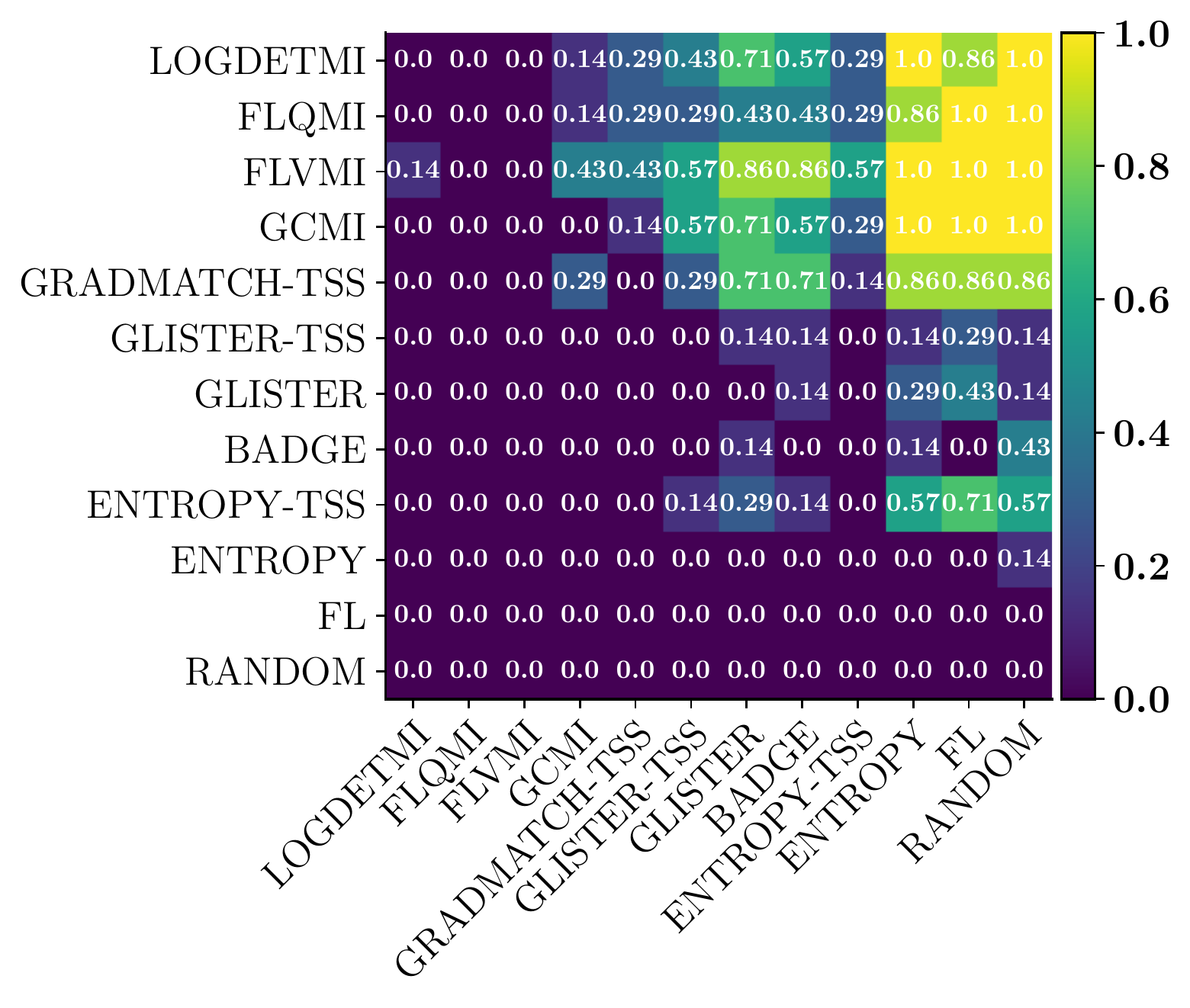}
\caption{Penalty Matrix comparing the average accuracy of rare classes for MNIST (\textbf{top}) and SVHN (\textbf{bottom}) for different datasets using targeted learning across multiple budgets. We observe that the SMI functions have a much lower column sum compared to other approaches.
}
\label{pen-matrix1}
\end{figure}

\begin{figure}[!ht]
\centering
\includegraphics[width =0.7\textwidth]{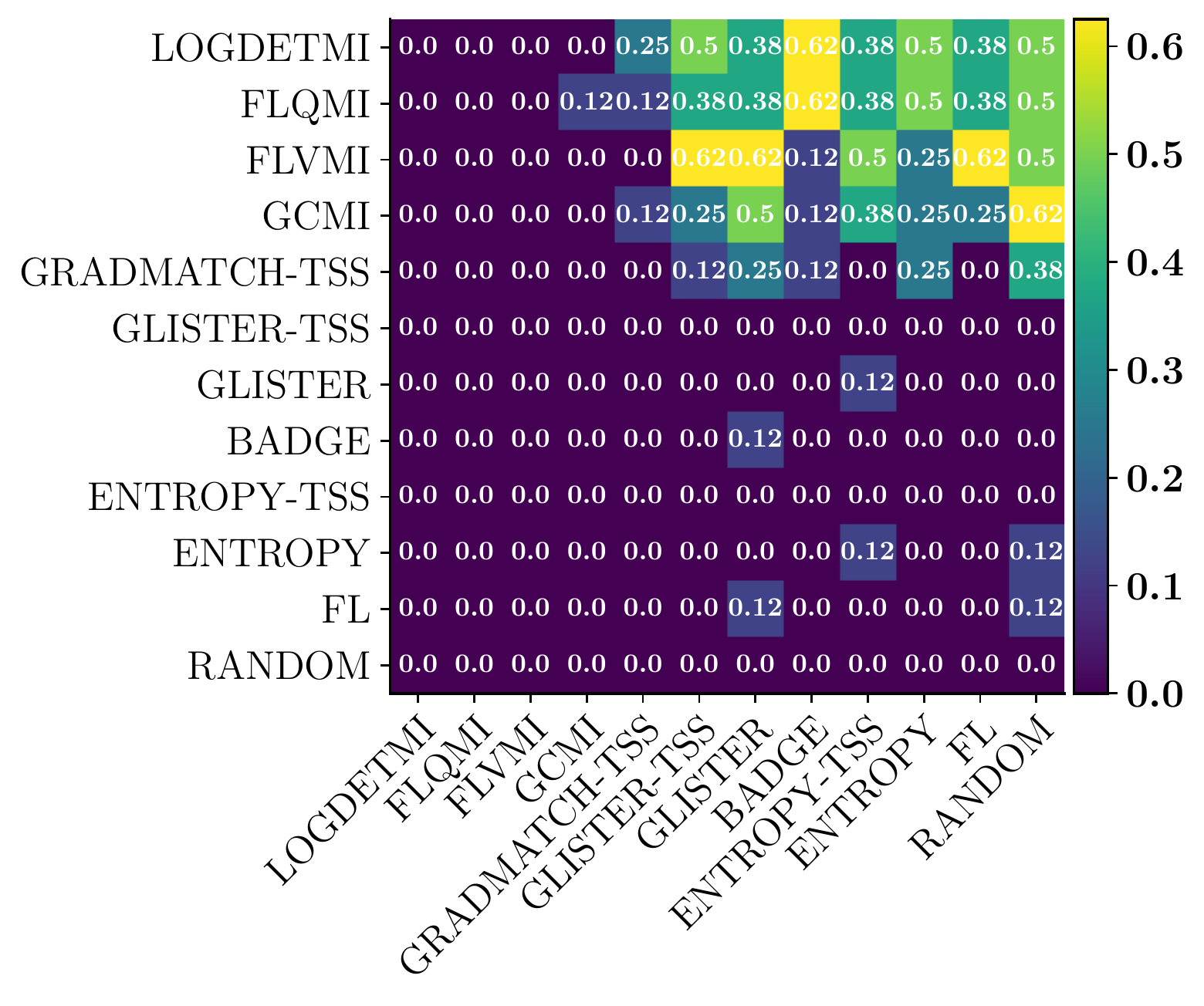}
\includegraphics[width = 0.7\textwidth]{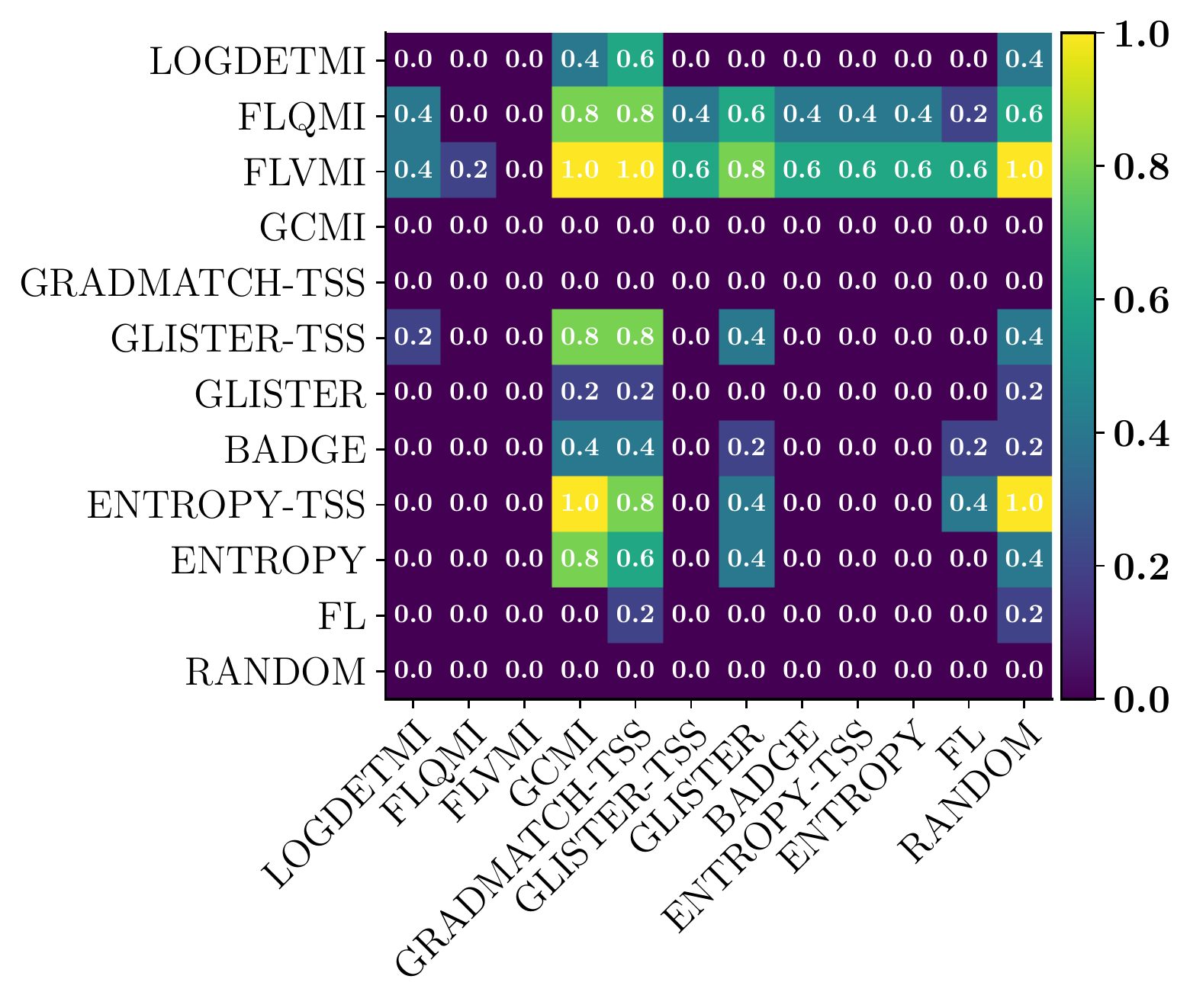}
\caption{Penalty Matrix comparing the average accuracy of rare classes for CIFAR-10 (\textbf{top}) and Pneumonia-MNIST (\textbf{bottom}) for different datasets using targeted learning across multiple budgets. We observe that the SMI functions have a much lower column sum compared to other approaches.
}
\label{pen-matrix2}
\end{figure}

\subsection{Details on Computing Labeling Efficiency}
The labeling efficiency is computed as the ratio of the minimum number of labeled training instances needed for an AL algorithm to reach a particular average rare class accuracy to the minimum number of labeled training instances needed by random sampling to reach that average rare class accuracy. Since most of the baselines are unable to reach each, the average rare class accuracy obtained by the MI functions at low budgets, we run targeted learning at the same intervals as shown in \figref{fig:TL} until the full unlabeled dataset is selected. Across multiple test accuracies for all datasets, we observe $20 \times - 50 \times$ labeling efficiency compared to random and $2 \times - 4 \times$ compared to the existing approaches.

\section{Additional Details on Guided Summarization using \model{} (Section ~\ref{sec:tsum})}\label{app:real-sum}

We use the image collection dataset of~\cite{tschiatschek2014learning}. The dataset has 14 image collections with 100 images each and provides many (~50-250) human summaries per collection. We extend it by creating dense noun concept annotations for every image to make it suitable for our task. We start by designing the universe of concepts based on the 600 object classes in OpenImagesv6~\cite{kuznetsova2018open} and 365 scenes in Places365~\cite{zhou2014learning}. We eliminate concepts common to both (for example, \emph{closet}) to get a unified list of 959 concepts. To ease the annotation process we adopt pseudo-labelling followed by human correction. Specifically for every image we get the concept labels from a Yolov3 model pre-trained on OpenImagesv6 (for object concepts) and a ResNet50 model pre-trained on Place365 (for scene concepts). We then ask 5 human annotators to separately and individually correct the automatically generated labels (pseudo-labels). This is followed by finding a consensus over the set of concepts for each image to arrive at the final annotation concept vectors for each image. We have developed a Python GUI tool to ease this pseudo-label correction process, which we plan to release.

In addition to the already available generic human summaries, we augment the dataset with query-focused, privacy-preserving and joint query-focused and privacy-preserving summaries for each image collection. Specifically, we design 2 uni-concept and 2 bi-concept queries / private sets for each image collection to cover different cases like a) both concepts belonging to same image b) both concepts belonging to different images c) only one concept in the image collection. This is similar in spirit to~\cite{sharghi2017query}. We ask a group of 10 human annotators (different from those who annotated for concepts) to create a human summary (of 5 images) for each image collection and query and/or private pair. To ensure gold standard summaries, we followed this by a verification round. Specifically, we asked at least three annotators to accept/reject the summaries thus produced and we discarded those human summaries which were rejected by two or more such human annotators.

\paragraph{Instructions to participants for object and scene annotations and screen shot of the tool}: these are available in Annotation\_Guidelines.docx attached with the supplementary material. 
% \href{https://docs.google.com/document/d/1VBrKSXl_JNLHXioTsOxP1FO1DS0C5-QVwuyYyC4R_O4/edit?usp=sharing}{a Google doc here}.

\paragraph{Hourly wage paid to participants and the total amount spent on participant compensation:} The entire effort of creating object and scene annotations along with producing the human summaries required about 170 man hours paid at the rate of Indian Rupees (INR) 700 per man hour and hence the total compensation provided was about INR 1,20,000. That is, approximately 1715 USD.

To instantiate the mixture model components, we represent images using the probabilistic feature vector taken from the output layer of YOLOv3 model~\cite{redmon2018yolov3} pre-trained on the open images dataset \cite{kuznetsova2018open} and concatenate it with the probability vector of scenes from the output layer of ~\cite{zhou2014learning} trained on the Places365 dataset \cite{zhou2017places}. The queries which are sets of concepts are mapped to a similar feature space as $k$-hot vectors ($k$ being the number of concepts in a query) to facilitate image-query similarity. Thus, both images and queries and/or elements in private set are represented using a $|\Ccal|$-dimensional vector where $\Ccal$ is the universe of concepts.

While more complex queries and methods of learning joint embedding between text and images could be employed, we chose simpler alternatives to stick to the main focus area of this work. 

We initialize the parameters randomly and train the mixture model for 20 epochs. As in \cite{tschiatschek2014learning}, we use 1 - V-ROUGE in the max-margin learning ( discussed in Section 5.2) and update parameters using Nesterov's accelerated gradient descent.

\end{document}